\documentclass{article}

\usepackage{iclr2026_conference,times}

\usepackage[utf8]{inputenc}
\usepackage[T1]{fontenc}
\usepackage{microtype}
\usepackage{graphicx}
\graphicspath{{figures/}}
\usepackage{wrapfig}
\usepackage{subfigure}
\usepackage{dblfloatfix}
\usepackage{booktabs}
\usepackage{url}
\usepackage{xcolor}

\usepackage{amsmath}
\usepackage{amssymb}
\usepackage{mathtools}
\usepackage{bm}
\usepackage{amsthm}
\usepackage{algorithm}
\usepackage{algorithmic}
\usepackage{array}
\usepackage{multirow}
\usepackage[hypertexnames=false,hidelinks]{hyperref}
\usepackage[capitalize,noabbrev]{cleveref}

\newcolumntype{P}[1]{>{\raggedright\arraybackslash}p{#1}}

\newcommand{\ntxt}[1]{\text{\normalfont #1}}

\newcommand{\E}{\mathbb{E}}

\newcommand{\FE}{\mathsf{FE}}

\newcommand{\vzero}{\bm{0}}
\newcommand{\vtheta}{\bm{\theta}}
\newcommand{\vu}{\bm{u}}
\newcommand{\vv}{\bm{v}}

\newcommand{\vx}{\bm{x}}

\newcommand{\vd}{\bm{d}}

\newcommand{\vg}{\bm{g}}
\newcommand{\vm}{\bm{m}}
\newcommand{\vdelta}{\bm{\delta}}

\newcommand{\rmI}{\bm{I}}

\newtheoremstyle{paperplain}
  {3pt plus 1pt minus 1pt}
  {3pt plus 1pt minus 1pt}
  {\fontfamily{LibertinusSerif-TLF}\selectfont\itshape}
  {}
  {\fontfamily{LibertinusSerif-TLF}\selectfont\bfseries}
  {.}
  {0.5em}
  {}
\theoremstyle{paperplain}
\newtheorem{theorem}{Theorem}[section]
\newtheorem{proposition}[theorem]{Proposition}
\newtheorem{lemma}[theorem]{Lemma}
\newtheorem{corollary}[theorem]{Corollary}
\theoremstyle{definition}
\newtheorem{definition}[theorem]{Definition}
\newtheorem{assumption}[theorem]{Assumption}
\theoremstyle{remark}

\title{ZoAQ: Adaptive Zeroth-Order Querying via Query-Reuse Coupling}
\author{
Yangyang Feng \qquad Yao Shu\\
The Hong Kong University of Science and Technology (Guangzhou)\\
\texttt{yfeng044@connect.hkust-gz.edu.cn} \quad
\texttt{yaoshu@hkust-gz.edu.cn}
}

\iclrfinalcopy
\begin{document}
\maketitle
\lhead{}
\begin{abstract}
Zeroth-order optimization (ZOO) estimates updates from function evaluations, making perturbation queries a primary cost. Fixed budgets spend the same number of queries at every step, while adaptive controllers may offset their savings by using additional oracle calls to test estimator reliability. We introduce ZoAQ\footnote{An implementation is available at \url{https://anonymous.4open.science/r/ZoAQ-312A}.}, an adaptive ZOO method built around query reuse. Rather than discarding past evaluations after each step, ZoAQ makes them useful for both the next update and the decision to query further. This enables adaptive query allocation without extra validation queries. Our analysis characterizes when this agreement identifies an update that supports descent and guides the controller to a sufficient query budget. On synthetic objectives, ZoAQ reduces queries by 43--48\% relative to fixed baselines using 1.2M queries. In black-box attacks, it reaches 100\% success with 320 and 625 average queries on MNIST and CIFAR-10, respectively. Across four OPT fine-tuning settings, ZoAQ saves 43--46\% forward evaluations relative to fixed $K=4$, with accuracy changes within tasks ranging from $-0.018$ to $+0.010$.
\end{abstract}

\section{Introduction}
\label{sec:introduction}

\begin{wrapfigure}{r}{0.39\linewidth}
\vspace{-0.8\baselineskip}
\centering
\includegraphics[width=\linewidth]{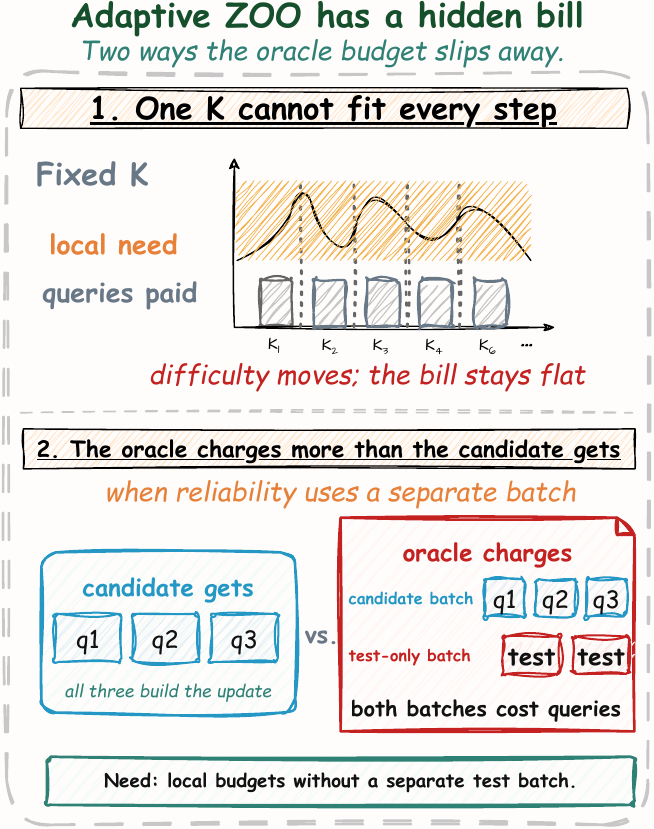}
\caption{Two hidden query costs in adaptive ZOO.}
\label{fig:zoaq_intro}
\vspace{-0.6\baselineskip}
\end{wrapfigure}

Zeroth-order optimization (ZOO) enables learning when gradients are unavailable or too costly to materialize, including black-box adversarial attacks and fine-tuning of large language models (LLMs) using forward passes only \citep{chen2017zoo,chen2019zoo,malladi2023finetuning}. Each update is estimated from function evaluations along random perturbations, so the local direction count determines both the reliability of the estimate and its cost \citep{spall1992multivariate,nesterov2017random}. A small budget may produce a noisy direction, whereas a large fixed budget wastes evaluations at iterates where fewer directions suffice. Choosing one budget for the entire trajectory therefore leaves the variation in local requirements unused.

Figure~\ref{fig:zoaq_intro} summarizes the two difficulties in allocating queries that motivate ZoAQ. Adaptive sampling can exploit this variation by adding directions only when the current estimate is unreliable. Existing rules often test estimator variance, norm, or accuracy \citep{Byrd2012,Bollapragada2023,bollapragada2024adaptive}. To run such a test, a controller may query a separate pilot or validation batch. Those evaluations consume the same oracle budget as the perturbations used to construct an update, but their responses do not improve the candidate estimate if they are used only for validation \citep{ghadimi2013stochastic}. When the candidate itself contains only a few directions, even a small validation batch can offset the savings from adaptation \citep{lin2025multi}. ZOO estimators based on variance reduction or query reuse can improve estimate quality under a prescribed sampling plan \citep{liu2018zeroth,qiu2025zoo}, but do not by themselves determine when the current iterate has received enough directions. A useful adaptive rule should instead make its decision from evidence that also enters the candidate update.

Momentum consistency is a natural candidate because it compares the current estimate with an exponential moving average (EMA) of past estimates. Momentum-based ZOO backbones already maintain this EMA \citep{chen2019zoo,shu2025refining}, so computing the cosine requires no additional function evaluations. When each step samples perturbation directions independently, however, this cosine is difficult to interpret. A low score may reflect either a change in the underlying gradient or simply a mismatch between the two direction samples.

ZoAQ makes this comparison useful by retaining past query information across consecutive steps. Each record stores its observed scalar response together with a reproducible perturbation seed, so the corresponding direction can be reconstructed without another oracle call. Retained records and newly acquired records form the same candidate estimate, whose agreement with the EMA determines whether ZoAQ accepts the update or adds more queries. After acceptance, the next step returns to the minimum budget. A difficult iterate can therefore spend more queries without making the larger budget the default at the next one. Because retained responses and every new evaluation contribute directly to the candidate update, the controller adapts the local query budget without a separate validation batch.

Our analysis explains when reuse is useful: additional historical records reduce estimation noise, but stale responses also introduce bias. We further show that the EMA error remains controllable even when consecutive estimates share records, so the consistency test can identify updates that support descent and stop expansion by a sufficient local budget.

Our contributions are summarized as follows:
\begin{itemize}
    \item We propose ZoAQ, which reuses past query information to construct the candidate update and decide whether more queries are needed, enabling adaptive local budgets without a separate validation batch.
    \item We characterize the trade-off between concentration and staleness induced by stored responses, derive pathwise EMA tracking under overlapping histories, and show that first passage yields a directionally safe update no later than a sufficient local budget.
    \item On synthetic objectives, ZoAQ reduces queries by 43--48\% relative to fixed baselines using 1.2M queries. In black-box attacks, it reaches 100\% success with 320 and 625 average queries on MNIST and CIFAR-10, respectively. Across four OPT fine-tuning settings, ZoAQ saves 43--46\% forward evaluations relative to fixed $K=4$, with accuracy changes within tasks ranging from $-0.018$ to $+0.010$.
\end{itemize}

\section{Related Work and Positioning}
\label{sec:related_work}

\paragraph{Fixed-budget zeroth-order backbones.}
A first line of work establishes how to make useful updates when gradients are unavailable. Classical zeroth-order methods estimate gradients from randomized function evaluations \citep{spall1992multivariate,nesterov2017random}, and the broader derivative-free optimization literature develops principled search and trust-region machinery for function-value oracles \citep{conn2000trust,conn2009introduction}. Modern ZOO methods bring this viewpoint to black-box learning and attacks through coordinate or random-direction estimators and adaptive-momentum updates \citep{chen2017zoo,chen2019zoo}, while R-AdaZO stabilizes the high-dimensional adaptive backbone used in our experiments \citep{shu2025refining}. Forward-only LLM fine-tuning gives the same query-accounting issue at model scale, with MeZO showing that language models can be tuned through forward passes and LoRA reducing the number of trainable parameters \citep{malladi2023finetuning,hu2021lora}. These methods solve the backbone problem: they specify how function evaluations become update directions. Their local query budget, however, is usually chosen outside the online accept-or-expand decision. ZoAQ keeps this backbone view and studies the missing allocation layer: when a step has enough query evidence to move. Normalized momentum provides a first-order precedent for separating directional progress from update magnitude \citep{cutkosky2020momentum,cutkosky2021high}; our normalized-update protocol follows this analytical convention to isolate directional reliability from update magnitude.

\paragraph{History reuse and variance reduction.}
A second line of work improves zeroth-order estimates by reducing variance or reusing information. ZO-SVRG adapts variance-reduction ideas to nonconvex ZOO \citep{liu2018zeroth}, and ZoAR shows that query reuse can produce more stable zeroth-order estimators \citep{qiu2025zoo}. These methods improve the estimator built from a collection of function evaluations, but they do not by themselves determine when sampling should stop at the current step. ZoAQ uses retained acquisition records to enlarge the candidate estimate and compares that candidate with an EMA of previously accepted estimates. Its distinct question is how the concentration gained from retained evidence can support local budget allocation without letting response staleness dominate.

\paragraph{Adaptive sampling.}
A third line of work chooses sample sizes online rather than fixing one budget for the whole trajectory. This idea is classical in first-order stochastic optimization \citep{Byrd2012}, and derivative-free or zeroth-order variants control stochastic approximation accuracy through norm-, variance-, or accuracy-based criteria \citep{Bollapragada2023,bollapragada2024adaptive}. Such criteria are effective when their reliability statistics can be estimated affordably. In low-query ZOO, additional oracle calls used only to assess reliability may consume a material fraction of the local budget. ZoAQ instead computes its momentum-consistency score from the candidate update and an EMA already maintained by the optimizer. Its novelty lies in combining stored-response estimation with an accept-or-expand allocation rule and characterizing when the resulting score is directionally meaningful.

\section{Problem Setup}
\label{sec:preliminaries}

\subsection{Zeroth-Order Estimation and Local Budgets}
\label{sec:zoo_problem}

The zeroth-order problem is to minimize $F(\vtheta)$ over $\vtheta\in\mathbb R^d$ using only function-value queries. For a smoothing radius $\mu>0$, define the Gaussian-smoothed objective
\begin{equation}
    F_\mu(\vtheta)
    \triangleq
    \mathbb E_{\vu\sim\mathcal N(0,\rmI_d)}
    [F(\vtheta+\mu\vu)].
\end{equation}
For an exogenous direction $\vu\sim\mathcal N(0,\rmI_d)$, two function evaluations define the scalar response and its single-direction vector estimate:
\begin{equation}
    r_\mu(\vtheta,\vu)
    \triangleq
    \frac{F(\vtheta+\mu\vu)-F(\vtheta-\mu\vu)}{2\mu},
    \qquad
    \vg_\mu(\vtheta,\vu)
    \triangleq
    r_\mu(\vtheta,\vu)\vu.
\end{equation}
For fixed $\vtheta$, the Gaussian smoothing identity gives $\mathbb E_{\vu}[\vg_\mu(\vtheta,\vu)]=\nabla F_\mu(\vtheta)$, so each direction yields an unbiased estimate of the smoothed gradient.

At step $t$, a fresh-query budget $K$ averages $K$ such estimates:
\begin{equation}
    \vg_t^{\mathrm{fresh}}(K)
    \triangleq
    \frac{1}{K}
    \sum_{k=1}^{K}
    \vg_\mu(\vtheta_t,\vu_{t,k}),
    \label{eq:grad_est_multi}
\end{equation}
where each $\vu_{t,k}$ has marginal distribution $\mathcal N(0,\rmI_d)$. The controller accepts $K_t$ from $\mathcal K=\{K_{\min},K_{\min}+\Delta K,\ldots,K_{\max}\}$, while a fixed-budget method uses the same $K$ at every step. ZoAQ combines the $K$ newly acquired records with retained records to form $\mathcal H_t(K)$ and the candidate estimate $\vg_t(K)$, with $n_t(K)=|\mathcal H_t(K)|$. Thus, $K$ counts fresh directions acquired at step $t$, while $n_t(K)$ counts the records used by the candidate estimate. Only the $K$ fresh directions incur new oracle calls at step $t$. Section~\ref{sec:framework} gives the exact construction.

\subsection{Directional Reliability}
\label{sec:norm_condition_flaw}

We assess a local budget by whether its estimated direction supports descent. To isolate this directional requirement, the analysis uses the following normalized-update protocol.
\begin{assumption}[Normalized-Update Analysis Protocol]
\label{ass:scale_invariant}
For a nonzero candidate estimate, let $\vd_t(K)=\vg_t(K)/\|\vg_t(K)\|$. The analyzed update is $\vtheta_{t+1}=\vtheta_t-\alpha_t\vd_t(K)$.
\end{assumption}
Under this protocol, a candidate budget satisfies the directional reliability condition when
\begin{equation}
    \cos\!\left(\vd_t(K),\nabla F_\mu(\vtheta_t)\right)
    \ge\delta,
    \qquad \delta>0,
    \label{eq:hidden_alignment_certificate}
\end{equation}

\begin{definition}[Oracle Alignment Budget]
\label{def:ideal_alignment_budget}
For $\delta>0$, when the following set is nonempty, define the oracle alignment budget
\begin{equation}
K_t^\circ(\delta)
\triangleq
\min\!\left\{
K\in\mathcal K:
\cos\!\left(\vd_t(K),\nabla F_\mu(\vtheta_t)\right)
\ge\delta
\right\}.
\label{eq:ideal_alignment_budget}
\end{equation}
If no grid point satisfies the condition, we set $K_t^\circ(\delta)=K_{\max}$ to keep the budget on the finite grid.
\end{definition}

The oracle budget depends on $\nabla F_\mu(\vtheta_t)$ and cannot be evaluated directly. ZoAQ instead chooses $K$ from a momentum-consistency score computed from the candidate estimate and recent accepted estimates. Every new record acquired during expansion enters the same candidate estimate.

\section{ZoAQ: Query-Reuse Adaptive Querying}
\label{sec:framework}

Figure~\ref{fig:zoaq_method_overview} summarizes how ZoAQ builds one nested candidate, expands it when needed, and commits the accepted update and controller state.

\begin{figure}[h]
\centering
\includegraphics[width=\linewidth]{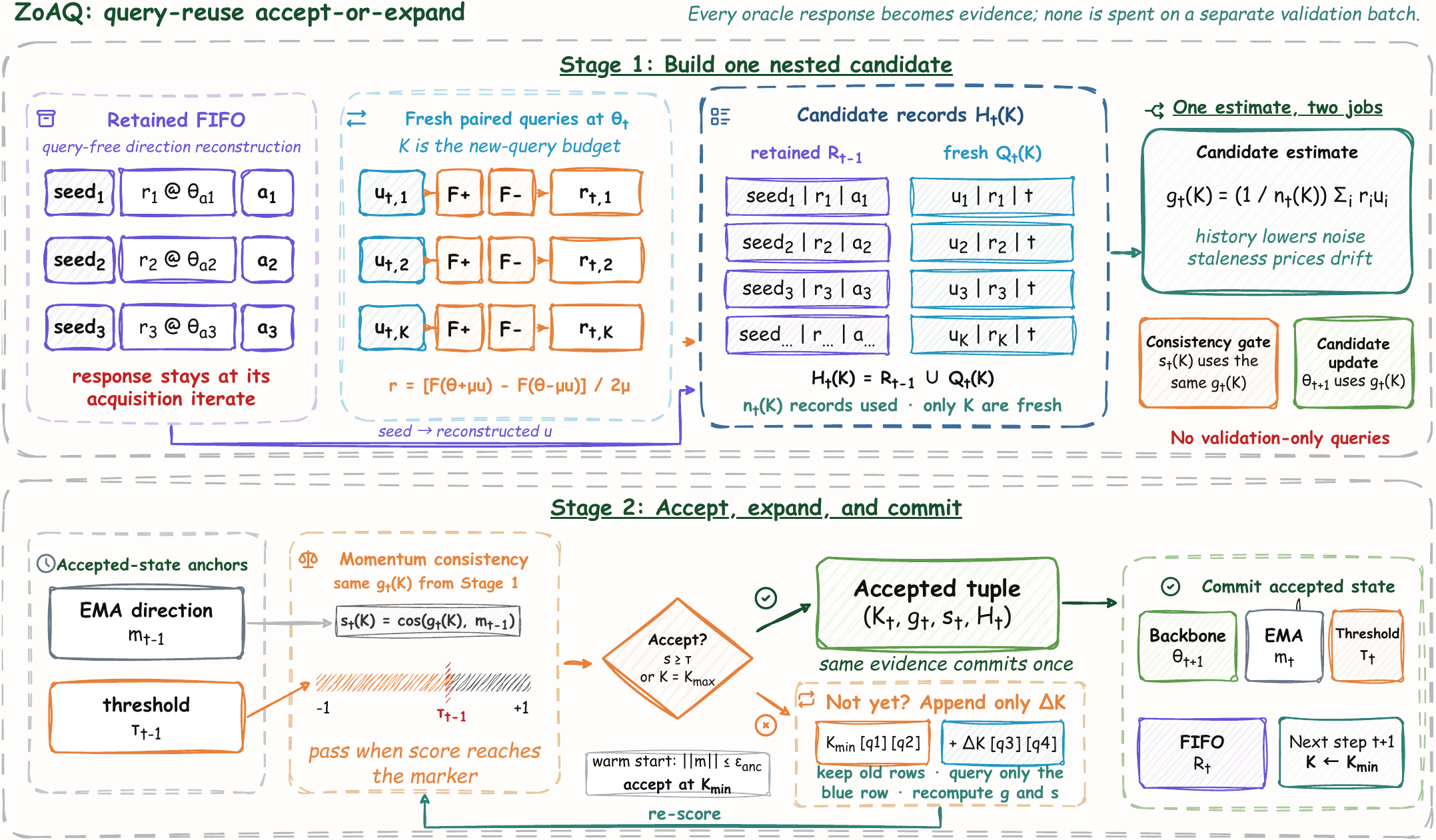}
\caption{ZoAQ at one optimizer step. Retained records and $K$ fresh directions form a nested candidate estimate. The same estimate drives both the backbone update and momentum-consistency gate; a failed test appends only $\Delta K$ fresh records, whereas acceptance commits the update, EMA, threshold, and FIFO history.}
\label{fig:zoaq_method_overview}
\end{figure}

\subsection{Query-Response Reuse}
\label{subsec:structural_correlation}

ZoAQ decides whether to expand the current query budget by comparing the candidate estimate with an exponential moving average (EMA) of previously accepted estimates. With fresh directions alone, a low consistency score can reflect either local gradient change or sampling noise. ZoAQ reduces this ambiguity by allowing retained query-response records to enter the candidate estimate.

Record $i$ stores a reproducible seed for the direction $\vu_i$, its acquisition step $a_i$, and the paired scalar response $r_i=r_\mu(\vtheta_{a_i},\vu_i)$ defined in Section~\ref{sec:zoo_problem}. The seed reconstructs $\vu_i$ without a new query, while $r_i$ remains the response observed at its acquisition iterate. Let $\mathcal R_{t-1}$ be the FIFO history and $\mathcal Q_t(K)$ the $K$ fresh records acquired at $\vtheta_t$. Their union $\mathcal H_t(K)=\mathcal R_{t-1}\cup\mathcal Q_t(K)$ defines the candidate estimate
\begin{equation}
    \vg_t(K)
    =
    \frac{1}{n_t(K)}
    \sum_{i=1}^{n_t(K)} r_i\vu_i,
    \qquad
    n_t(K)=|\mathcal H_t(K)|.
    \label{eq:history_reuse_grad_est}
\end{equation}
The sets $\mathcal Q_t(K)$ are nested, so expanding from $K$ to $K+\Delta K$ queries only the additional $\Delta K$ directions. After acceptance, FIFO retention keeps the most recent records for subsequent steps.

Query reuse and momentum average information at different stages. The candidate in Eq.~\ref{eq:history_reuse_grad_est} pools acquisition records before a budget is accepted, changing both the evidence available to the gate and the estimate sent to the backbone. The EMA instead pools previously accepted estimates after those decisions and supplies the reference at the next step. The first operation trades concentration against response staleness, whereas the second trades update noise against trajectory lag. They are complementary rather than duplicate smoothing operations; Appendix~\ref{app:aggregation_analysis} gives the corresponding decomposition.

\subsection{Momentum-Consistency Gate}

For each candidate budget, ZoAQ computes
\begin{equation}
    s_t(K)\triangleq \cos\!\left(\vg_t(K),\vm_{t-1}\right),
\end{equation}
where $\vm_{t-1}$ is the EMA of previously accepted estimates. Starting from $K_{\min}$, the controller accepts the first budget whose score reaches the current threshold:
\begin{equation}
    K_t=\min\!\left(\{K\in\mathcal K:s_t(K)\ge\tau_{t-1}\}\cup\{K_{\max}\}\right).
    \label{eq:zoaq_budget_rule}
\end{equation}
If $\|\vm_{t-1}\|\le\varepsilon_{\ntxt{anc}}$, ZoAQ takes a warm-start update at $K_{\min}$. If no score passes, it accepts the candidate at $K_{\max}$. Otherwise, every failed test adds fresh records to the same candidate estimate before the score is recomputed. The added queries therefore improve the candidate that can be accepted at the same step. After acceptance, ZoAQ passes $\vg_t$ to the ZOO backbone, updates the estimate EMA and score threshold, and restarts the next step from $K_{\min}$.

This control flow separates evidence acquisition within a step from state adaptation across steps. Within a step, rejection changes only the fresh budget: it neither advances $\vtheta_t$ nor updates the EMA, threshold, or FIFO buffer. Because the candidate sets are nested, each retest preserves all records already collected at that step and pays only for the next $\Delta K$ fresh directions. Acceptance is the commit point at which the chosen candidate updates the backbone and persistent controller state. Restarting from $K_{\min}$ then lets every new iterate test whether the cheapest budget is adequate, while the calibrated threshold carries information about the consistency of previously accepted estimates. Accordingly, $K_t$ is a local stopping outcome rather than a persistent operating level: expansion at one iterate does not force the next iterate to inherit the larger budget.

In Algorithm~\ref{alg:zoaq_main}, \textsc{Score} returns $\tau_{t-1}$ during warm start, returns $-\infty$ when the candidate estimate is zero and the EMA is nonzero, and otherwise returns the cosine score above. \textsc{ExtendHistory} carries retained records forward and adds only the new records from the current step needed to reach $K$. For a finite accepted score, \textsc{Calibrate} applies $\tau_t=\beta_\tau\tau_{t-1}+(1-\beta_\tau)s_t$; it leaves the threshold unchanged when $s_t=-\infty$.

\begin{algorithm}[t]
   \caption{ZoAQ Accept-or-Expand with Query Reuse}
   \label{alg:zoaq_main}
\begin{algorithmic}
   \STATE \textbf{Input:} oracle $F$, initial point $\vtheta_0$, horizon $T$
   \STATE \textbf{Parameters:} grid $(K_{\min},\Delta K,K_{\max})$, $N_{\ntxt{hist}}$, and $(\tau_0,\beta_1,\beta_\tau,\varepsilon_{\ntxt{anc}})$
   \STATE \textbf{Initialize:} retained buffer $\mathcal H_0\leftarrow\emptyset$ and EMA $\vm_0\leftarrow\vzero$
   \FOR{$t=1,\dots,T$}
       \STATE $K\leftarrow K_{\min}$
       \WHILE{true}
           \STATE $\mathcal H_t(K)\leftarrow \textsc{ExtendHistory}(F,\vtheta_t,\mathcal H_{t-1},K)$
           \STATE $\vg_t(K)\leftarrow \textsc{PairedEstimate}(\mathcal H_t(K))$ \hfill Eq.~\ref{eq:history_reuse_grad_est}
           \STATE $s_t(K)\leftarrow \textsc{Score}(\vg_t(K),\vm_{t-1},\tau_{t-1};\varepsilon_{\ntxt{anc}})$
           \STATE \textbf{if} $s_t(K)\ge\tau_{t-1}$ or $K=K_{\max}$ \textbf{break; else} $K\leftarrow\min\{K+\Delta K,K_{\max}\}$
       \ENDWHILE
       \STATE $(K_t,\vg_t,s_t,\mathcal H_t)\leftarrow (K,\vg_t(K),s_t(K),\mathcal H_t(K))$
       \STATE $\vtheta_{t+1}\leftarrow \textsc{BackboneStep}(\vtheta_t,\vg_t)$
       \STATE $\vm_t\leftarrow \beta_1\vm_{t-1}+(1-\beta_1)\vg_t$
       \STATE $\tau_t\leftarrow \textsc{Calibrate}(\tau_{t-1},s_t;\beta_\tau)$
       \STATE $\mathcal H_t\leftarrow \textsc{FIFO}(\mathcal H_t,N_{\ntxt{hist}})$
   \ENDFOR
   \STATE \textbf{return} $\vtheta_T$ and $\{K_t\}_{t=1}^{T}$
\end{algorithmic}
\end{algorithm}

The same candidate records determine both the update and the accept-or-expand score, so the controller requires no separate validation queries. Appendix~\ref{app:algorithmic_summary} gives the function evaluation accounting and update backbone for each protocol. Section~\ref{sec:theory} analyzes the stored-response controller and uses the normalized-update protocol in Assumption~\ref{ass:scale_invariant} for its descent result.

\section{Theory: Estimation, Consistency, and Local Budgets}
\label{sec:theory}

ZoAQ reuses scalar responses at their acquisition iterates. This increases the evidence available to a candidate without new function evaluations, but older records target earlier local gradients. We analyze this stored-response estimator and its accept-or-expand controller by quantifying the trade-off between concentration and staleness, propagating the accuracy of accepted candidates through the EMA, and characterizing safe first passage. The resulting estimation, tracking, and gate guarantees apply directly to the controller; the final descent statement uses the normalized-update protocol in Assumption~\ref{ass:scale_invariant}. Appendices~\ref{app:assumptions}--\ref{app:proof_complexity} give the full assumptions and proofs.

\subsection{Stored Responses Trade Concentration for Staleness}
\label{sec:assumptions}
\label{sec:properties}

For a candidate built from $n_t(K)$ records, define its total acquisition staleness
\begin{equation}
    A_t(K)
    \triangleq
    \sum_{j\in\mathcal J_t(K)}
    \|\vtheta_t-\vtheta_{a_j}\|.
    \label{eq:main_candidate_staleness}
\end{equation}
This quantity measures staleness by optimization movement, so an old record need not be harmful when the iterate has moved little.
The candidate differs from the current smoothed gradient $h_t=\nabla F_\mu(\vtheta_t)$ through acquisition noise and gradient movement.

\begin{lemma}[Selection-Safe Stored-Response Control]
\label{lemma:correlation_decomp}
On the simultaneous candidate event established in Appendix~\ref{app:proof_stat}, every candidate inspected by the controller satisfies
\begin{equation}
    \|\vg_t(K)-h_t\|
    \le
    e_t(K)
    \triangleq
    q(n_t(K),\delta)
    +\frac{L A_t(K)}{n_t(K)}.
    \label{eq:main_candidate_error}
\end{equation}
The statement remains valid for the random first-passing candidate.
\end{lemma}

Under a common local noise scale, the radius $q$ decreases as records accumulate, whereas the second term grows with the distance traveled since acquisition. At a common fresh budget $K$, reuse improves the error certificate whenever
\begin{equation}
    \frac{L A_t(K)}{n_t(K)}
    \le
    q(K,\delta)-q(n_t(K),\delta).
    \label{eq:main_reuse_benefit}
\end{equation}
Thus history is useful when its concentration gain exceeds its staleness cost. Under a full FIFO buffer, the leading envelope has the form $a_t/\sqrt C+b_tC$, yielding the interior capacity scale
\begin{equation}
    C_{\ntxt{opt}}
    \asymp
    \left(
        \frac{a_tK_{\min}}{L\bar\ell}
    \right)^{2/3}.
    \label{eq:main_capacity_scale}
\end{equation}
More history is therefore not uniformly better: noisier problems favor larger buffers, while faster movement favors smaller ones.

\subsection{Momentum Consistency as a Directional Certificate}
\label{sec:accepted_descent}

Because cosine is invariant to positive scaling, the implemented score is unchanged if the raw EMA is replaced by its debiased form $\widetilde{\vm}_{t-1}=\vm_{t-1}/(1-\beta_1^{t-1})$. If accepted estimates satisfy $\|\vg_r-h_r\|\le\varepsilon_r$, Appendix~\ref{app:momentum_alignment} proves the pathwise bound
\begin{equation}
    \|\widetilde{\vm}_{t-1}-h_t\|
    \le
    \sum_{r<t}\bar w_{r,t}\varepsilon_r
    +LD_t^{\ntxt{ema}}.
    \label{eq:main_debiased_tracking}
\end{equation}
This avoids independence assumptions between overlapping accepted buffers. Let $\kappa_t$ denote the right-hand side divided by $\|h_t\|$.

\begin{theorem}[Soundness and Completeness of Momentum Consistency]
\label{thm:implicit_control}
Suppose $\kappa_t<1$ and a nonzero candidate passes $\cos(\vg_t(K),\vm_{t-1})\ge\tau_{t-1}$. Then
\begin{equation}
    \cos(\vg_t(K),h_t)
    \ge
    \Delta(\tau_{t-1},\kappa_t)
    \triangleq
    \tau_{t-1}\sqrt{1-\kappa_t^2}
    -\kappa_t\sqrt{1-\tau_{t-1}^2}.
    \label{eq:main_gate_soundness}
\end{equation}
Conversely, a candidate is guaranteed to pass whenever
\begin{equation}
    \|\vg_t(K)-h_t\|
    \le
    \eta_\star(\tau_{t-1},\kappa_t)\|h_t\|,
    \label{eq:main_gate_completeness}
\end{equation}
where
\begin{equation}
    \eta_\star(\tau,\kappa)
    \triangleq
    \sqrt{1-\tau^2}\sqrt{1-\kappa^2}-\tau\kappa.
\end{equation}
\end{theorem}

Both directions are geometrically sharp. They coexist in the nonempty band
\begin{equation}
    \kappa_t<\tau_{t-1}<\sqrt{1-\kappa_t^2},
    \qquad
    \kappa_t<\frac{1}{\sqrt2}.
    \label{eq:main_threshold_band}
\end{equation}
Within the band, every pass is descent-relevant and sufficiently accurate candidates can pass. Our experiments use Algorithm~\ref{alg:zoaq_main}'s unprojected EMA, so these guarantees apply when its threshold lies in the band; Appendix~\ref{app:adaptive_threshold_theory} gives the projected calibration.

\subsection{Safe First Passage and Query Control}
\label{sec:controller_bridge}

Let $\bar\kappa<1/\sqrt2$ uniformly bound $\kappa_t$, and choose $\bar\kappa<\underline\tau\le\tau_{t-1}\le\overline\tau<\sqrt{1-\bar\kappa^2}$. Define
\begin{equation}
    \underline\Delta
    \triangleq\Delta(\underline\tau,\bar\kappa),
    \qquad
    \underline\eta
    \triangleq\eta_\star(\overline\tau,\bar\kappa),
\end{equation}
and the first sufficient grid point
\begin{equation}
    K_t^{\ntxt{suf}}
    \triangleq
    \min\left\{
        K\in\mathcal K:
        e_t(K)\le\underline\eta\|h_t\|
    \right\}.
    \label{eq:oracle_budget_main}
\end{equation}

\begin{theorem}[Safe First-Passing Local Allocation]
\label{thm:total_complexity}
\label{thm:zoaq_local_allocation}
If $K_t^{\ntxt{suf}}\le K_{\max}$, the increasing-grid search passes no later than $K_t^{\ntxt{suf}}$. Its first passing candidate satisfies
\begin{equation}
    K_t\le K_t^{\ntxt{suf}},
    \qquad
    \cos(\vg_t(K_t),h_t)\ge\underline\Delta>0.
    \label{eq:main_first_passing}
\end{equation}
Hence capped fallback does not occur on the simultaneous event. Under the normalized-update protocol in Assumption~\ref{ass:scale_invariant},
\begin{equation}
    F_\mu(\vtheta_{t+1})-F_\mu(\vtheta_t)
    \le
    -\alpha_t\underline\Delta\|h_t\|
    +\frac{L\alpha_t^2}{2}.
    \label{eq:main_normalized_descent}
\end{equation}
\end{theorem}

Soundness and completeness play different roles in this stopping argument. Soundness evaluates whichever candidate happens to pass, including one below $K_t^{\ntxt{suf}}$, and guarantees its alignment from the observed score rather than from its budget. Completeness supplies the stopping envelope: once the candidate error reaches the sufficient level in Eq.~\ref{eq:oracle_budget_main}, the score must pass. Their combination therefore does not require the empirical score to increase with $K$, nor does it require every earlier candidate to fail. An irregular score sequence may pass early, but any such pass remains covered by the same soundness guarantee; otherwise, the search stops by $K_t^{\ntxt{suf}}$.

Reuse enters this conclusion through the candidate error $e_t(K)$. When Eq.~\ref{eq:main_reuse_benefit} holds at the sufficient budget based on fresh directions alone, retained records meet the same accuracy requirement with no larger fresh budget. The comparison therefore holds locally under the stated reuse condition. Over a window before stationarity, the cost of fresh directions is bounded by $\sum_tK_t^{\ntxt{suf}}$ and hence by $K_{\max}$ times the window length. These guarantees cover steps satisfying the movement, tracking, threshold, and cap conditions, while the experiments measure realized allocation with the optimization backbone used in each task.

\section{Experiments: Mechanism-to-Scale Evidence}
\label{sec:experiments}

The experiments measure the oracle cost incurred by ZoAQ together with the resulting task outcome. Synthetic objectives expose how its scores and direction budgets evolve, black-box attacks compare queries at a common success rate, and OPT fine-tuning reports forward evaluations alongside task metrics. Mechanism diagnostics help interpret these results but are not used to assign all end-to-end gains to the gate alone. Appendix~\ref{app:experimental} and Table~\ref{tab:experiment_claim_map} provide the full configurations, accounting conventions, baseline standardization, and evidence-to-claim map.

\subsection{Synthetic Mechanism Tests}
\label{sec:exp_synthetic}

The synthetic regime makes local allocation directly visible. Figure~\ref{fig:synthetic_combined} shows the main effect under the protocol of 60k steps: the fixed baselines spend 1.2M queries, whereas ZoAQ uses about 0.62M--0.68M on the displayed objectives, a 43--48\% reduction. The fixed budget sweep in Table~\ref{tab:app_synthetic_sweep} controls for budget choice: changing the budget trades cost against final gap on individual tasks, but no single fixed value reproduces the adaptive balance between cost and quality across the suite. ZoAQ instead revisits $K_{\min}$ after every accepted update and acquires more evidence only when the current candidate fails the gate.

\begin{figure}[t]
\begin{center}
\centerline{\includegraphics[width=0.82\textwidth,trim=8 8 8 8,clip]{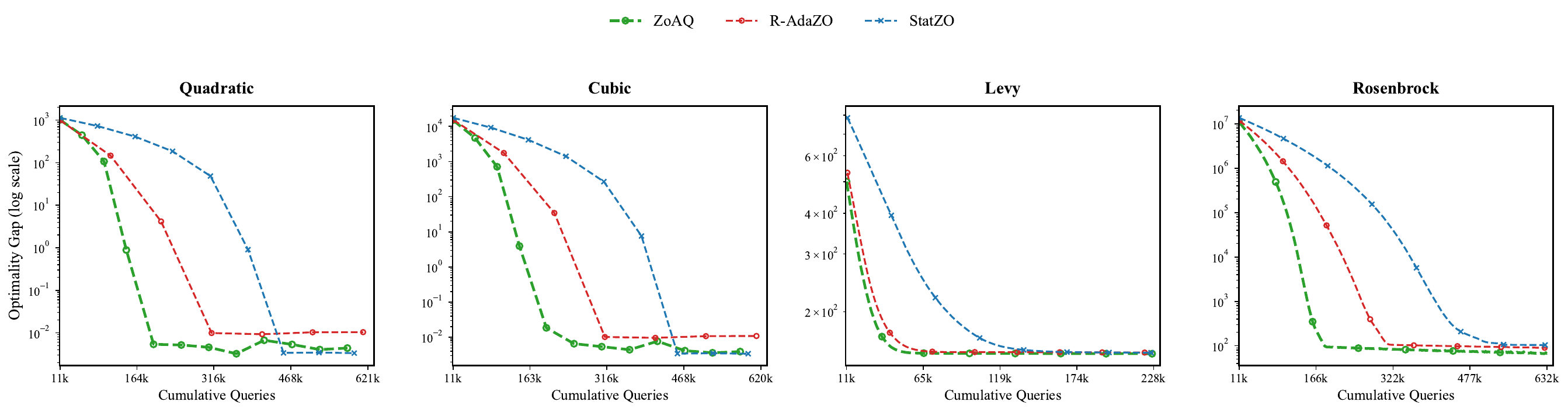}}
\caption{Synthetic convergence under the protocol of 60k steps, averaged over five runs. On the query axis, ZoAQ reaches regions of low error with fewer queries than baselines based on fixed budgets or variance.}
\label{fig:synthetic_combined}
\end{center}
\vskip -0.15in
\end{figure}

The auxiliary controls support the controller design. ZoAQ-Norm needs 28.1\% more queries to reach the same Rosenbrock target and shows a weaker descent rate than ZoAQ (1.3 vs. 1.9 loss drop per 1k queries); Appendix~\ref{app:ablation_norm_theory} and Figure~\ref{fig:ablation_norm_main} provide the corresponding diagnostic. Figure~\ref{fig:cosine_main} checks the gate statistic directly: the displayed baselines stay near zero or decay toward it, whereas ZoAQ remains more positive through most of the trajectory. Appendix~\ref{app:ablation} reports the reset diagnostics, and Appendix~\ref{app:aggregation_analysis} compares direction reuse with fresh independent sampling.

\subsection{Black-Box Attacks under Query Constraints}
\label{sec:exp_attack}

Black-box attacks test query allocation in a regime where $d \gg K_t$, local budgets are often one or two directions, and every additional oracle call changes the realized cost. Under a common stabilized R-AdaZO backbone, Table~\ref{tab:attack_results} gives 100\% success for all methods and separates them by required queries. ZoAQ reduces average queries to 320 on MNIST and 625 on CIFAR-10, corresponding to 4.77$\times$ and 1.61$\times$ speedups over fixed budget; StatZO does not recover comparable savings in this regime. Figure~\ref{fig:attack_cosine} provides a complementary diagnostic: ZoAQ maintains a more stable momentum-consistency score under the same attack regime. At the same 100\% success rate, history reuse alone does not recover the full gain: relative to ZoAR with fixed $K=2$, ZoAQ uses 56.9\% fewer queries on MNIST and 17.3\% fewer on CIFAR-10; Table~\ref{tab:attack_history_reuse_control} reports the full comparison.

\begin{table}[t]
\caption{Black-box attack query efficiency (mean$\pm$std over five runs). Full protocol details are deferred to Appendix~\ref{app:experimental}.}
\label{tab:attack_results}
\begin{center}
\begin{footnotesize}
\begin{sc}
\setlength{\tabcolsep}{2.2pt}
\renewcommand{\arraystretch}{0.92}
\resizebox{0.98\textwidth}{!}{
\begin{tabular}{lcccccccccccc}
\toprule
\multirow{2}{*}{Dataset} & \multicolumn{4}{c}{R-AdaZO ($K=2$)} & \multicolumn{4}{c}{StatZO ($K_{\max}=2$)} & \multicolumn{4}{c}{ZoAQ ($K_t\in[1,2]$)} \\
\cmidrule(lr){2-5} \cmidrule(lr){6-9} \cmidrule(lr){10-13}
 & Success & Avg & Std & Speedup & Success & Avg & Std & Speedup & Success & Avg & Std & Speedup \\
\midrule
MNIST & 100\% & 1,527 & 701 & 1.00$\times$ & 100\% & 1,996 & 1,231 & 0.77$\times$ & 100\% & 320 & 45 & 4.77$\times$ \\
CIFAR-10 & 100\% & 1,006 & 100 & 1.00$\times$ & 100\% & 964 & 28 & 1.04$\times$ & 100\% & 625 & 15 & 1.61$\times$ \\
\bottomrule
\end{tabular}
}
\end{sc}
\end{footnotesize}
\end{center}
\vskip -0.1in
\end{table}

Table~\ref{tab:attack_results} keeps the common stabilized R-AdaZO backbone explicit, with fixed $K=2$ for R-AdaZO, best swept $K_{\max}=2$ for StatZO, and $K_t\in[1,2]$ for ZoAQ. Appendix~\ref{app:ablation_nhist} justifies $N_{\ntxt{hist}}=8$, and Appendix~\ref{app:statzo_analysis} reports both the original and stabilized StatZO behavior.

\subsection{LLM Fine-Tuning with Forward Passes Only}
\label{sec:llm_finetuning}

LLM fine-tuning tests the same accounting when each function evaluation is a forward pass through a large model. We fine-tune OPT-1.3B and OPT-13B on SST-2 and COPA for 5,000 steps with LoRA under a two-sided ZOO protocol that uses forward passes only, with three seeds. Table~\ref{tab:llm_main_summary} summarizes the resulting trade-off between FE and quality.

\begin{table}[t]
\caption{Main LLM fine-tuning summary over three seeds. FE denotes forward evaluations; full metrics and the $K_t \in [1,2]$ endpoint are deferred to Table~\ref{tab:llm_results}.}
\label{tab:llm_main_summary}
\begin{center}
\begin{footnotesize}
\begin{sc}
\setlength{\tabcolsep}{2.3pt}
\renewcommand{\arraystretch}{0.82}
\resizebox{0.94\textwidth}{!}{
\begin{tabular}{lllcccc}
\toprule
Model & Task & Method & Total FE & FE Savings & Best Eval Loss & Accuracy \\
\midrule
\multirow{6}{*}{OPT-1.3B}
& \multirow{3}{*}{SST-2}
& R-AdaZO & 40000 $\pm$ 0 & 0.00 $\pm$ 0.00\% & 0.198 $\pm$ 0.002 & 0.920 $\pm$ 0.004 \\
& & StatZO & 30746 $\pm$ 412 & 23.14 $\pm$ 1.03\% & 0.206 $\pm$ 0.005 & 0.910 $\pm$ 0.007 \\
& & \textbf{ZoAQ} & 21714 $\pm$ 318 & \textbf{45.72 $\pm$ 0.80\%} & 0.195 $\pm$ 0.003 & 0.920 $\pm$ 0.005 \\
\cmidrule(lr){2-7}
& \multirow{3}{*}{COPA}
& R-AdaZO & 40000 $\pm$ 0 & 0.00 $\pm$ 0.00\% & 0.431 $\pm$ 0.004 & 0.780 $\pm$ 0.010 \\
& & StatZO & 30498 $\pm$ 505 & 23.76 $\pm$ 1.26\% & 0.461 $\pm$ 0.007 & 0.790 $\pm$ 0.014 \\
& & \textbf{ZoAQ} & 21946 $\pm$ 345 & \textbf{45.14 $\pm$ 0.86\%} & 0.485 $\pm$ 0.005 & 0.790 $\pm$ 0.012 \\
\midrule
\multirow{6}{*}{OPT-13B}
& \multirow{3}{*}{SST-2}
& R-AdaZO & 40000 $\pm$ 0 & 0.00 $\pm$ 0.00\% & 0.185 $\pm$ 0.002 & 0.938 $\pm$ 0.003 \\
& & StatZO & 31350 $\pm$ 390 & 21.63 $\pm$ 0.98\% & 0.202 $\pm$ 0.004 & 0.908 $\pm$ 0.006 \\
& & \textbf{ZoAQ} & 22245 $\pm$ 290 & \textbf{44.39 $\pm$ 0.73\%} & 0.192 $\pm$ 0.003 & 0.920 $\pm$ 0.005 \\
\cmidrule(lr){2-7}
& \multirow{3}{*}{COPA}
& R-AdaZO & 40000 $\pm$ 0 & 0.00 $\pm$ 0.00\% & 0.405 $\pm$ 0.005 & 0.798 $\pm$ 0.012 \\
& & StatZO & 31480 $\pm$ 480 & 21.30 $\pm$ 1.20\% & 0.452 $\pm$ 0.009 & 0.785 $\pm$ 0.016 \\
& & \textbf{ZoAQ} & 22780 $\pm$ 360 & \textbf{43.05 $\pm$ 0.90\%} & 0.428 $\pm$ 0.006 & 0.792 $\pm$ 0.011 \\
\bottomrule
\end{tabular}
}
\end{sc}
\end{footnotesize}
\end{center}
\vskip -0.12in
\end{table}

ZoAQ with $K_t \in [1,4]$ is the main operating point. Across the four OPT/task pairs in Table~\ref{tab:llm_main_summary}, it saves 43--46\% FE relative to fixed $K=4$. The allocation is repeated at each of the 5,000 steps rather than chosen once as a smaller fixed direction count. On OPT-1.3B SST-2, it matches R-AdaZO accuracy and slightly improves best evaluation loss; on COPA, it keeps accuracy close while accepting higher loss; on OPT-13B SST-2, it trails R-AdaZO with a fixed budget in both metrics. These results show FE savings relative to fixed $K=4$, while the quality trade-off varies by task. Reporting both quantities makes the saved computation and resulting task quality directly comparable. Appendix~\ref{app:llm_configurations} gives FE accounting, hyperparameters, full metrics, and the aggressive $K_t \in [1,2]$ endpoint; Appendix~\ref{app:statzo_analysis}, Sections~\ref{app:tau_sensitivity}--\ref{app:ablation_nhist}, and Appendix~\ref{app:ablation} provide the related stability and sensitivity checks.

\paragraph{Reuse and staleness across regimes.}
Sweeps over the history window expose the balance predicted by the analysis of stored responses. On MNIST attacks, average queries decrease from 374.4 at $N_{\ntxt{hist}}=4$ to 319.6 at $N_{\ntxt{hist}}=8$, then increase to 349.2 at $N_{\ntxt{hist}}=12$. On OPT-1.3B, enlarging the main window from 15 to 24 reduces FE by 3.00\% on SST-2 and 3.67\% on COPA, with changes in loss and accuracy that vary by task. The best intermediate window is consistent with the capacity trade-off in Eq.~\ref{eq:main_capacity_scale}. Together, these results show that the balance between reuse and staleness depends on the regime; Tables~\ref{tab:ablation_nhist} and \ref{tab:llm_history_sensitivity} report the full sweeps.

The controlled objective switch probes the complementary case in which retained records suddenly become stale. When the objective changes from Rosenbrock to Levy, the effective threshold drops with the consistency mismatch, failed tests add fresh evidence, and ZoAQ reaches a lower final loss than R-AdaZO in this test. The controller therefore refreshes its evidence instead of remaining locked to the previous history. Together, the window and switch tests show that useful reuse depends on current movement rather than a fixed history length, consistent with the staleness quantity in Eq.~\ref{eq:main_candidate_staleness}. Appendix~\ref{app:sharp_turn} gives the protocol and trajectories.

\section{Conclusion}
\label{sec:conclusion}

ZoAQ uses acquired responses as evidence for both updates and allocation. Retained records join fresh directions in each candidate, and failed EMA consistency tests trigger additional queries. Our analysis identifies when the gain in concentration exceeds staleness and when the first passing candidate supports descent under a sufficient local budget. ZoAQ reduces queries by 43--48\% on synthetic objectives, delivers 4.77$\times$ and 1.61$\times$ attack speedups at 100\% success, and saves 43--46\% FE in OPT fine-tuning with quality changes reported for each task. The history window and controlled objective switch show how the controller exploits useful records and refreshes stale ones, supporting query reuse as a practical basis for adapting local ZOO budgets.

\label{page:main-content-end}

\clearpage
\bibliography{references}
\bibliographystyle{iclr2026_conference}

\clearpage
\appendix
\section{Assumptions and Preliminaries}
\label{app:assumptions}

This appendix fixes the notation and probability interfaces used throughout the proofs. The analysis uses paired Gaussian two-sided responses stored at their acquisition iterates. Evaluation counts for each protocol are reported with the experiments rather than folded into the estimator definition.

\subsection{Proof Notation}

\begin{description}
    \item[Objective and gradients.] $F:\mathbb R^d\to\mathbb R$ is the black-box objective, $\vtheta_t$ is the iterate, and
    \begin{equation}
        F_\mu(\vtheta)
        \triangleq
        \E_{\vu\sim\mathcal N(0,\rmI_d)}[F(\vtheta+\mu\vu)]
        \label{eq:smoothed_def_app}
    \end{equation}
    is its Gaussian smoothing at radius $\mu>0$. We write $h_t\triangleq\nabla F_\mu(\vtheta_t)$. The constants $L_f$ and $L$ are gradient-Lipschitz constants for $F$ and $F_\mu$, respectively.

    \item[Acquisition records.] Directions are indexed in acquisition order. The $j$th record is sampled at step $a_j$ after conditioning on the pre-sampling information $\mathcal F_{j-1}$. It contains $\vu_j\mid\mathcal F_{j-1}\sim\mathcal N(0,\rmI_d)$ and the paired response $r_j\triangleq r_\mu(\vtheta_{a_j},\vu_j)$. Its vector contribution, conditional mean, and centered acquisition noise are
    \begin{equation}
        X_j\triangleq r_j\vu_j,
        \qquad
        h_j\triangleq\nabla F_\mu(\vtheta_{a_j}),
        \qquad
        \xi_j\triangleq X_j-h_j.
        \label{eq:record_notation_app}
    \end{equation}

    \item[Candidate estimates.] At step $t$, $K$ is the cumulative number of fresh directions acquired so far, $\mathcal H_t(K)$ is the candidate record set, and $\mathcal J_t(K)$ contains its acquisition indices. With $n_t(K)=|\mathcal J_t(K)|$,
    \begin{equation}
        \vg_t(K)
        \triangleq
        \frac{1}{n_t(K)}\sum_{j\in\mathcal J_t(K)}X_j.
        \label{eq:candidate_record_estimator_app}
    \end{equation}
    The controller searches $\mathcal K=\{K_{\min},K_{\min}+\Delta K,\ldots,K_{\max}\}$ and accepts $K_t$, after which $\vg_t\triangleq\vg_t(K_t)$.

    \item[Momentum and path.] Accepted estimates form $\vm_t=\beta_1\vm_{t-1}+(1-\beta_1)\vg_t$. Its debiased version is $\widetilde{\vm}_t\triangleq\vm_t/(1-\beta_1^t)$. The score is $s_t(K)\triangleq\cos(\vg_t(K),\vm_{t-1})$, and $\tau_{t-1}$ is the threshold used at step $t$. The realized path increment is $\ell_t\triangleq\|\vtheta_{t+1}-\vtheta_t\|$.
\end{description}

\subsection{Regularity and Smoothing}

\begin{assumption}[Smooth Objective and Surrogate]
\label{ass:smoothness_app}
The original objective $F$ is differentiable with $L_f$-Lipschitz gradient, and $F_\mu$ has an $L$-Lipschitz gradient on the region analyzed. Hence
\begin{equation}
    F_\mu(\vtheta')
    \le
    F_\mu(\vtheta)
    +\langle\nabla F_\mu(\vtheta),\vtheta'-\vtheta\rangle
    +\frac{L}{2}\|\vtheta'-\vtheta\|^2.
    \label{eq:smooth_descent_inequality_app}
\end{equation}
\end{assumption}

\begin{lemma}[Smoothing Bias]
\label{lemma:smoothing_properties_app}
Under Assumption~\ref{ass:smoothness_app},
\begin{equation}
    \|\nabla F_\mu(\vtheta)-\nabla F(\vtheta)\|
    \le
    \mu L_f\sqrt d.
    \label{eq:smoothing_bias_app}
\end{equation}
Thus an $\epsilon$-stationary point of $F_\mu$ is an $(\epsilon+\mu L_f\sqrt d)$-stationary point of $F$.
\end{lemma}

\begin{proof}
Differentiation under the Gaussian expectation gives $\nabla F_\mu(\vtheta)=\E[\nabla F(\vtheta+\mu\vu)]$. Gradient Lipschitzness and $\E\|\vu\|\le\sqrt{\E\|\vu\|^2}=\sqrt d$ prove Eq.~\ref{eq:smoothing_bias_app}.
\end{proof}

\subsection{Paired Gaussian Contributions}

For $\vu\sim\mathcal N(0,\rmI_d)$, define
\begin{equation}
    r_\mu(\vtheta,\vu)
    \triangleq
    \frac{F(\vtheta+\mu\vu)-F(\vtheta-\mu\vu)}{2\mu},
    \qquad
    \vg_\mu(\vtheta,\vu)
    \triangleq
    r_\mu(\vtheta,\vu)\vu.
    \label{eq:paired_gaussian_contribution_app}
\end{equation}
An acquired record stores a reproducible seed for $\vu$ and its scalar response. Reuse reconstructs the direction and reuses the scalar observed at acquisition without evaluating $F$ again.

\begin{lemma}[Mean Paired Contribution]
\label{lemma:paired_gaussian_unbiased_app}
Whenever Gaussian differentiation is valid,
\begin{equation}
    \E_{\vu}[\vg_\mu(\vtheta,\vu)]
    =\nabla F_\mu(\vtheta).
    \label{eq:paired_gaussian_unbiased_app}
\end{equation}
\end{lemma}

\begin{proof}
Gaussian symmetry gives
\begin{equation}
    \E_{\vu}[\vg_\mu(\vtheta,\vu)]
    =
    \frac{1}{\mu}\E_{\vu}[F(\vtheta+\mu\vu)\vu].
\end{equation}
Stein's identity identifies the right-hand side with $\nabla F_\mu(\vtheta)$.
\end{proof}

\begin{lemma}[Local Second-Moment Scale]
\label{lemma:paired_gaussian_second_moment_app}
Under Assumption~\ref{ass:smoothness_app},
\begin{equation}
    \E_{\vu}\|\vg_\mu(\vtheta,\vu)\|^2
    \le
    2(d+2)\|\nabla F(\vtheta)\|^2
    +\frac{L_f^2\mu^2}{2}d(d+2)(d+4).
    \label{eq:paired_gaussian_second_moment_app}
\end{equation}
Consequently, for a universal constant $C>0$,
\begin{equation}
    \E_{\vu}\|\vg_\mu(\vtheta,\vu)-\nabla F_\mu(\vtheta)\|^2
    \le
    C\left(d\|\nabla F_\mu(\vtheta)\|^2+L_f^2\mu^2d^3\right).
    \label{eq:paired_gaussian_noise_scale_app}
\end{equation}
\end{lemma}

\begin{proof}
Let $g^\star=\nabla F(\vtheta)$. The fundamental theorem of calculus gives
\begin{equation}
    r_\mu(\vtheta,\vu)
    =\langle g^\star,\vu\rangle+\zeta(\vtheta,\vu),
    \qquad
    |\zeta(\vtheta,\vu)|
    \le\frac{L_f\mu}{2}\|\vu\|^2.
    \label{eq:central_difference_remainder_app}
\end{equation}
Therefore
\begin{equation}
    \|\vg_\mu(\vtheta,\vu)\|^2
    \le
    2\langle g^\star,\vu\rangle^2\|\vu\|^2
    +\frac{L_f^2\mu^2}{2}\|\vu\|^6.
\end{equation}
For a standard Gaussian vector,
\begin{equation}
    \E[\langle g^\star,\vu\rangle^2\|\vu\|^2]
    =(d+2)\|g^\star\|^2,
    \qquad
    \E\|\vu\|^6=d(d+2)(d+4).
\end{equation}
This proves Eq.~\ref{eq:paired_gaussian_second_moment_app}. Lemma~\ref{lemma:paired_gaussian_unbiased_app}, Eq.~\ref{eq:smoothing_bias_app}, and $(a+b)^2\le2a^2+2b^2$ then give Eq.~\ref{eq:paired_gaussian_noise_scale_app} after collecting constants.
\end{proof}

\subsection{Probability Interfaces}

\begin{assumption}[Second Moments at Acquisition]
\label{ass:bounded_variance_app}
For every acquired record,
\begin{equation}
    \E[\xi_j\mid\mathcal F_{j-1}]=0,
    \qquad
    \E[\|\xi_j\|^2\mid\mathcal F_{j-1}]\le\sigma_j^2,
    \label{eq:conditional_second_moment_app}
\end{equation}
where $\sigma_j$ is $\mathcal F_{j-1}$-measurable. Lemma~\ref{lemma:paired_gaussian_second_moment_app} permits
\begin{equation}
    \sigma_j^2
    =C\left(d\|h_j\|^2+L_f^2\mu^2d^3\right)
    \label{eq:local_sigma_choice_app}
\end{equation}
for the paired Gaussian model.
\end{assumption}

\begin{assumption}[Conditional Bernstein Tails]
\label{ass:conditional_tail_app}
There are $\mathcal F_{j-1}$-measurable $\nu_j,b_j>0$ such that, for every deterministic unit vector $\vv$ and $|\lambda|<1/b_j$,
\begin{equation}
    \E\!\left[
        \exp(\lambda\langle\vv,\xi_j\rangle)
        \middle|\mathcal F_{j-1}
    \right]
    \le
    \exp\!\left(
        \frac{\lambda^2\nu_j^2}{2(1-b_j|\lambda|)}
    \right).
    \label{eq:conditional_bernstein_app}
\end{equation}
\end{assumption}

The second-moment route is complete under Assumption~\ref{ass:bounded_variance_app}. Assumption~\ref{ass:conditional_tail_app} is a separate stronger premise used only for concentration with logarithmic dependence on confidence; it is not inferred from a variance bound.

\section{Stored-Response Estimation}
\label{app:proof_stat}

This section establishes the estimation guarantee that supports the rest of the theory. A stored-response candidate differs from the current smoothed gradient through two terms: acquisition noise, whose concentration can improve as records accumulate, and staleness, which increases as retained records move farther from the current iterate.

The analysis uses the objects at acquisition defined in Appendix~\ref{app:assumptions} and FIFO pruning by pairs after acceptance. The main event is uniform over all candidates that the controller may inspect, so it remains valid after selection of the first passing candidate. We give a complete second-moment bound and a sharper version with logarithmic dependence on confidence under the additional conditional tail assumption.

\subsection{A Uniform Error Bound for Stored-Response Candidates}

The controller inspects a nested sequence of candidates and accepts the first one that passes. We therefore control every candidate that can be formed, rather than analyze a fixed budget and substitute the selected budget afterward.

Under FIFO pruning by pairs after acceptance, each candidate index set $\mathcal J_t(K)$ is an interval in acquisition order. Up to a horizon $T$, let $\mathfrak I_T$ contain every interval of acquisition indices that can occur for any $t\le T$ and $K\in\mathcal K$, and write $M_T\triangleq|\mathfrak I_T|$. Since at most $TK_{\max}$ directions are acquired, one may use the deterministic bound
\begin{equation}
    M_T
    \le
    \frac{TK_{\max}(TK_{\max}+1)}{2}.
    \label{eq:candidate_interval_count_app}
\end{equation}
For a realized candidate, define its total staleness
\begin{equation}
    A_t(K)
    \triangleq
    \sum_{j\in\mathcal J_t(K)}
    \|\vtheta_t-\vtheta_{a_j}\|.
    \label{eq:candidate_staleness_app}
\end{equation}

\begin{lemma}[Uniform Candidate Error]
\label{lemma:variance_detailed}
Suppose Assumptions~\ref{ass:smoothness_app} and~\ref{ass:bounded_variance_app} hold. For each admissible candidate, let $\overline V_t(K)$ be a deterministic upper bound on $\sum_{j\in\mathcal J_t(K)}\sigma_j^2$, and define
\begin{equation}
    q_2(t,K,\delta)
    \triangleq
    \frac{1}{n_t(K)}
    \sqrt{\frac{M_T\overline V_t(K)}{\delta}}.
    \label{eq:second_moment_radius_app}
\end{equation}
Then, with probability at least $1-\delta$, every candidate inspected up to step $T$ satisfies
\begin{equation}
    \|\vg_t(K)-h_t\|
    \le
    q_2(t,K,\delta)
    +
    \frac{L A_t(K)}{n_t(K)}.
    \label{eq:uniform_candidate_bound_app}
\end{equation}

If Assumption~\ref{ass:conditional_tail_app} also holds, let $\overline V_t^\nu(K)$ and $\overline b_t(K)$ be deterministic bounds on $\sum_{j\in\mathcal J_t(K)}\nu_j^2$ and $\max_{j\in\mathcal J_t(K)}b_j$. There is a universal constant $c>0$ such that Eq.~\ref{eq:uniform_candidate_bound_app} also holds with $q_2$ replaced by
\begin{equation}
    q_{\ntxt{B}}(t,K,\delta)
    \triangleq
    \frac{c}{n_t(K)}
    \left[
        \sqrt{\overline V_t^\nu(K)
        \left(d+\log\frac{M_T}{\delta}\right)}
        +
        \overline b_t(K)
        \left(d+\log\frac{M_T}{\delta}\right)
    \right].
    \label{eq:bernstein_radius_app}
\end{equation}
\end{lemma}

\begin{proof}
For any candidate, the definitions in Appendix~\ref{app:assumptions} give the exact decomposition
\begin{equation}
    \vg_t(K)-h_t
    =
    \frac{1}{n_t(K)}
    \sum_{j\in\mathcal J_t(K)}\xi_j
    +
    \frac{1}{n_t(K)}
    \sum_{j\in\mathcal J_t(K)}(h_j-h_t).
    \label{eq:candidate_noise_staleness_decomp_app}
\end{equation}
The second term is controlled pathwise. The $L$-smoothness of $F_\mu$ yields
\begin{equation}
    \left\|
        \frac{1}{n_t(K)}
        \sum_{j\in\mathcal J_t(K)}(h_j-h_t)
    \right\|
    \le
    \frac{L A_t(K)}{n_t(K)}.
    \label{eq:staleness_term_bound_app}
\end{equation}

Fix an acquisition interval $I\in\mathfrak I_T$. The sequence $\{\xi_j\}$ is a martingale-difference sequence with respect to the acquisition filtration, so its cross terms vanish and
\begin{equation}
    \E\left\|\sum_{j\in I}\xi_j\right\|^2
    \le
    \E\sum_{j\in I}\sigma_j^2
    \le
    \overline V_t(K),
\end{equation}
where $(t,K)$ denotes the candidate associated with $I$. Markov's inequality with failure probability $\delta/M_T$, followed by a union bound over $\mathfrak I_T$, gives Eq.~\ref{eq:second_moment_radius_app} simultaneously for every admissible interval.

Under Assumption~\ref{ass:conditional_tail_app}, scalar martingale Bernstein applied to $\langle\vv,\sum_{j\in I}\xi_j\rangle$ gives the usual variance-plus-range bound for each fixed unit vector $\vv$. A $1/2$-net of the unit sphere has at most $5^d$ elements. Taking a union bound over this net and over $\mathfrak I_T$ yields Eq.~\ref{eq:bernstein_radius_app}. Combining either noise radius with Eq.~\ref{eq:staleness_term_bound_app} proves the result. Because the event holds for every admissible interval, it also holds for the adaptively selected first-passing candidate.
\end{proof}

\subsection{When Stored Responses Reduce the Required Fresh Budget}
\label{rem:variance_bias_detailed}

History is useful only when the reduction in sampling error exceeds the staleness it introduces. The next result states this comparison at a common fresh budget.

\begin{proposition}[Reuse-Benefit Condition]
\label{prop:reuse_benefit_app}
Fix a step $t$ and fresh budget $K$. Suppose the same nonincreasing noise radius $q(n,\delta)$ bounds an average of $n$ admissible records. The candidate using only fresh directions contains $K$ records from the current step and has error certificate $q(K,\delta)$. The candidate with stored responses uses $n_t(K)\ge K$ records and has the certificate from Lemma~\ref{lemma:variance_detailed}. If
\begin{equation}
    \frac{L A_t(K)}{n_t(K)}
    \le
    q(K,\delta)-q(n_t(K),\delta),
    \label{eq:reuse_benefit_condition_app}
\end{equation}
then the candidate with stored responses has an error certificate no larger than the candidate using only fresh directions at the same fresh query cost.
\end{proposition}

\begin{proof}
The candidate using only fresh directions has no acquisition staleness, so its certificate is $q(K,\delta)$. Lemma~\ref{lemma:variance_detailed} bounds the candidate with stored responses by
\begin{equation}
    q(n_t(K),\delta)
    +
    \frac{L A_t(K)}{n_t(K)}.
\end{equation}
Applying Eq.~\ref{eq:reuse_benefit_condition_app} proves the claim.
\end{proof}

FIFO pruning turns the staleness term into an explicit capacity constraint.

\begin{lemma}[FIFO Staleness Envelope]
\label{lemma:fifo_staleness_app}
Let the retained buffer contain at most $C$ direction records, with pruning applied after each accepted step. Suppose every accepted step acquires at least $K_{\min}$ fresh directions and $\ell_r\le\bar\ell$ over the steps that can contribute to the current buffer. Once the buffer is full, every retained record is at most
\begin{equation}
    H_C
    \triangleq
    \left\lceil\frac{C}{K_{\min}}\right\rceil
    \label{eq:fifo_age_horizon_app}
\end{equation}
accepted steps old. Before pruning after acceptance, a candidate with fresh budget $K$ therefore satisfies
\begin{equation}
    n_t(K)=C+K,
    \qquad
    A_t(K)
    \le
    C H_C\bar\ell.
    \label{eq:fifo_staleness_envelope_app}
\end{equation}
\end{lemma}

\begin{proof}
During any $H_C$ consecutive accepted steps, at least $H_CK_{\min}\ge C$ newer records arrive. FIFO pruning by pairs must therefore remove every older record. A retained record acquired $r$ accepted steps earlier is at distance at most $\sum_s\ell_s\le r\bar\ell$ from $\vtheta_t$. Summing this bound over at most $C$ retained records gives Eq.~\ref{eq:fifo_staleness_envelope_app}; the $K$ fresh records contribute zero staleness.
\end{proof}

\begin{corollary}[Interior Capacity Scale]
\label{cor:capacity_scale_app}
Suppose the leading sampling radius for a full buffer has the homogeneous form $a_t(\delta)/\sqrt C$, and the average FIFO age is $c_{\ntxt{age}}C/K_{\min}$ for a constant $c_{\ntxt{age}}>0$. The leading certificate is then
\begin{equation}
    e_t(C)
    =
    \frac{a_t(\delta)}{\sqrt C}
    +
    \frac{c_{\ntxt{age}}L\bar\ell}{K_{\min}}C.
    \label{eq:capacity_certificate_app}
\end{equation}
Its continuous minimizer is
\begin{equation}
    C_t^\star
    =
    \left(
        \frac{a_t(\delta)K_{\min}}
        {2c_{\ntxt{age}}L\bar\ell}
    \right)^{2/3}.
    \label{eq:capacity_scale_app}
\end{equation}
Thus the useful history size grows with local noise and refresh rate, but shrinks with smoothness-weighted movement. More retained records are not uniformly better.
\end{corollary}

\begin{proof}
Differentiate Eq.~\ref{eq:capacity_certificate_app} with respect to $C$ and set the derivative to zero:
\begin{equation}
    -\frac{a_t(\delta)}{2C^{3/2}}
    +
    \frac{c_{\ntxt{age}}L\bar\ell}{K_{\min}}
    =0.
\end{equation}
Solving for $C$ gives Eq.~\ref{eq:capacity_scale_app}. The implementable capacity is the nearest admissible integer.
\end{proof}

\subsection{Outputs of the Stored-Response Analysis}
\label{app:proof_stat_details}

The results above reduce the effect of query reuse to one candidate-specific error certificate. For either probability route, write
\begin{equation}
    \varepsilon_t(K)
    \triangleq
    q(t,K,\delta)
    +
    \frac{L A_t(K)}{n_t(K)},
    \label{eq:candidate_error_certificate_app}
\end{equation}
where $q$ is the applicable radius from Lemma~\ref{lemma:variance_detailed}. On the simultaneous event, $\|\vg_t(K)-h_t\|\le\varepsilon_t(K)$ for every inspected candidate, including the first one accepted by the controller.

Two consequences feed the remaining proofs. First, Eq.~\ref{eq:reuse_benefit_condition_app} identifies when retained responses improve the certificate at a fixed fresh budget. Second, Eq.~\ref{eq:capacity_scale_app} shows how FIFO capacity balances sampling error against staleness induced by movement. Appendix~\ref{app:implicit_variance_proof} next propagates the certificates for accepted candidates through the EMA and derives its tracking accuracy.

\paragraph{Movement enters through staleness.}
\label{prop:drift_bound_appendix}
\label{prop:drift_bound_formal}
For every retained record,
\begin{equation}
    \|\vtheta_t-\vtheta_{a_j}\|
    \le
    \sum_{r=a_j}^{t-1}\ell_r.
    \label{eq:record_path_distance_app}
\end{equation}
Thus $A_t(K)$ already carries the complete path dependence needed by the stored-response estimator. Once a candidate is accepted, Appendix~\ref{app:implicit_variance_proof} uses its realized certificate $\varepsilon_t(K_t)$ together with the EMA-weighted path length; no separate covariance-transport margin is required.

\paragraph{Probability interface.}
\label{rem:centering_effect}
\label{lemma:appendix_fresh_bernstein}
\label{prop:centering_appendix}
The analytical estimator directly averages the paired contributions $X_j=r_j\vu_j$ defined in Eq.~\ref{eq:record_notation_app}. Its statistical input is therefore contained in Lemma~\ref{lemma:variance_detailed}: second moments at acquisition yield the radius $q_2$ with polynomial dependence on confidence, while conditional Bernstein tails yield the radius $q_{\ntxt{B}}$ with logarithmic dependence. All subsequent results condition on the corresponding simultaneous candidate event.

\section{Momentum Tracking and Directional Control}
\label{app:proof_geo}

This section converts the candidate certificates from Appendix~\ref{app:proof_stat} into a certificate for the momentum-consistency gate. The argument is pathwise after the simultaneous candidate event is fixed, so overlapping FIFO buffers and adaptive first-passing selection introduce no additional independence requirement.

\subsection{Debiased EMA Tracking}
\label{app:momentum_alignment}
\label{app:implicit_variance_proof}

Assume $\vm_0=0$ and $\vm_r=\beta_1\vm_{r-1}+(1-\beta_1)\vg_r$. For $t\ge2$, define
\begin{equation}
    z_{t-1}\triangleq1-\beta_1^{t-1},
    \qquad
    \widetilde{\vm}_{t-1}\triangleq\frac{\vm_{t-1}}{z_{t-1}},
    \qquad
    \bar w_{r,t}\triangleq
    \frac{(1-\beta_1)\beta_1^{t-1-r}}{z_{t-1}}.
    \label{eq:debiased_ema_weights_app}
\end{equation}
The weights are nonnegative and sum to one. Moreover, $z_{t-1}>0$, so replacing $\vm_{t-1}$ by $\widetilde{\vm}_{t-1}$ leaves every cosine score unchanged.

\begin{proposition}[Pathwise Debiased-EMA Tracking]
\label{prop:zoaq_tracking_decomp}
\label{prop:tracking_anchor_main}
Suppose the final estimate accepted at step $r$ satisfies
\begin{equation}
    \|\vg_r-h_r\|\le\varepsilon_r.
    \label{eq:accepted_error_app}
\end{equation}
Then
\begin{equation}
    \|\widetilde{\vm}_{t-1}-h_t\|
    \le
    \sum_{r<t}\bar w_{r,t}\varepsilon_r
    L D_t^{\ntxt{ema}},
    \qquad
    D_t^{\ntxt{ema}}
    \triangleq
    \sum_{r<t}\bar w_{r,t}\sum_{s=r}^{t-1}\ell_s.
    \label{eq:debiased_tracking_app}
\end{equation}
In particular, if $\varepsilon_r\le\bar\varepsilon$ and $\ell_s\le\bar\ell$ over the contributing steps, then
\begin{equation}
    \|\widetilde{\vm}_{t-1}-h_t\|
    \le
    \bar\varepsilon+
    L\bar\ell A_{t,\beta_1}
    \le
    \bar\varepsilon+\frac{L\bar\ell}{1-\beta_1},
    \label{eq:uniform_debiased_tracking_app}
\end{equation}
where
\begin{equation}
    A_{t,\beta_1}
    =
    \frac{1-t\beta_1^{t-1}+(t-1)\beta_1^t}
    {(1-\beta_1)(1-\beta_1^{t-1})}.
\end{equation}
\end{proposition}

\begin{proof}
Equation~\ref{eq:debiased_ema_weights_app} gives
\begin{equation}
    \widetilde{\vm}_{t-1}-h_t
    =
    \sum_{r<t}\bar w_{r,t}(\vg_r-h_r)
    +
    \sum_{r<t}\bar w_{r,t}(h_r-h_t).
\end{equation}
The first sum is bounded by $\sum_{r<t}\bar w_{r,t}\varepsilon_r$. Smoothness gives
$\|h_r-h_t\|\le L\sum_{s=r}^{t-1}\ell_s$, which yields Eq.~\ref{eq:debiased_tracking_app}. Under uniform bounds, the first weighted sum is at most $\bar\varepsilon$, while the second is $L\bar\ell\sum_{r<t}\bar w_{r,t}(t-r)$. Evaluating the finite geometric sum gives $A_{t,\beta_1}$, and $A_{t,\beta_1}\le(1-\beta_1)^{-1}$.
\end{proof}

The proposition uses errors of the final estimates, not an independence assumption between them. On the simultaneous event of Lemma~\ref{lemma:variance_detailed}, one may take
\begin{equation}
    \varepsilon_r
    =
    q(n_r(K_r),\delta)
    +\frac{L A_r(K_r)}{n_r(K_r)}.
    \label{eq:accepted_certificate_app}
\end{equation}
Thus acquisition noise and response staleness enter the EMA with their realized geometric weights.

\begin{corollary}[Relative Tracking Level]
\label{prop:anchor_validity_hp}
At any pre-stationary step with $\|h_t\|>0$, define
\begin{equation}
    \kappa_t
    \triangleq
    \frac{
        \sum_{r<t}\bar w_{r,t}\varepsilon_r+LD_t^{\ntxt{ema}}
    }{\|h_t\|}.
    \label{eq:relative_tracking_level_app}
\end{equation}
Then $\|\widetilde{\vm}_{t-1}-h_t\|\le\kappa_t\|h_t\|$. If $\|h_t\|\ge\epsilon$, $\varepsilon_r\le\bar\varepsilon$, and $\ell_s\le\bar\ell$, the explicit sufficient bound is
\begin{equation}
    \kappa_t
    \le
    \frac{\bar\varepsilon+L\bar\ell/(1-\beta_1)}{\epsilon}.
    \label{eq:relative_tracking_uniform_app}
\end{equation}
\end{corollary}

\subsection{Sharp Gate Geometry}
\label{app:implicit_variance}

The next result characterizes both directions of the consistency test. Soundness states what a passing score guarantees; completeness states how accurate a candidate must be in order to pass.

\begin{theorem}[Soundness and Completeness of Momentum Consistency]
\label{thm:implicit_norm_condition_proof}
Let $h\ne0$, let an anchor $m\ne0$ satisfy $\|m-h\|\le\kappa\|h\|$ for $0\le\kappa<1$, and let $g\ne0$.

If $\cos(g,m)\ge\tau$, then
\begin{equation}
    \cos(g,h)
    \ge
    \Delta(\tau,\kappa)
    \triangleq
    \tau\sqrt{1-\kappa^2}
    -\kappa\sqrt{1-\tau^2}.
    \label{eq:gate_soundness_app}
\end{equation}
In particular, every passing candidate is positively aligned with $h$ whenever $\tau>\kappa$.

Conversely, if $\|g-h\|\le\eta\|h\|$ for $0\le\eta<1$, then
\begin{equation}
    \cos(g,m)
    \ge
    \sqrt{1-\eta^2}\sqrt{1-\kappa^2}-\eta\kappa.
    \label{eq:gate_completeness_bound_app}
\end{equation}
Therefore the gate must pass whenever
\begin{equation}
    \eta\le\eta_\star(\tau,\kappa)
    \triangleq
    \sqrt{1-\tau^2}\sqrt{1-\kappa^2}-\tau\kappa.
    \label{eq:eta_star_app}
\end{equation}
\end{theorem}

\begin{proof}
The relative anchor error implies $\angle(m,h)\le\arcsin\kappa$: among all $m$ in the radius-$\kappa\|h\|$ ball centered at $h$, the largest angle is attained by a tangent ray. A passing score gives $\angle(g,m)\le\arccos\tau$. The spherical triangle inequality therefore yields
\begin{equation}
    \angle(g,h)
    \le\arccos\tau+\arcsin\kappa.
\end{equation}
Taking the cosine proves Eq.~\ref{eq:gate_soundness_app}.

Similarly, $\|g-h\|\le\eta\|h\|$ implies $\angle(g,h)\le\arcsin\eta$. Hence
\begin{equation}
    \angle(g,m)\le\arcsin\eta+\arcsin\kappa,
\end{equation}
whose cosine gives Eq.~\ref{eq:gate_completeness_bound_app}. Solving the resulting inequality for $\eta$ yields Eq.~\ref{eq:eta_star_app}.
\end{proof}

Soundness and nontrivial completeness coexist exactly in the band
\begin{equation}
    \kappa<\tau<\sqrt{1-\kappa^2},
    \label{eq:feasible_threshold_band_app}
\end{equation}
which is nonempty if and only if $\kappa<1/\sqrt2$. Both constants are sharp: coplanar vectors placed at the tangent angles used in the proof attain equality.

\subsection{Threshold Interface}
\label{app:adaptive_threshold_theory}

The geometry requires the threshold to remain inside a feasible band; it does not require a particular calibration statistic. A fixed threshold suffices. An adaptive rule can preserve the same guarantee through projection,
\begin{equation}
    \tau_t
    =
    \Pi_{[\underline\tau,\overline\tau]}
    \left(
        \beta_\tau\tau_{t-1}
        +(1-\beta_\tau)s_t^{\ntxt{ref}}
    \right),
    \label{eq:projected_threshold_app}
\end{equation}
where, for a uniform tracking level $\bar\kappa<1/\sqrt2$,
\begin{equation}
    \bar\kappa<\underline\tau
    \le\overline\tau
    <\sqrt{1-\bar\kappa^2}.
    \label{eq:uniform_threshold_band_app}
\end{equation}
Projection separates the guarantee from the selection bias of the score written after first passage. Without a fixed upper bound, updating the threshold only from passing scores makes it nondecreasing and can eventually destroy completeness.

\section{First-Passing Budget Control}
\label{app:controller_bridge_proofs}

The controller inspects a finite increasing grid $\mathcal K$ and stops at the first candidate whose momentum-consistency score reaches the current threshold. The proof below controls this stopping rule directly through the candidate error envelope and the gate geometry. It does not introduce an expected score curve or require earlier candidates to be separated from an oracle threshold.

\subsection{Sufficient Local Budget}

Fix a step $t$ and work on the simultaneous candidate event from Lemma~\ref{lemma:variance_detailed}. Write
\begin{equation}
    e_t(K)
    \triangleq
    q(n_t(K),\delta)
    +\frac{L A_t(K)}{n_t(K)},
    \qquad K\in\mathcal K,
    \label{eq:candidate_envelope_controller_app}
\end{equation}
so that $\|\vg_t(K)-h_t\|\le e_t(K)$ for every inspected candidate. The first-passing result below uses the first grid point that meets a sufficient error level; it does not require the empirical score or a candidate-specific concentration radius to be monotone in $K$.

Suppose the debiased anchor satisfies the uniform relative tracking bound
\begin{equation}
    \|\widetilde{\vm}_{t-1}-h_t\|
    \le\bar\kappa\|h_t\|,
    \qquad
    0\le\bar\kappa<\frac{1}{\sqrt2},
    \label{eq:controller_tracking_app}
\end{equation}
and the threshold lies in
\begin{equation}
    \bar\kappa<\underline\tau
    \le\tau_{t-1}\le\overline\tau
    <\sqrt{1-\bar\kappa^2}.
    \label{eq:controller_threshold_band_app}
\end{equation}
Define the uniform geometry margins
\begin{equation}
    \underline\Delta
    \triangleq
    \Delta(\underline\tau,\bar\kappa)>0,
    \qquad
    \underline\eta
    \triangleq
    \eta_\star(\overline\tau,\bar\kappa)>0.
    \label{eq:uniform_geometry_margins_app}
\end{equation}

\begin{definition}[Sufficient Grid Budget]
\label{def:sufficient_grid_budget_app}
For $\|h_t\|>0$, let
\begin{equation}
    K_t^{\ntxt{suf}}
    \triangleq
    \min\left\{
        K\in\mathcal K:
        e_t(K)\le\underline\eta\|h_t\|
    \right\},
    \label{eq:sufficient_grid_budget_app}
\end{equation}
whenever the set is nonempty.
\end{definition}

\begin{lemma}[A Sufficient Budget Exists Below the Cap]
\label{lemma:sample_boundedness}
\label{lemma:boundedness_main}
If
\begin{equation}
    e_t(K_{\max})
    \le
    \underline\eta\|h_t\|,
    \label{eq:cap_sufficiency_app}
\end{equation}
then $K_t^{\ntxt{suf}}$ exists and $K_t^{\ntxt{suf}}\le K_{\max}$. On a pre-stationary region $\|h_t\|\ge\epsilon$, it is sufficient to replace the right-hand side of Eq.~\ref{eq:cap_sufficiency_app} by $\underline\eta\epsilon$.
\end{lemma}

\begin{proof}
Equation~\ref{eq:cap_sufficiency_app} places $K_{\max}$ in the defining set of Eq.~\ref{eq:sufficient_grid_budget_app}; the minimum therefore exists on the finite grid. The pre-stationary statement follows from $\|h_t\|\ge\epsilon$.
\end{proof}

\subsection{Safe First Passage}

\begin{proposition}[First-Passing Control]
\label{prop:controller_budget_overshoot}
Assume Eqs.~\ref{eq:controller_tracking_app}--\ref{eq:controller_threshold_band_app}, and suppose $K_t^{\ntxt{suf}}$ exists. Let $K_t$ be the first grid point accepted by the consistency test. Then
\begin{equation}
    K_t\le K_t^{\ntxt{suf}},
    \qquad
    \cos(\vg_t(K_t),h_t)
    \ge\underline\Delta>0.
    \label{eq:first_passing_control_app}
\end{equation}
Consequently, if Eq.~\ref{eq:cap_sufficiency_app} holds, the search passes before capped fallback. If a continuous sufficient level is rounded upward to the next grid point, the corresponding implementation bound incurs at most one increment $\Delta K$.
\end{proposition}

\begin{proof}
At $K_t^{\ntxt{suf}}$, the simultaneous candidate event and Definition~\ref{def:sufficient_grid_budget_app} give
\begin{equation}
    \|\vg_t(K_t^{\ntxt{suf}})-h_t\|
    \le\underline\eta\|h_t\|.
\end{equation}
Because $\tau_{t-1}\le\overline\tau$ and the actual tracking level is no larger than $\bar\kappa$, the completeness part of Theorem~\ref{thm:implicit_norm_condition_proof} forces this candidate to pass. The increasing search therefore stops no later, so $K_t\le K_t^{\ntxt{suf}}$.

For the selected candidate, passage gives $\cos(\vg_t(K_t),\widetilde{\vm}_{t-1})\ge\tau_{t-1}\ge\underline\tau$. The soundness part of Theorem~\ref{thm:implicit_norm_condition_proof} then yields
\begin{equation}
    \cos(\vg_t(K_t),h_t)
    \ge
    \Delta(\tau_{t-1},\kappa_t)
    \ge
    \Delta(\underline\tau,\bar\kappa)
    =\underline\Delta.
\end{equation}
Thus an earlier accidental pass can reduce the budget but cannot invalidate the directional guarantee.
\end{proof}

The proposition deliberately provides an upper bound rather than exact recovery of an oracle budget. Exact equality would require a lower separation condition for every earlier score, which is unnecessary for either safety or query savings.

\subsection{Full Buffer and Query Accounting}

Suppose the FIFO buffer after acceptance is full with capacity $C$, each accepted step adds at least $K_{\min}$ fresh records, and the relevant path increments satisfy $\ell_s\le\bar\ell$. When a common nonincreasing radius $q(n,\delta)$ controls averages of $n$ admissible records, let
\begin{equation}
    H_C\triangleq\left\lceil\frac{C}{K_{\min}}\right\rceil.
\end{equation}
Before pruning after acceptance, a candidate with $K$ fresh records has $n_t(K)=C+K$ and $A_t(K)\le CH_C\bar\ell$. Hence
\begin{equation}
    e_C(K)
    \triangleq
    q(C+K,\delta)
    +\frac{LCH_C\bar\ell}{C+K}
    \label{eq:full_buffer_envelope_app}
\end{equation}
is a deterministic sufficient envelope. The cap condition $e_C(K_{\max})\le\underline\eta\epsilon$ makes the no-fallback statement operational using buffer, movement, confidence, and cap parameters alone.

\begin{corollary}[Controller Query Cost]
\label{cor:controller_query_cost}
Let $\mathcal G$ contain the steps on which the simultaneous candidate event, the tracking bound, the threshold band, and cap sufficiency all hold. Then
\begin{equation}
    \sum_{t=1}^{T}K_t
    \le
    \sum_{t\in\mathcal G}K_t^{\ntxt{suf}}
    +K_{\max}(T-|\mathcal G|).
    \label{eq:controller_query_cost_app}
\end{equation}
If the joint conditions hold throughout a window, the second term vanishes. Moreover, whenever the reuse condition in Eq.~\ref{eq:reuse_benefit_condition_app} holds at the sufficient budget for fresh directions alone, the sufficient budget with reuse is no larger.
\end{corollary}

\begin{proof}
Proposition~\ref{prop:controller_budget_overshoot} gives $K_t\le K_t^{\ntxt{suf}}$ on $\mathcal G$. Every other step is bounded by the deterministic cap $K_{\max}$. Summing these two cases proves Eq.~\ref{eq:controller_query_cost_app}. The final comparison follows by evaluating the reuse envelope at the sufficient budget for fresh directions alone and applying Eq.~\ref{eq:reuse_benefit_condition_app}.
\end{proof}

\section{Descent and Query Accounting}
\label{app:proof_complexity}

This section closes the theorem chain for the normalized-update model. The result is local to a pre-stationary window on which the estimation, movement, threshold, and cap conditions established above hold. It does not identify coordinatewise adaptive preconditioning with global Euclidean normalization.

\subsection{Normalized Descent}
\label{app:convergence_analysis}

\begin{assumption}[Normalized-Update Analysis Protocol]
\label{ass:scale_invariant_app}
For the descent analysis, the selected nonzero candidate defines
\begin{equation}
    \vd_t\triangleq\frac{\vg_t}{\|\vg_t\|},
    \qquad
    \vtheta_{t+1}=\vtheta_t-\alpha_t\vd_t,
    \qquad
    \|\vd_t\|=1.
    \label{eq:normalized_update_app}
\end{equation}
This is the formal optimizer scope of the descent theorem. The experiments use the stated R-AdaZO backbone; a coordinatewise second-moment preconditioner need not preserve Euclidean angles and is not covered without an additional spectral condition.
\end{assumption}

\begin{lemma}[Normalized Descent]
\label{lemma:descent}
\label{lemma:descent_restate}
Under Assumption~\ref{ass:smoothness_app}, if $\cos(\vd_t,h_t)\ge\delta_t>0$, then
\begin{equation}
    F_\mu(\vtheta_{t+1})-F_\mu(\vtheta_t)
    \le
    -\alpha_t\delta_t\|h_t\|
    +\frac{L\alpha_t^2}{2}.
    \label{eq:normalized_descent_app}
\end{equation}
\end{lemma}

\begin{proof}
The smoothness inequality and Eq.~\ref{eq:normalized_update_app} give
\begin{align}
    F_\mu(\vtheta_{t+1})-F_\mu(\vtheta_t)
    &\le
    -\alpha_t\langle h_t,\vd_t\rangle
    +\frac{L\alpha_t^2}{2}\|\vd_t\|^2 \\
    &\le
    -\alpha_t\delta_t\|h_t\|
    +\frac{L\alpha_t^2}{2}.
\end{align}
\end{proof}

\subsection{Operational Pre-Stationary Window}

The next theorem gathers the preceding interfaces into one statement. Its probability is exactly that of the simultaneous candidate event used to instantiate the error radii; after conditioning on that event, every step of the argument is deterministic.

\begin{theorem}[Certified Descent and Local Query Control]
\label{thm:total_complexity_restate}
Consider a finite window beginning after warm-up in which $\|h_t\|\ge\epsilon$ and the candidate event of Lemma~\ref{lemma:variance_detailed} holds simultaneously. Suppose:
\begin{enumerate}
    \item the final estimates entering the debiased EMA satisfy Eq.~\ref{eq:accepted_certificate_app};
    \item the resulting tracking levels obey $\kappa_t\le\bar\kappa<1/\sqrt2$;
    \item the threshold is fixed or projected into the band in Eq.~\ref{eq:controller_threshold_band_app};
    \item the cap satisfies $e_t(K_{\max})\le\underline\eta\epsilon$; and
    \item the update follows Eq.~\ref{eq:normalized_update_app} with $0<\alpha_t\le\underline\Delta\epsilon/L$.
\end{enumerate}
Then every search in the window passes by $K_t^{\ntxt{suf}}\le K_{\max}$, and its selected direction satisfies
\begin{equation}
    \cos(\vd_t,h_t)\ge\underline\Delta.
    \label{eq:window_alignment_app}
\end{equation}
Consequently,
\begin{equation}
    F_\mu(\vtheta_{t+1})-F_\mu(\vtheta_t)
    \le
    -\frac{\alpha_t\underline\Delta\epsilon}{2}.
    \label{eq:window_descent_app}
\end{equation}
\end{theorem}

\begin{proof}
Proposition~\ref{prop:tracking_anchor_main} and the certificates for accepted candidates give the stated relative tracking level. The threshold band and cap condition then satisfy the premises of Proposition~\ref{prop:controller_budget_overshoot}, which rules out capped fallback, gives $K_t\le K_t^{\ntxt{suf}}$, and proves Eq.~\ref{eq:window_alignment_app}. Lemma~\ref{lemma:descent_restate} yields
\begin{equation}
    F_\mu(\vtheta_{t+1})-F_\mu(\vtheta_t)
    \le
    -\alpha_t\underline\Delta\epsilon
    +\frac{L\alpha_t^2}{2}.
\end{equation}
The step-size condition bounds the curvature term by $\alpha_t\underline\Delta\epsilon/2$, proving Eq.~\ref{eq:window_descent_app}.
\end{proof}

\subsection{Finite-Horizon Accounting}

\begin{corollary}[Pre-Stationary Steps and Fresh Directions]
\label{cor:controller_window_complexity_app}
Suppose the conditions of Theorem~\ref{thm:total_complexity_restate} hold until the first step $T_\epsilon$ with $\|h_{T_\epsilon}\|<\epsilon$, and suppose $F_\mu$ is bounded below by $F_\star$. If $\alpha_t=\alpha=c\epsilon$ with $0<c\le\underline\Delta/L$, then
\begin{equation}
    T_\epsilon-t_1
    \le
    \frac{2\bigl(F_\mu(\vtheta_{t_1})-F_\star\bigr)}
    {c\underline\Delta\epsilon^2}.
    \label{eq:prestationary_steps_app}
\end{equation}
The cumulative number of fresh directions satisfies
\begin{equation}
    \sum_{t=t_1}^{T_\epsilon-1}K_t
    \le
    \sum_{t=t_1}^{T_\epsilon-1}K_t^{\ntxt{suf}}
    \le
    (T_\epsilon-t_1)K_{\max}.
    \label{eq:prestationary_queries_app}
\end{equation}
\end{corollary}

\begin{proof}
Summing Eq.~\ref{eq:window_descent_app} from $t_1$ to $T_\epsilon-1$ and using the lower bound $F_\star$ gives Eq.~\ref{eq:prestationary_steps_app}. Proposition~\ref{prop:controller_budget_overshoot} gives $K_t\le K_t^{\ntxt{suf}}\le K_{\max}$ at every step in the window, which proves Eq.~\ref{eq:prestationary_queries_app}.
\end{proof}

The sharper quantity in Eq.~\ref{eq:prestationary_queries_app} is the sum of local sufficient budgets, not the cap product. For the full-buffer envelope in Eq.~\ref{eq:full_buffer_envelope_app}, each $K_t^{\ntxt{suf}}$ is obtained by inverting
\begin{equation}
    q(C+K,\delta)
    +\frac{LCH_C\bar\ell}{C+K}
    \le\underline\eta\epsilon.
    \label{eq:full_buffer_inverse_app}
\end{equation}
This expression exposes the two ways in which queries help: they add acquisition evidence and dilute the staleness carried by the retained records.

\begin{corollary}[Reuse Versus Fresh-Only Accounting]
\label{cor:main_comparison}
\label{cor:comparison_appendix}
Let $K_{t,\ntxt{fresh}}^{\ntxt{suf}}$ be the first grid point satisfying $q(K,\delta)\le\underline\eta\|h_t\|$. If the reuse-benefit condition in Eq.~\ref{eq:reuse_benefit_condition_app} holds at this budget, then
\begin{equation}
    K_t^{\ntxt{suf}}
    \le
    K_{t,\ntxt{fresh}}^{\ntxt{suf}}.
    \label{eq:reuse_fresh_budget_comparison_app}
\end{equation}
The inequality is strict whenever an earlier reuse candidate already satisfies the sufficient-error condition.
\end{corollary}

\begin{proof}
At $K_{t,\ntxt{fresh}}^{\ntxt{suf}}$, Eq.~\ref{eq:reuse_benefit_condition_app} makes the reuse envelope no larger than the fresh-only envelope, which is at most $\underline\eta\|h_t\|$. This grid point is therefore feasible for Definition~\ref{def:sufficient_grid_budget_app}; taking the minimum proves the result.
\end{proof}

The conclusion is a conditional local saving, not an unconditional dominance statement. When response staleness exceeds the concentration gain, the comparison reverses, which is precisely why Appendix~\ref{app:proof_stat} predicts an interior useful history capacity.

\section{Diagnostics for Variance-Based Adaptation}
\label{app:statzo_analysis}

The main text compares strategies for query allocation under a common stabilized backbone. This section explains why adaptation based on variance remains costly in the tested regimes after separating instability in the optimization backbone from the sampling rule itself.

\subsection{StatZO in Ultra-High-Dimensional Fine-Tuning}

The original implementation of StatZO \citep{bollapragada2024adaptive} relies on standard zeroth-order optimizers (e.g., ZO-SGD or ZO-AdaMM). While theoretically sound for lower-dimensional problems, this configuration exhibits severe instability in ultra-high-dimensional non-convex landscapes ($d \approx 10^5$).

As illustrated in Figure~\ref{fig:statzo_failure} (referencing the ablation in Section~\ref{sec:llm_finetuning}), the original StatZO configuration fails to descend effectively on SST-2, stagnating at high loss values ($>3.5$). In this regime, estimating the population variance $\sigma^2$ from a mini-batch in such high dimensions produces unstable SNR signals, which in turn destabilize the optimizer's second-moment estimates.

\begin{figure}[t]
\begin{center}
\centerline{\includegraphics[width=0.56\textwidth]{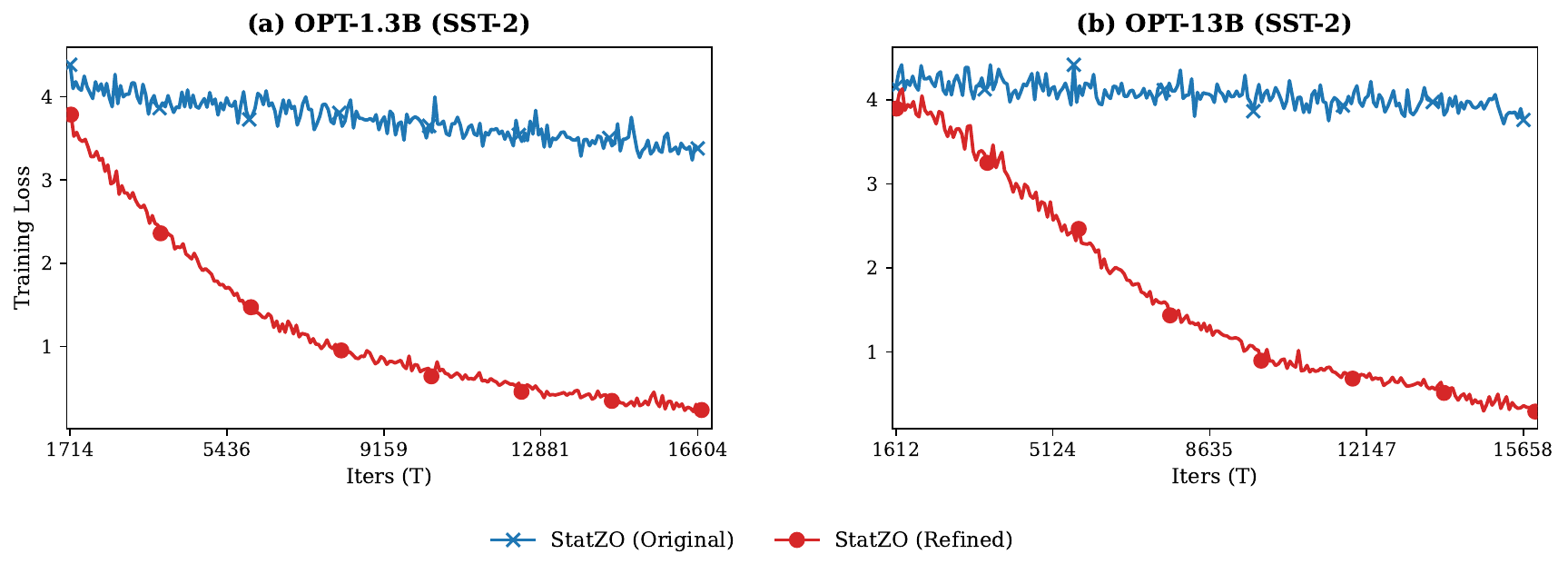}}
\caption{Behavior of the original StatZO configuration on SST-2. Without the stabilized update rule used in our main experiments, the original StatZO configuration stagnates at a high loss, motivating the standardized optimizer backbone used in the main comparison.}
\label{fig:statzo_failure}
\end{center}
\end{figure}

\subsection{Stabilized StatZO in Black-Box Attacks}

To evaluate the adaptive sampling strategy more cleanly, we examine StatZO in two stages: its original formulation and a stabilized version using the same backbone as the main comparisons. This two-step analysis helps separate backbone effects from sampling effects.

First, we evaluate the original StatZO equipped with the standard ZO-AdaMM backbone. As shown in Table~\ref{tab:statzo_vanilla_sweep}, the configuration requiring the fewest queries among the tested caps is $K_{\max}=4$, with 4,382 queries. However, the overall efficiency is hampered by the instability of the standard update rule in high-dimensional settings ($d=784$).

\begin{table}[ht]
\caption{Original StatZO hyperparameter sweep. All configurations use $K_{\min}=2$ and variance-to-norm threshold 0.75. Among the tested settings for this version, $K_{\max}=4$ gives the lowest query cost, but the absolute cost remains high in the reported setting.}
\label{tab:statzo_vanilla_sweep}
\vskip 0.15in
\begin{center}
\begin{footnotesize}
\begin{sc}
\setlength{\tabcolsep}{3.5pt}
\renewcommand{\arraystretch}{0.95}
\resizebox{0.88\textwidth}{!}{
\begin{tabular}{cccccc}
\toprule
$K_{\max}$ & Success & Avg & Std & Speedup vs ZoAQ \\
\midrule
2 & 100\% & 6,212 & 675 & 0.05$\times$ \\
4 & 100\% & 4,382 & 511 & 0.07$\times$ \\
6 & 100\% & 7,349 & 1,714 & 0.04$\times$ \\
8 & 100\% & 7,424 & 588 & 0.04$\times$ \\
10 & 100\% & 10,022 & 1,316 & 0.03$\times$ \\
\bottomrule
\end{tabular}
}
\end{sc}
\end{footnotesize}
\end{center}
\vskip -0.1in
\end{table}

To ensure that the comparison targets the sampling strategy rather than optimizer defects, we also evaluate a stabilized StatZO variant with the R-AdaZO backbone. Table~\ref{tab:statzo_refined_sweep} shows the effect of this standardization: the query cost drops from roughly 4,382 to 1,996. The tested cap requiring the fewest queries is $K_{\max}=2$, while increasing the adaptive budget ($K_{\max} \ge 4$) degrades performance. Even this strongest tested stabilized configuration remains well above the reported ZoAQ cost of 320 queries, suggesting that the query overhead required to validate geometry can outweigh the benefit of variance reduction in this setting.

\begin{table}[ht]
\caption{Stabilized StatZO hyperparameter sweep. Equipping StatZO with the R-AdaZO backbone improves performance relative to Table~\ref{tab:statzo_vanilla_sweep}. Among the tested caps, the lowest query count occurs at $K_{\max}=2$, and the method remains less query-efficient than ZoAQ in this setting.}
\label{tab:statzo_refined_sweep}
\vskip 0.15in
\begin{center}
\begin{footnotesize}
\begin{sc}
\setlength{\tabcolsep}{3.5pt}
\renewcommand{\arraystretch}{0.95}
\resizebox{0.88\textwidth}{!}{
\begin{tabular}{cccccc}
\toprule
$K_{\max}$ & Success & Avg & Std & Speedup vs ZoAQ \\
\midrule
2 & 100\% & 1,996 & 1,231 & 0.16$\times$ \\
4 & 100\% & 4,089 & 953 & 0.08$\times$ \\
6 & 100\% & 5,137 & 2,004 & 0.06$\times$ \\
8 & 100\% & 3,937 & 1,802 & 0.08$\times$ \\
10 & 100\% & 4,572 & 2,189 & 0.07$\times$ \\
\bottomrule
\end{tabular}
}
\end{sc}
\end{footnotesize}
\end{center}
\vskip -0.1in
\end{table}

\paragraph{Cross-dataset pattern: CIFAR-10.}
We replicate the StatZO sweep on CIFAR-10 ($d=3072$) to examine whether the same pattern persists across dimensions. As shown in Table~\ref{tab:statzo_cifar_sweep}, the tested configuration requiring the fewest queries remains $K_{\max}=2$, consistent with MNIST. Even with this tuned setting, StatZO achieves 964 queries, still 1.54$\times$ slower than ZoAQ (625 queries with $N_{\ntxt{hist}}=8$).

This analysis across datasets suggests that the same overhead from validation persists across problem scales: StatZO's variance estimation overhead prevents it from matching ZoAQ's efficiency in the reported settings.

\begin{table}[ht]
\caption{StatZO hyperparameter sweep on CIFAR-10 ($d=3072$). Consistent with MNIST, the tested configuration requiring the fewest queries is $K_{\max}=2$, yet StatZO still lags behind ZoAQ (625 queries) in the reported setting.}
\label{tab:statzo_cifar_sweep}
\vskip 0.15in
\begin{center}
\begin{footnotesize}
\begin{sc}
\setlength{\tabcolsep}{3.5pt}
\renewcommand{\arraystretch}{0.95}
\resizebox{0.84\textwidth}{!}{
\begin{tabular}{ccccc}
\toprule
$K_{\max}$ & Success & Avg Queries & Std & Speedup vs ZoAQ \\
\midrule
2 & 100\% & 964 & 28 & 0.65$\times$ \\
4 & 100\% & 1,134 & 73 & 0.55$\times$ \\
6 & 100\% & 1,148 & 37 & 0.54$\times$ \\
8 & 100\% & 1,229 & 68 & 0.51$\times$ \\
10 & 100\% & 1,361 & 41 & 0.46$\times$ \\
\bottomrule
\end{tabular}
}
\end{sc}
\end{footnotesize}
\end{center}
\vskip -0.1in
\end{table}

\subsection{Validation Overhead}

\paragraph{Remark: Why history reuse is not folded into StatZO.}
A natural question is whether history reuse could also be added to StatZO. That modification is not theoretically neutral. As formalized by \citet{bollapragada2024adaptive}, the adaptive sampling analysis behind the Norm Condition is built around fresh i.i.d. samples for estimating local accuracy and variance quantities.

Incorporating historical queries $\{ (\vu, y) \in \mathcal{H}_{t-1} \}$ into the current estimate $\vg_t$ introduces cross-step correlations and transport bias because $\vtheta_t \neq \vtheta_{t-\tau}$. It would therefore require a different variance-estimation analysis rather than being a drop-in improvement to StatZO's sample-variance estimator. ZoAQ's momentum-consistency condition is instead built around this correlation structure.

The empirical gap is therefore not only an implementation issue; it also reflects a structural bottleneck from validation overhead.

To estimate the variance $\mathbb{V}[\vg]$, StatZO requires a reliable estimate of the function value $F(\vtheta_t)$ at the current center point. This incurs a fixed cost of at least one extra query per iteration purely for statistical estimation.
\begin{itemize}
    \item Overhead ratio: In a low-budget one-sided regime with two perturbed directional calls, an additional center-value query changes the accounting from $2$ directional calls to $2 \ntxt{ directions} + 1 \ntxt{ baseline} = 3$ oracle calls, a 50\% overhead relative to the directional calls alone.
    \item ZoAQ: The momentum-consistency test uses $\vm_{t-1}$, a historical variable stored in memory, and requires no additional queries to validate the search direction.
\end{itemize}

This analysis makes the practical implication explicit: when only a few queries are used, the cost of validating geometry can itself become a meaningful part of the total budget.

\subsection{Reference Stability}
\label{app:stability_gap}

Beyond query overhead, there is a second statistical reason for StatZO's weaker performance in this regime: the instability of the reference signal itself. Both methods rely on a reference vector to judge the quality of the current estimate $\vg_t$.
\begin{itemize}
    \item StatZO uses the sample variance $V_k$ (or effectively, the batch mean $\bar\vg_k$ relative to variance) estimated from a small batch of size $k$.
    \item ZoAQ uses the historical momentum $\vm_{t-1}$ as the reference anchor.
\end{itemize}

The comparisons above are empirical diagnostics of the tested configurations. The pathwise treatment of the EMA under overlapping accepted histories is given in Appendix~\ref{app:proof_geo}.

\section{Experimental Protocols}
\label{app:experimental}

This appendix collects the implementation details, protocol standardization choices, and additional result tables referenced in the main text. Its role is to make the experimental comparison reproducible without expanding the main paper's evidence budget.

\begin{table}[ht]
\caption{Experimental evidence map. Each regime tests a distinct part of the mechanism chain, with claims stated at the level supported by the source ledger.}
\label{tab:experiment_claim_map}
\vskip 0.05in
\begin{center}
\begin{footnotesize}
\begin{tabular}{P{0.23\textwidth}P{0.38\textwidth}P{0.29\textwidth}}
\toprule
Regime & Evidence used & Claim boundary \\
\midrule
Synthetic mechanisms & Convergence traces, sweep over fixed budgets, ablation of the norm trigger, and cosine diagnostics. & Query reductions are measured relative to the fixed protocol using 1.2M queries; cosine diagnostics measure observable consistency with the anchor. \\
Black-box attacks & MNIST and CIFAR-10 attacks under a common stabilized R-AdaZO backbone. & The result concerns black-box optimization under query constraints, with success and query counts reported under the stated protocol. \\
LLM fine-tuning & OPT-1.3B and OPT-13B LoRA fine-tuning measured by cumulative forward evaluations and task metrics. & Savings are FE reductions with loss and accuracy trade-offs that vary by task. \\
Stress checks & Controlled objective switch, history window, reset policy, threshold, and StatZO diagnostics. & These checks show how the mechanism behaves across the tested settings and parameter choices. \\
\bottomrule
\end{tabular}
\end{footnotesize}
\end{center}
\vskip -0.1in
\end{table}

\subsection{Algorithm and Cost Accounting}
\label{app:algorithmic_summary}

Algorithm~\ref{alg:zoaq_main} in the main paper is the authoritative pseudocode. This appendix section records the implementation-level accounting assumptions used by that algorithm in our experiments, together with the protocol choices needed for reproducibility.

\paragraph{Query accounting.}
The momentum-consistency test itself introduces no additional oracle queries beyond those already used to construct $\vg_t$. In attacks, a one-sided forward-difference step with $K_t$ sampled directions requires one baseline evaluation plus $K_t$ perturbed evaluations. In synthetic benchmarks and LLM fine-tuning, the two-sided protocol requires two evaluations per sampled direction. The practical saving provided by ZoAQ is therefore a reduction in oracle calls rather than auxiliary computation: the extra cost is primarily memory, since the method retains recent direction-response pairs in order to induce query-reuse coupling. In seed-based implementations, this footprint can be reduced because directions need not be stored explicitly.

These accounting rules vary by protocol. Within each regime, all methods are compared under the same estimator protocol and the same cost definition; across regimes, the reported synthetic query savings, attack query counts, and LLM savings in forward evaluations measure efficiency for their respective tasks rather than one absolute cost scale.

\subsection{Compute, Assets, and Scope}
\label{app:reporting_notes}

\paragraph{Compute resources.}
Synthetic and attack experiments were run on standard CPU-accessible hardware or a single GPU. The main LLM fine-tuning runs were executed on a server with four NVIDIA RTX 4090 GPUs (24GB each), which was sufficient for the OPT-1.3B and OPT-13B LoRA runs using forward passes only. We report oracle calls or forward evaluations as the primary cost metric because the paper studies query allocation rather than systems throughput; elapsed time also depends on model loading, the data pipeline, and scheduling across runs.

\paragraph{Existing assets and licenses.}
The experimental study uses standard public datasets and model checkpoints, including MNIST, CIFAR-10, GLUE/SST-2, COPA, and OPT checkpoints, under their respective public terms of use. We rely on these assets only for benchmarking and do not redistribute modified versions in the paper artifact.

\begin{table}[ht]
\caption{Existing assets used in the experiments. When the original source or benchmark metadata does not state a single explicit license, we report that status rather than assigning a license name.}
\label{tab:asset_licenses}
\vskip 0.15in
\begin{center}
\begin{small}
\setlength{\tabcolsep}{3pt}
\renewcommand{\arraystretch}{1.08}
\begin{tabular}{P{0.13\textwidth}P{0.14\textwidth}P{0.24\textwidth}P{0.18\textwidth}P{0.21\textwidth}}
\toprule
Asset & Version & Source & License & Terms-of-use note \\
\midrule
MNIST & Original train/test split & MNIST database \citep{lecun1998mnist} & Source terms; no explicit single license & Public benchmark use; no data redistribution \\
CIFAR-10 & CIFAR-10 Python version & CIFAR page and report \citep{krizhevsky2009learning} & Source terms; no explicit single license & Public benchmark use; no data redistribution \\
GLUE/SST-2 & GLUE SST-2 task & GLUE benchmark \citep{wang2018glue} & GLUE/SST-2 source terms & Downstream evaluation only \\
COPA & COPA benchmark task & COPA release \citep{roemmele2011choice} & COPA source terms & Downstream evaluation only \\
OPT checkpoints & OPT-1.3B and OPT-13B & Meta OPT release \citep{zhang2022opt}; HF cards & OPT model license & LoRA fine-tuning; no checkpoint redistribution \\
\bottomrule
\end{tabular}
\end{small}
\end{center}
\vskip -0.1in
\end{table}

\paragraph{Broader impact.}
The attack experiments are included to study query efficiency in a controlled black-box setting, not to advocate misuse. The same techniques are also relevant to robustness evaluation and to black-box optimization problems where gradients are unavailable for benign reasons. We therefore view the main broader-impact consideration as responsible disclosure and evaluation practice rather than a new capability claim.

\subsection{Baseline Standardization}

\paragraph{Unified Optimizer Backbone.}
As noted in Section~\ref{sec:experiments}, standard zeroth-order optimizers often suffer from diverging second-moment estimates in high dimensions. To reduce this confounding factor in the main attack and stabilized adaptive comparisons, we use the R-AdaZO rectified update rule as the common backbone for R-AdaZO, stabilized StatZO, ZoAR-style reuse diagnostics, and ZoAQ. Legacy ZO-AdaMM and MeZO appear only in the protocol roles listed in Table~\ref{tab:baseline_protocol_map}. This standardization makes the comparison focus on sampling strategy rather than claiming uniform dominance over every original adaptive ZOO implementation.

\paragraph{Baselines Configuration.}
We compare ZoAQ against the following representative methods:
\begin{itemize}
    \item R-AdaZO (Fixed-Budget SOTA): Uses a static query budget $K$ per iteration. We set $K=10$ for synthetic benchmarks and $K=2$ or $4$ for high-dimensional tasks, following the optimal settings reported in \citep{shu2025refining}.
    \item Refined StatZO (Norm-Based SOTA): A stabilized version of StatZO \citep{bollapragada2024adaptive}. We integrate it with the R-AdaZO backbone to prevent divergence. It adapts $K_t \in [K_{\min}, K_{\max}]$ based on the norm condition $\|\vg_t-\vm_{t-1}\|/\|\vm_{t-1}\| \le c_{\ntxt{norm}}$, with $c_{\ntxt{norm}}=0.5$ as the default tolerance.
    \item ZoAQ (Ours): Utilizes the momentum-consistency condition $\cos(\vg_t, \vm_{t-1}) \ge \tau_{t-1}$ to trigger budget expansion. The effective threshold $\tau_t$ is updated via EMA.
\end{itemize}

Table~\ref{tab:baseline_protocol_map} lists every baseline family that appears in the figures or tables and records its protocol role. When a baseline has a tunable cap, we report the sweep or selection rule rather than comparing against an unreported hand-picked setting.

\begin{table}[ht]
\caption{Baseline and protocol map. The table separates method role, update backbone, query rule, and selection rule so that the experiments compare strategies for query allocation under the intended protocol rather than mixing incompatible baselines.}
\label{tab:baseline_protocol_map}
\vskip 0.15in
\begin{center}
\begin{scriptsize}
\setlength{\tabcolsep}{3pt}
\renewcommand{\arraystretch}{1.08}
\resizebox{\textwidth}{!}{
\begin{tabular}{P{0.12\textwidth}P{0.20\textwidth}P{0.20\textwidth}P{0.23\textwidth}P{0.19\textwidth}}
\toprule
Method & Role in experiments & Backbone/update rule & Query rule & Selection/reporting rule \\
\midrule
ZO-AdaMM & Legacy synthetic reference & Adam-style ZOO & Fixed $K$ in source run & Reference curve only \\
R-AdaZO & Fixed-budget baseline & Rectified adaptive ZOO & Fixed $K$: synthetic 10, attack 2, LLM 4 & Protocol-fixed \\
ZoAR & Reuse diagnostic & History reuse without ZoAQ gate & Source-recorded reuse/fixed budget & Separates reuse from control \\
StatZO & Statistics-based adaptive baseline & Stabilized R-AdaZO; original diagnostic only & Synthetic $[1,10]$; attack cap sweep; LLM $[2,4]$ & Best cap in sweep \\
MeZO & LLM fixed-budget baseline & Forward-only MeZO/LoRA & Fixed $K=4$ & LLM FE reference \\
ZoAQ & Ours & Stabilized adaptive ZO + query-reuse gate & Synthetic $[1,10]$; attack $[1,2]$; LLM $[1,4]$ or $[1,2]$ & Matched seeds; no extra validation \\
\bottomrule
\end{tabular}
}
\end{scriptsize}
\end{center}
\vskip -0.1in
\end{table}

\subsection{Hyperparameter Settings}
\label{app:hyperparams}

Synthetic experiments were conducted over $T=60,000$ iterations in a $d=100$ dimensional space, with results averaged over 5 independent runs. Random seeds are matched across methods within each synthetic run; figures plot mean trajectories without uncertainty bands for readability, while tables report means and standard deviations where the corresponding columns are shown. Attack and LLM configurations are specified in their dedicated subsections below.

\paragraph{Gradient Estimator Configuration.} 
Consistent with the protocol defined in Sec.~\ref{sec:preliminaries}:
\begin{itemize}
    \item Synthetic Benchmarks ($d=100$): We use the two-sided symmetric-difference estimator to minimize gradient approximation error. Thus, R-AdaZO with $K=10$ directions corresponds to 20 queries/step (Total $\approx$ 1.2M queries over 60k steps).
    \item High-Dimensional Attacks ($d \ge 784$): Following the R-AdaZO protocol \citep{shu2025refining}, we use the one-sided forward-difference estimator. For example, R-AdaZO with $K=2$ directions corresponds to 3 queries/step ($1 \ntxt{ baseline} + 2 \ntxt{ perturbed evaluations}$).
    \item LLM Fine-Tuning: We use the two-sided protocol with forward passes only and report forward evaluations (FE). A step with $K_t$ directions costs $2K_t$ FE, as detailed in Appendix~\ref{app:llm_configurations}.
\end{itemize}

\paragraph{Common Hyperparameters.}
A set of common hyperparameters was shared across most algorithms to ensure a fair comparison. These are summarized in Table~\ref{tab:common_hyperparams}.

\begin{table}[ht]
\caption{Common hyperparameters used across all algorithms.}
\label{tab:common_hyperparams}
\vskip 0.15in
\begin{center}
\begin{small}
\begin{sc}
\begin{tabular}{llc}
\toprule
Parameter & Description & Value \\
\midrule
$T$ & Total iterations & 60,000 \\
$d$ & Dimension & 100 \\
num\_runs & Number of independent runs & 5 \\
$\eta$ (lr) & Learning rate & 0.001 \\
$(\beta_1, \beta_2)$ & Momentum coefficients & (0.9, 0.99) \\
$\mu$ & Smoothing parameter & 0.005 \\
$\epsilon$ (eps) & Term for numerical stability & 1e-8 \\
\bottomrule
\end{tabular}
\end{sc}
\end{small}
\end{center}
\vskip -0.1in
\end{table}

\paragraph{Algorithm-Specific Configurations.}
The specific settings for each algorithm, including both non-adaptive and adaptive methods, are detailed in Table~\ref{tab:algo_specific_hyperparams}. The query budget per iteration is denoted by $K_t$. For non-adaptive methods, this is a fixed value. For adaptive methods, it varies within a specified range.

\begin{table}[ht]
\caption{Algorithm-specific hyperparameter settings.}
\label{tab:algo_specific_hyperparams}
\vskip 0.15in
\begin{center}
\begin{small}
\begin{sc}
\begin{tabular}{llc}
\toprule
Algorithm & Parameter & Value \\
\midrule
\multicolumn{3}{l}{Non-Adaptive Baselines} \\
\midrule
R-AdaZO  & Fixed queries per iteration, $K_t$ & 10 \\
\midrule
\multicolumn{3}{l}{Statistics-Based Adaptive SOTA} \\
\midrule
StatZO   & Query range, $K_t \in [K_{\min}, K_{\max}]$ & [1, 10] \\
         & Variance-to-norm threshold & 0.5 \\
\midrule
\multicolumn{3}{l}{Our Proposal} \\
\midrule
ZoAQ   & Query range, $K_t \in [K_{\min}, K_{\max}]$ & [1, 10] \\
         & Initial EMA threshold, $\tau_0$ & 1.0 \\
         & EMA factor, $\beta_\tau$ & 0.9 \\
         & History Window Size, $N_{\ntxt{hist}}$ & 8 \\
         & Budget reset policy & 'min' (Greedy Reset) \\
\bottomrule
\end{tabular}
\end{sc}
\end{small}
\end{center}
\vskip -0.1in
\end{table}

\subsection{Attack Protocol}
\label{app:experimental_setup}

\paragraph{Experimental Protocol.}
For the black-box adversarial attack experiments (Section~\ref{sec:exp_attack}), we follow the experimental protocol established in \citep{shu2025refining}. We evaluate on two standard vision benchmarks:
\begin{itemize}
    \item MNIST ($d=784$) \citep{lecun1998mnist}: A CNN with two convolutional layers (32 and 64 filters, kernel size 3$\times$3) followed by max-pooling and two fully connected layers (128 and 10 units), achieving 98.5\% test accuracy. Perturbation constraint $\epsilon=0.2$.
    \item CIFAR-10 ($d=3072$) \citep{krizhevsky2009learning}: A deeper CNN achieving 85\% test accuracy. Perturbation constraint $\epsilon=0.03$.
\end{itemize}
The attack objective is to find a minimal perturbation $\vdelta$ such that the model misclassifies $\vx+\vdelta$ while satisfying the $\ell_\infty$ constraint.

\paragraph{MNIST Hyperparameter Settings.}
All methods use identical common hyperparameters: learning rate $\eta=0.01$, momentum coefficients $(\beta_1, \beta_2)=(0.9, 0.99)$, perturbation scale $\mu=0.5$, and Adam's numerical stability constant $\varepsilon_{\ntxt{Adam}}=10^{-8}$ \citep{kingma2014adam}. We evaluate over 5 independent runs with different random seeds (101, 202, 303, 404, 505), with a maximum of 20,000 iterations per attack.

For specific algorithms:
\begin{itemize}
    \item Non-adaptive methods (R-AdaZO, ZO-AdaMM) use a fixed query budget of $K=2$ per iteration.
    \item Methods with history reuse (ZoAQ and ZoAR) use a history buffer size $N_{\ntxt{hist}}=8$ (the tested value requiring the fewest queries in Table~\ref{tab:ablation_nhist}).
    \item ZoAQ is configured with $K_{\min}=1$, $K_{\max}=2$, initial threshold $\tau_0=1.0$, EMA factor $\beta_\tau=0.9$, and the `min' (greedy) reset policy.
\end{itemize}

\paragraph{CIFAR-10 Hyperparameter Settings.}
For the CIFAR-10 dataset ($d=3072$), we use the following configurations:
\begin{itemize}
    \item Learning rate: $\eta=0.0001$ (lower than MNIST due to higher sensitivity)
    \item Perturbation scale: $\mu=0.005$ (smaller due to higher sensitivity)
    \item History window: $N_{\ntxt{hist}}=8$ for ZoAQ (unified with MNIST for simplicity)
    \item StatZO: tested setting $K_{\max}=2$ (the cap requiring the fewest queries in Table~\ref{tab:statzo_cifar_sweep})
    \item All other parameters remain identical to MNIST
\end{itemize}
The CIFAR-10 table uses the same five-run reporting protocol as MNIST, with matched random seeds across methods within each run.

\paragraph{Fixed-Budget History-Reuse Control.}
History reuse improves query efficiency, but it does not recover the full benefit of adaptive allocation. Table~\ref{tab:attack_history_reuse_control} compares fixed-budget ZoAR with R-AdaZO and ZoAQ under the common stabilized backbone. All methods achieve 100\% success. Relative to ZoAR at $K=2$, ZoAQ uses 56.9\% fewer queries on MNIST and 17.3\% fewer on CIFAR-10.

\begin{table}[ht]
\caption{Fixed-budget history-reuse control for black-box attacks. Values are average queries $\pm$ standard deviation; all methods achieve 100\% success.}
\label{tab:attack_history_reuse_control}
\vskip 0.15in
\begin{center}
\begin{small}
\begin{sc}
\setlength{\tabcolsep}{4.0pt}
\begin{tabular}{lcccc}
\toprule
Dataset & R-AdaZO ($K=2$) & ZoAR ($K=1$) & ZoAR ($K=2$) & ZoAQ \\
\midrule
MNIST & $1527 \pm 701$ & $1089.2 \pm 92.3$ & $742.8 \pm 99.9$ & $320 \pm 45$ \\
CIFAR-10 & $1006 \pm 100$ & $802.0 \pm 38.0$ & $755.4 \pm 39.1$ & $625 \pm 15$ \\
\bottomrule
\end{tabular}
\end{sc}
\end{small}
\end{center}
\vskip -0.1in
\end{table}

\subsection{LLM Fine-Tuning Protocol}
\label{app:llm_configurations}

For the LLM fine-tuning experiments described in Section~\ref{sec:llm_finetuning}, we use the hyperparameter settings listed in Table~\ref{tab:llm_hyperparams}. The base models are OPT checkpoints \citep{zhang2022opt}, and the downstream tasks include SST-2 from GLUE \citep{wang2018glue} and COPA.

The main ZoAQ configuration uses the moderate cap $K_t \in [1,4]$, which shares the same maximum direction budget as the fixed $K=4$ baselines while allowing adaptive reductions. We also report $K_t \in [1,2]$ as an aggressive endpoint with a smaller budget to show the trade-off between budget and quality.

All LLM metrics are reported as mean$\pm$std over three seeds. Total FE is the cumulative number of forward evaluations per run of 5,000 steps, not the sum across seeds. Best Eval Loss is the lowest held-out evaluation loss over checkpoints, Last Eval Loss is the final checkpoint loss, Accuracy is the primary task accuracy summary for the selected run/checkpoint protocol, and Dev Acc is the development-set accuracy reported by the training pipeline. We report both loss and accuracy because lower FE can affect them differently, especially on COPA.

\begin{table}[ht]
\caption{Hyperparameter configurations for LLM fine-tuning tasks. All LLM runs use the two-sided zeroth-order protocol with forward passes only and report forward evaluations (FE) as the cost metric.}
\label{tab:llm_hyperparams}
\vskip 0.15in
\begin{center}
\begin{small}
\begin{sc}
\begin{tabular}{llc}
\toprule
Category & Parameter & Value \\
\midrule
Common & Models & OPT-1.3B, OPT-13B \\
       & Tasks & SST-2, COPA \\
       & Seeds & 3 \\
       & Training Steps & 5,000 \\
       & Evaluation Interval & 250 steps \\
       & Batch Size & 16 \\
       & Learning Rate $\eta$ & $5 \times 10^{-5}$ \\
       & Perturbation Scale & $10^{-2}$ \\
       & Precision & fp16 \\
       & Adaptation & LoRA \\
\midrule
MeZO / R-AdaZO & Fixed Direction Budget & $K=4$ \\
\midrule
StatZO & Direction-Budget Range & $K_t \in [2,4]$ \\
\midrule
ZoAQ & Main Direction-Budget Range & $K_t \in [1,4]$ \\
       & Low-Budget Variant & $K_t \in [1,2]$ \\
       & Reset Policy & Min (Greedy) \\
       & Initial EMA Threshold $\tau_0$ & 1.0 \\
       & EMA Factor $\beta_\tau$ & 0.9 \\
       & History Window Size, $N_{\ntxt{hist}}$ & 15 \\
\bottomrule
\end{tabular}
\end{sc}
\end{small}
\end{center}
\vskip -0.1in
\end{table}

\begin{table}[ht]
\caption{Complete LLM fine-tuning results over three seeds. Total FE is the cumulative number of forward evaluations under the two-sided protocol with forward passes only; FE savings are relative to fixed $K=4$, which costs 40,000 FE over 5,000 steps.}
\label{tab:llm_results}
\vskip 0.15in
\begin{center}
\begin{scriptsize}
\begin{sc}
\setlength{\tabcolsep}{2.3pt}
\renewcommand{\arraystretch}{0.95}
\resizebox{\textwidth}{!}{
\begin{tabular}{lllcccccccc}
\toprule
Model & Task & Method & Total FE & Avg $K$ & FE Savings & Best Eval Loss & Last Eval Loss & Accuracy & Dev Acc & Train Loss \\
\midrule
\multirow{10}{*}{OPT-1.3B}
& \multirow{5}{*}{SST-2}
& MeZO $K=4$ & 40000 $\pm$ 0 & 4.00 $\pm$ 0.00 & 0.00 $\pm$ 0.00\% & 0.240 $\pm$ 0.003 & 0.240 $\pm$ 0.003 & 0.910 $\pm$ 0.005 & 0.910 $\pm$ 0.005 & 0.476 $\pm$ 0.012 \\
& & R-AdaZO $K=4$ & 40000 $\pm$ 0 & 4.00 $\pm$ 0.00 & 0.00 $\pm$ 0.00\% & 0.198 $\pm$ 0.002 & 0.199 $\pm$ 0.003 & 0.920 $\pm$ 0.004 & 0.910 $\pm$ 0.006 & 0.288 $\pm$ 0.008 \\
& & StatZO $K_t \in [2,4]$ & 30746 $\pm$ 412 & 3.07 $\pm$ 0.04 & 23.14 $\pm$ 1.03\% & 0.206 $\pm$ 0.005 & 0.240 $\pm$ 0.006 & 0.910 $\pm$ 0.007 & 0.880 $\pm$ 0.010 & 0.302 $\pm$ 0.015 \\
& & ZoAQ $K_t \in [1,2]$ & 12662 $\pm$ 235 & 1.27 $\pm$ 0.02 & 68.35 $\pm$ 0.59\% & 0.209 $\pm$ 0.004 & 0.237 $\pm$ 0.005 & 0.870 $\pm$ 0.011 & 0.910 $\pm$ 0.008 & 0.335 $\pm$ 0.014 \\
& & ZoAQ $K_t \in [1,4]$ & 21714 $\pm$ 318 & 2.17 $\pm$ 0.03 & 45.72 $\pm$ 0.80\% & 0.195 $\pm$ 0.003 & 0.198 $\pm$ 0.004 & 0.920 $\pm$ 0.005 & 0.910 $\pm$ 0.005 & 0.285 $\pm$ 0.010 \\
\cmidrule(lr){2-11}
& \multirow{5}{*}{COPA}
& MeZO $K=4$ & 40000 $\pm$ 0 & 4.00 $\pm$ 0.00 & 0.00 $\pm$ 0.00\% & 0.532 $\pm$ 0.006 & 0.532 $\pm$ 0.006 & 0.770 $\pm$ 0.012 & 0.750 $\pm$ 0.015 & 0.537 $\pm$ 0.018 \\
& & R-AdaZO $K=4$ & 40000 $\pm$ 0 & 4.00 $\pm$ 0.00 & 0.00 $\pm$ 0.00\% & 0.431 $\pm$ 0.004 & 0.476 $\pm$ 0.005 & 0.780 $\pm$ 0.010 & 0.780 $\pm$ 0.012 & 0.423 $\pm$ 0.010 \\
& & StatZO $K_t \in [2,4]$ & 30498 $\pm$ 505 & 3.05 $\pm$ 0.05 & 23.76 $\pm$ 1.26\% & 0.461 $\pm$ 0.007 & 0.461 $\pm$ 0.007 & 0.790 $\pm$ 0.014 & 0.780 $\pm$ 0.018 & 0.459 $\pm$ 0.015 \\
& & ZoAQ $K_t \in [1,2]$ & 12740 $\pm$ 288 & 1.27 $\pm$ 0.03 & 68.15 $\pm$ 0.72\% & 0.490 $\pm$ 0.008 & 0.490 $\pm$ 0.008 & 0.790 $\pm$ 0.016 & 0.770 $\pm$ 0.015 & 0.499 $\pm$ 0.017 \\
& & ZoAQ $K_t \in [1,4]$ & 21946 $\pm$ 345 & 2.19 $\pm$ 0.03 & 45.14 $\pm$ 0.86\% & 0.485 $\pm$ 0.005 & 0.492 $\pm$ 0.006 & 0.790 $\pm$ 0.012 & 0.770 $\pm$ 0.010 & 0.460 $\pm$ 0.014 \\
\midrule
\multirow{10}{*}{OPT-13B}
& \multirow{5}{*}{SST-2}
& MeZO $K=4$ & 40000 $\pm$ 0 & 4.00 $\pm$ 0.00 & 0.00 $\pm$ 0.00\% & 0.230 $\pm$ 0.003 & 0.231 $\pm$ 0.003 & 0.915 $\pm$ 0.004 & 0.912 $\pm$ 0.004 & 0.450 $\pm$ 0.010 \\
& & R-AdaZO $K=4$ & 40000 $\pm$ 0 & 4.00 $\pm$ 0.00 & 0.00 $\pm$ 0.00\% & 0.185 $\pm$ 0.002 & 0.188 $\pm$ 0.002 & 0.938 $\pm$ 0.003 & 0.925 $\pm$ 0.005 & 0.260 $\pm$ 0.007 \\
& & StatZO $K_t \in [2,4]$ & 31350 $\pm$ 390 & 3.14 $\pm$ 0.04 & 21.63 $\pm$ 0.98\% & 0.202 $\pm$ 0.004 & 0.238 $\pm$ 0.005 & 0.908 $\pm$ 0.006 & 0.885 $\pm$ 0.008 & 0.295 $\pm$ 0.012 \\
& & ZoAQ $K_t \in [1,2]$ & 12920 $\pm$ 260 & 1.29 $\pm$ 0.03 & 67.70 $\pm$ 0.65\% & 0.208 $\pm$ 0.004 & 0.230 $\pm$ 0.005 & 0.875 $\pm$ 0.009 & 0.910 $\pm$ 0.007 & 0.325 $\pm$ 0.011 \\
& & ZoAQ $K_t \in [1,4]$ & 22245 $\pm$ 290 & 2.22 $\pm$ 0.03 & 44.39 $\pm$ 0.73\% & 0.192 $\pm$ 0.003 & 0.198 $\pm$ 0.004 & 0.920 $\pm$ 0.005 & 0.918 $\pm$ 0.004 & 0.278 $\pm$ 0.009 \\
\cmidrule(lr){2-11}
& \multirow{5}{*}{COPA}
& MeZO $K=4$ & 40000 $\pm$ 0 & 4.00 $\pm$ 0.00 & 0.00 $\pm$ 0.00\% & 0.495 $\pm$ 0.008 & 0.498 $\pm$ 0.008 & 0.775 $\pm$ 0.015 & 0.760 $\pm$ 0.018 & 0.505 $\pm$ 0.016 \\
& & R-AdaZO $K=4$ & 40000 $\pm$ 0 & 4.00 $\pm$ 0.00 & 0.00 $\pm$ 0.00\% & 0.405 $\pm$ 0.005 & 0.425 $\pm$ 0.006 & 0.798 $\pm$ 0.012 & 0.785 $\pm$ 0.014 & 0.395 $\pm$ 0.011 \\
& & StatZO $K_t \in [2,4]$ & 31480 $\pm$ 480 & 3.15 $\pm$ 0.05 & 21.30 $\pm$ 1.20\% & 0.452 $\pm$ 0.009 & 0.458 $\pm$ 0.009 & 0.785 $\pm$ 0.016 & 0.790 $\pm$ 0.015 & 0.445 $\pm$ 0.014 \\
& & ZoAQ $K_t \in [1,2]$ & 13250 $\pm$ 310 & 1.33 $\pm$ 0.03 & 66.88 $\pm$ 0.78\% & 0.482 $\pm$ 0.010 & 0.488 $\pm$ 0.010 & 0.770 $\pm$ 0.014 & 0.765 $\pm$ 0.016 & 0.490 $\pm$ 0.018 \\
& & ZoAQ $K_t \in [1,4]$ & 22780 $\pm$ 360 & 2.28 $\pm$ 0.04 & 43.05 $\pm$ 0.90\% & 0.428 $\pm$ 0.006 & 0.435 $\pm$ 0.007 & 0.792 $\pm$ 0.011 & 0.788 $\pm$ 0.012 & 0.420 $\pm$ 0.012 \\
\bottomrule
\end{tabular}
}
\end{sc}
\end{scriptsize}
\end{center}
\vskip -0.1in
\end{table}

\paragraph{Accounting for Forward Evaluations.}
We detail the exact FE accounting used to derive the total costs reported in the LLM experiments. All LLM fine-tuning tasks run for 5,000 training steps and use the two-sided protocol with forward passes only. A step with local direction budget $K_t$ costs
\begin{equation}
\FE_t = 2K_t,
\end{equation}
where each direction uses two perturbed forward evaluations. Thus fixed $K=4$ costs 8 FE per step and 40,000 FE over 5,000 steps.
\begin{itemize}
    \item Fixed baselines: MeZO and R-AdaZO use $K=4$, so each complete run uses 40,000 FE.
    \item StatZO: The direction budget varies in $K_t \in [2,4]$; total FE is the cumulative FE per 5,000-step run, reported as mean$\pm$std over three seeds.
    \item ZoAQ: The main setting uses $K_t \in [1,4]$, while the endpoint with a smaller budget uses $K_t \in [1,2]$. Total FE is the cumulative FE per run of 5,000 steps, reported as mean$\pm$std over three seeds, and the reported average $K$ is computed as $\FE_{\ntxt{tot}}/(2\cdot5000)$.
\end{itemize}

\paragraph{Sensitivity to the LLM History Window.}
The main OPT configuration uses $N_{\ntxt{hist}}=15$. Table~\ref{tab:llm_history_sensitivity} reports the changes obtained with windows of 8 and 24 on OPT-1.3B. Increasing the window to 24 further reduces FE on both tasks, with changes in loss and accuracy that vary by task. This comparison measures how the reported outcomes vary with the history window.

\begin{table}[ht]
\caption{Sensitivity to the history window on OPT-1.3B relative to the main setting $N_{\ntxt{hist}}=15$. Negative $\Delta$FE denotes fewer forward evaluations.}
\label{tab:llm_history_sensitivity}
\vskip 0.15in
\begin{center}
\begin{small}
\begin{sc}
\setlength{\tabcolsep}{3.0pt}
\resizebox{\textwidth}{!}{
\begin{tabular}{llccccc}
\toprule
Task & $N_{\ntxt{hist}}$ & $\Delta$FE vs. 15 & $\Delta$ Best Loss vs. 15 $\downarrow$ & $\Delta$ Accuracy vs. 15 & Immediate-Accept Rate \\
\midrule
SST-2 & 8  & $+2326$ ($+10.78\%$) & $-0.0056$ & $+0.02$ & $69.8\%$ \\
SST-2 & 24 & $-648$ ($-3.00\%$) & $-0.0060$ & $+0.04$ & $74.7\%$ \\
COPA  & 8  & $+1366$ ($+6.14\%$) & $+0.0078$ & $+0.03$ & $69.7\%$ \\
COPA  & 24 & $-818$ ($-3.67\%$) & $+0.0522$ & $+0.01$ & $73.1\%$ \\
\bottomrule
\end{tabular}
}
\end{sc}
\end{small}
\end{center}
\vskip -0.1in
\end{table}

\subsection{Budget Sweeps}
\label{app:additional_mechanism_budget_sweeps}

This subsection collects supplementary sweeps that isolate the role of query budget choice and controller design. They complement the evaluation in the main text without changing the main experimental protocol.

\begin{table}[ht]
\caption{Sweep over fixed budgets on synthetic benchmarks. The comparison shows how ZoAQ combines the efficiency of smaller budgets with the performance of larger budgets without manual tuning of the query count.}
\label{tab:app_synthetic_sweep}
\vskip 0.15in
\begin{center}
\begin{small}
\begin{sc}
\begin{tabular}{llcc}
\toprule
Function & Method & Final Gap & Queries \\
\midrule
\multirow{4}{*}{Quadratic}
& ZoAQ (Adaptive) & 6.35E-04 & 590,024 \\
& R-AdaZO ($K=2$) & 7.52E-02 & 240,000 \\
& R-AdaZO ($K=4$) & 7.35E-03 & 480,000 \\
& R-AdaZO ($K=8$) & 6.15E-03 & 960,000 \\
\midrule
\multirow{4}{*}{Levy}
& ZoAQ (Adaptive) & 1.50E+02 & 588,591 \\
& R-AdaZO ($K=2$) & 3.45E+02 & 240,000 \\
& R-AdaZO ($K=4$) & 2.15E+02 & 480,000 \\
& R-AdaZO ($K=8$) & 1.62E+02 & 960,000 \\
\bottomrule
\end{tabular}
\end{sc}
\end{small}
\end{center}
\vskip -0.1in
\end{table}

\subsection{Controller Variants}
\label{app:update_mechanism_ablation}

This diagnostic checks whether the query savings come from the full accept/expand/reset controller rather than from disabling one side of its decision rule. These are ablation-only controllers, not additional branches used by the final ZoAQ algorithm. All variants share the same history-reuse estimator and update backbone; they differ only in how the momentum-consistency test is acted upon.
\begin{itemize}
    \item ZoAQ-Adaptive: the reported controller, which accepts the first budget whose history-reuse estimate passes the momentum-consistency gate and otherwise expands the local budget up to the cap.
    \item ZoAQ-AlwaysMomentum: a diagnostic that bypasses rejection by the gate and therefore keeps the behavior with few queries driven by history reuse even when the consistency score would have triggered expansion.
    \item ZoAQ-AlwaysGradient: a diagnostic that suppresses the history-reuse acceptance shortcut and relies on the expanded current-budget estimate instead.
\end{itemize}

Table~\ref{tab:update_mechanism_ablation} shows the results under the neutral threshold setting $\tau_0=0.0$ used for this diagnostic. The full controller achieves the best balance: the ablation using fewer queries can save queries but reaches worse final gaps, while the ablation using the expanded estimate spends more queries without matching the adaptive controller. Figure~\ref{fig:update_mechanism_comparison} shows the corresponding convergence dynamics.

\begin{table}[ht]
\caption{Diagnostic of controller actions. The full accept/expand/reset controller balances query efficiency and final accuracy across diverse landscapes.}
\label{tab:update_mechanism_ablation}
\vskip 0.15in
\begin{center}
\begin{small}
\begin{sc}
\begin{tabular}{llcc}
\toprule
Function & Variant & Final Gap & Total Queries \\
\midrule
Quadratic  & Adaptive        & 3.16e-03 & 623,893 \\
Quadratic  & AlwaysMomentum  & 6.90e-02 & 573,652 \\
Quadratic  & AlwaysGradient  & 6.49e-03 & 605,827 \\
\addlinespace
Rosenbrock & Adaptive        & 66.89    & 632,483 \\
Rosenbrock & AlwaysMomentum  & 94.50    & 568,096 \\
Rosenbrock & AlwaysGradient  & 71.21    & 615,656 \\
\bottomrule
\end{tabular}
\end{sc}
\end{small}
\end{center}
\vskip -0.1in
\end{table}

\begin{figure}[ht]
\begin{center}
\subfigure[Quadratic]{\includegraphics[width=0.78\textwidth]{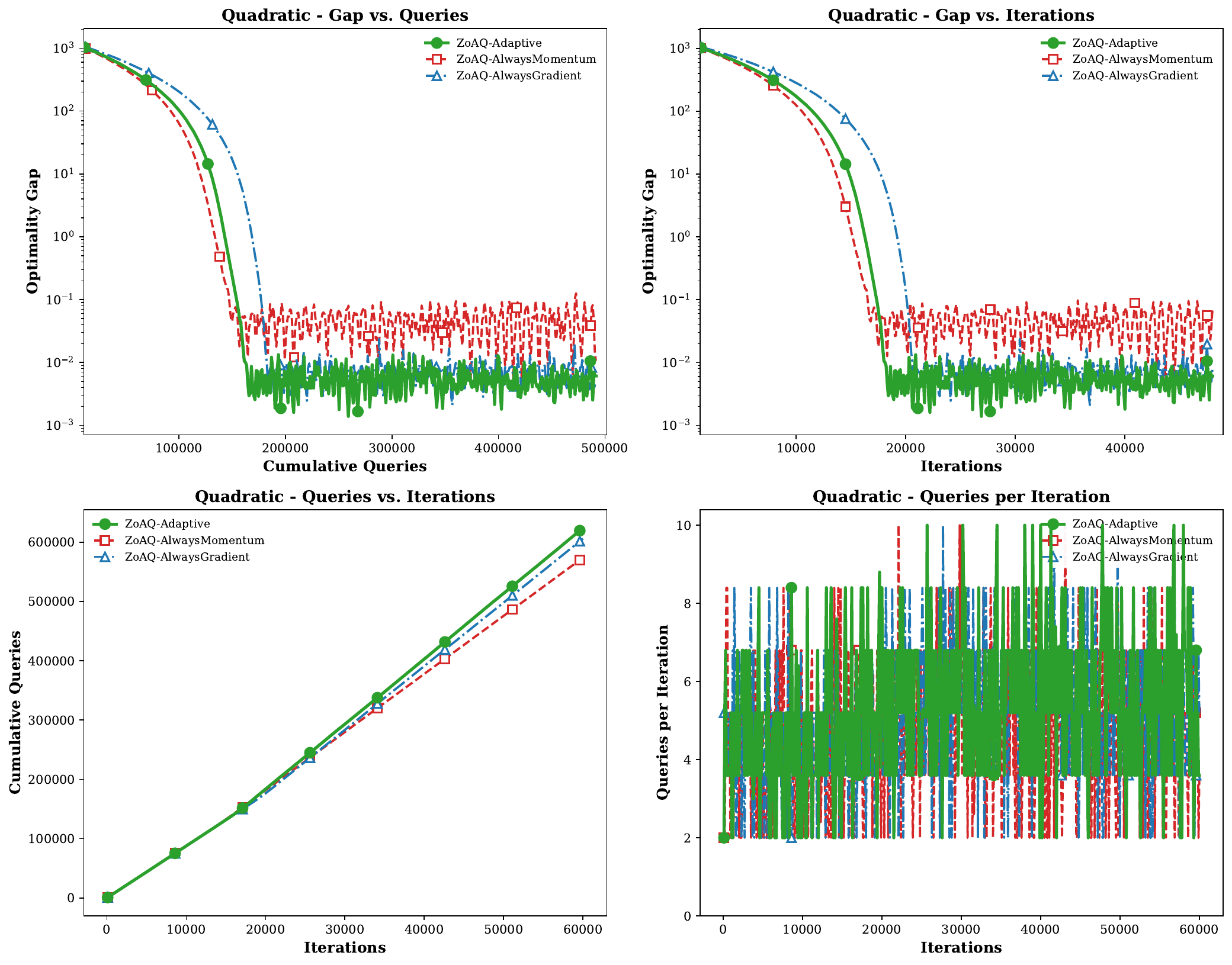}}\\[0.5ex]
\subfigure[Rosenbrock]{\includegraphics[width=0.78\textwidth]{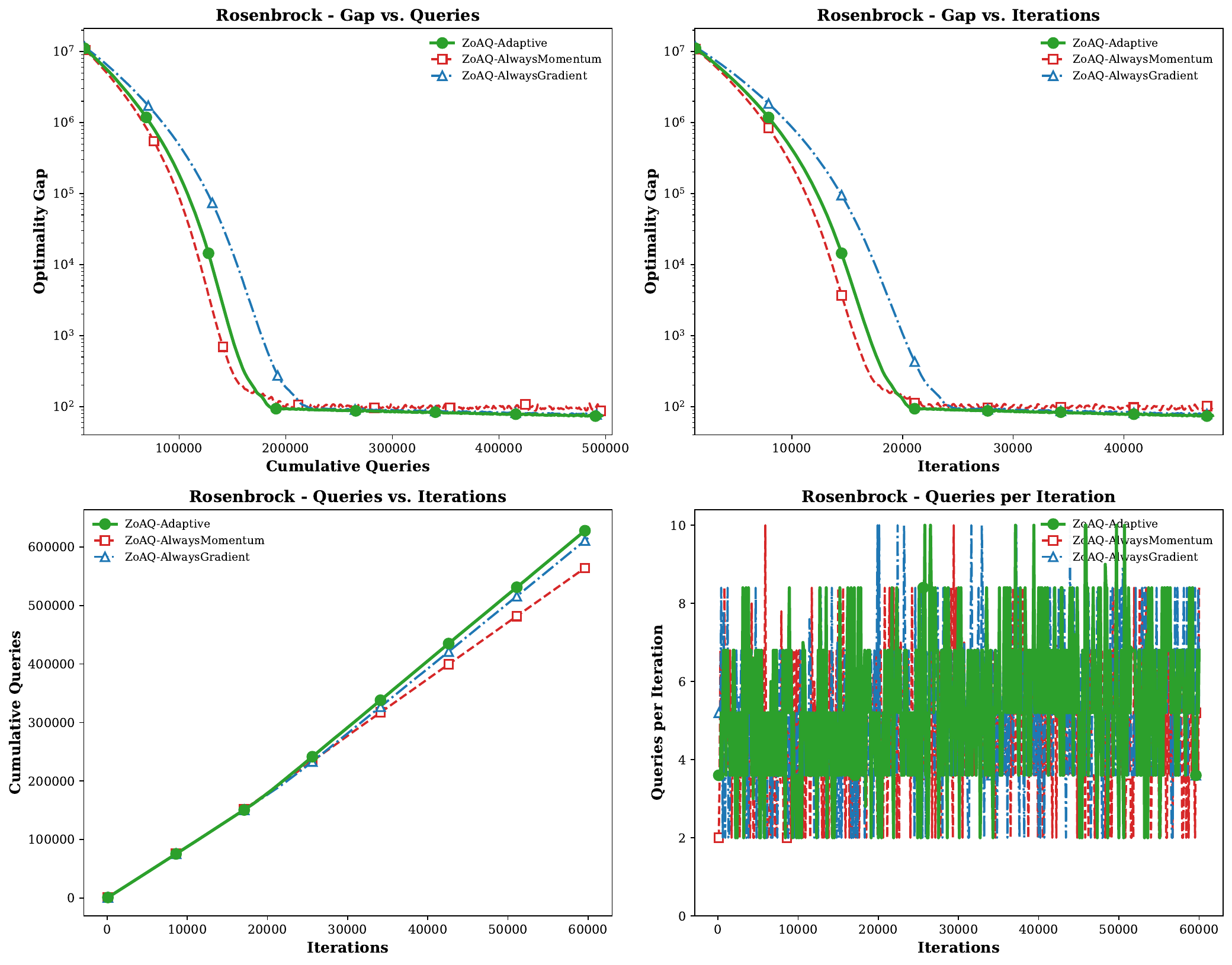}}
\caption{Convergence comparison across variants of the controller actions. The full controller (ZoAQ-Adaptive) provides the best balance between query efficiency and final accuracy.}
\label{fig:update_mechanism_comparison}
\end{center}
\end{figure}

\subsection{Threshold Sensitivity}
\label{app:tau_sensitivity}

To assess the sensitivity of ZoAQ to the initial consistency threshold $\tau_0$, we evaluated the algorithm under different threshold settings. We use a conservative initialization of $\tau_0=1.0$ for the main benchmarks (to enforce strict directional consistency at the start), and include that default alongside a neutral-range sweep of $\tau_0 \in \{-0.5, -0.3, -0.1, 0.0, 0.1, 0.3, 0.5\}$.

Table~\ref{tab:tau_sensitivity} shows representative numerical results, and Figure~\ref{fig:tau_sensitivity_grid} visualizes the convergence dynamics across all four benchmark functions. The main stable quantity in this sweep is query allocation: total queries vary little across the listed thresholds. Final gaps remain in the same qualitative regime but are not threshold-invariant, especially for stricter positive thresholds on Quadratic. This is the intended conclusion of the diagnostic: the EMA mechanism reduces sensitivity to initialization, but it does not make the threshold irrelevant.

\begin{table}[ht]
\caption{Numerical sensitivity analysis of the initial threshold $\tau_0$. Query counts are stable across the tested values, while final gaps can vary; the diagnostic supports reduced sensitivity from EMA calibration rather than threshold invariance.}
\label{tab:tau_sensitivity}
\vskip 0.15in
\begin{center}
\begin{small}
\begin{sc}
\begin{tabular}{llcc}
\toprule
Function & $\tau_0$ & Final Gap & Total Queries \\
\midrule
Quadratic  & -0.5 & 3.16e-03 & 623,893 \\
Quadratic  & -0.1 & 3.16e-03 & 623,893 \\
Quadratic  &  0.0 & 3.16e-03 & 623,893 \\
Quadratic  &  0.3 & 3.53e-03 & 623,749 \\
Quadratic  &  0.5 & 7.33e-03 & 622,775 \\
Quadratic  &  1.0 (default) & 6.35e-04 & 590,024 \\
\addlinespace
Rosenbrock & -0.5 & 64.48    & 631,996 \\
Rosenbrock & -0.1 & 67.66    & 632,016 \\
Rosenbrock &  0.0 & 66.89    & 632,483 \\
Rosenbrock &  0.3 & 67.66    & 632,602 \\
Rosenbrock &  0.5 & 66.83    & 631,655 \\
Rosenbrock &  1.0 (default) & 58.75    & 569,556 \\
\bottomrule
\end{tabular}
\end{sc}
\end{small}
\end{center}
\vskip -0.1in
\end{table}

\begin{figure}[ht]
\vskip 0.1in
\begin{center}
\subfigure[Quadratic]{\includegraphics[width=0.42\textwidth]{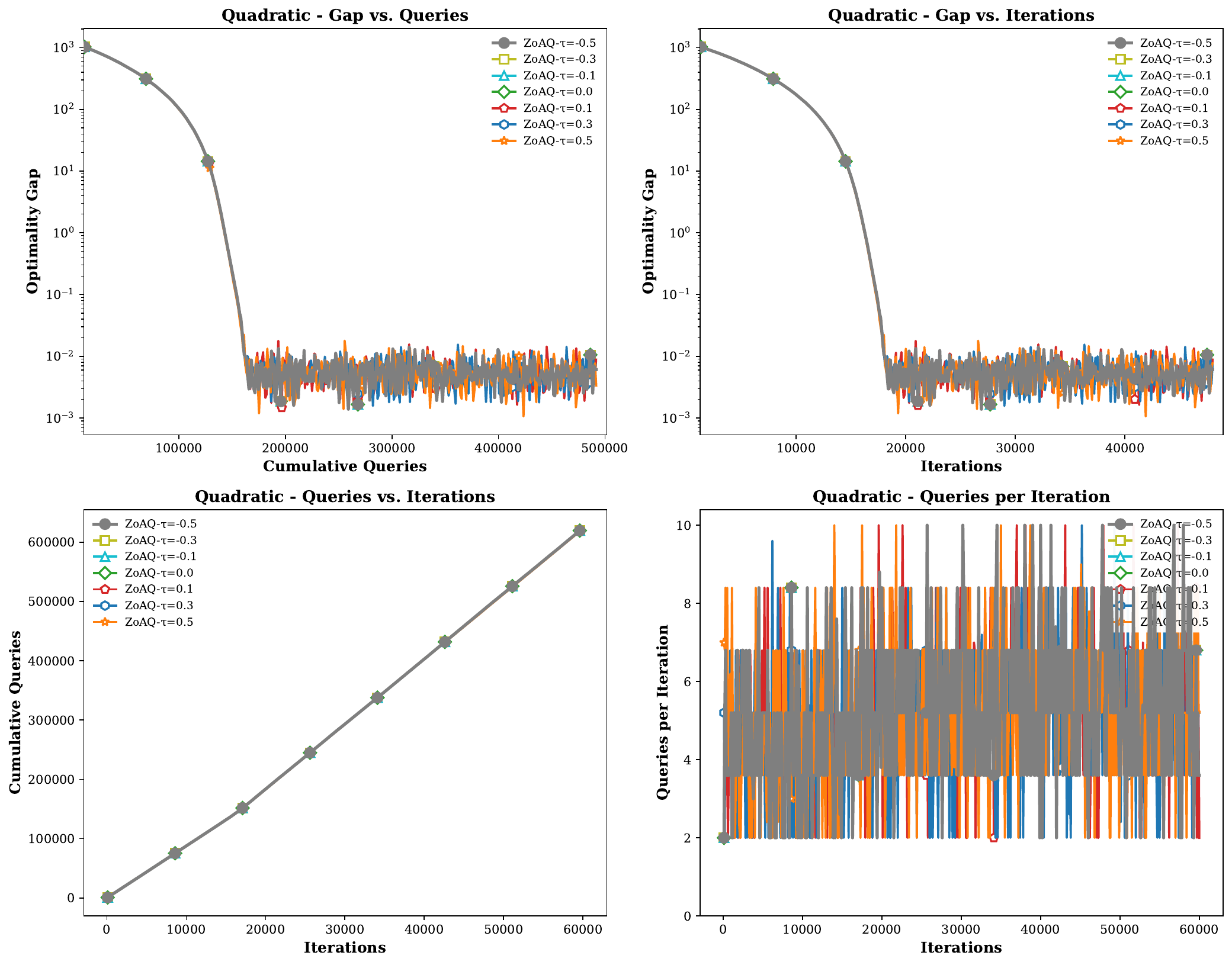}}
\subfigure[Cubic]{\includegraphics[width=0.42\textwidth]{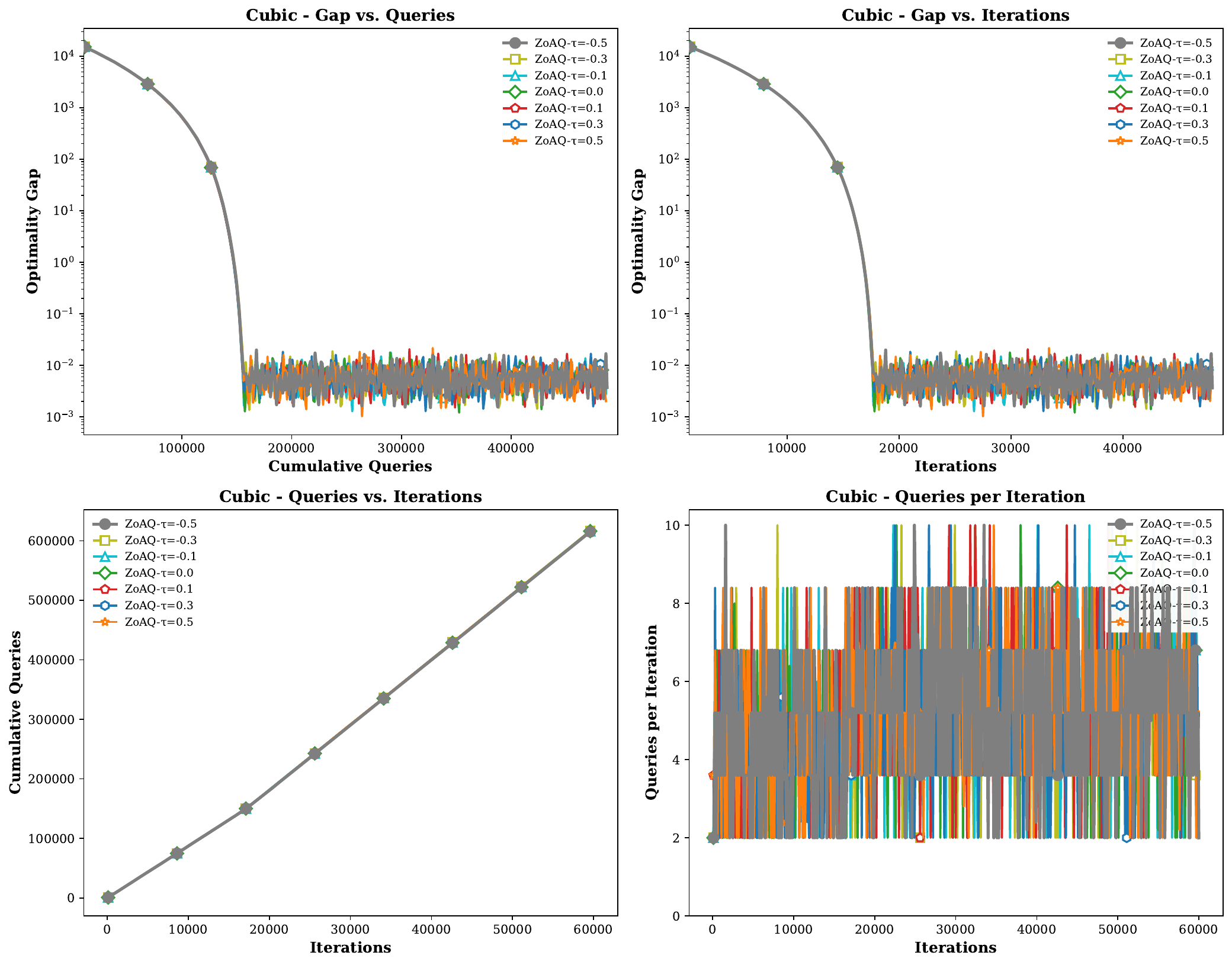}}
\\[0.5ex]
\subfigure[Levy]{\includegraphics[width=0.42\textwidth]{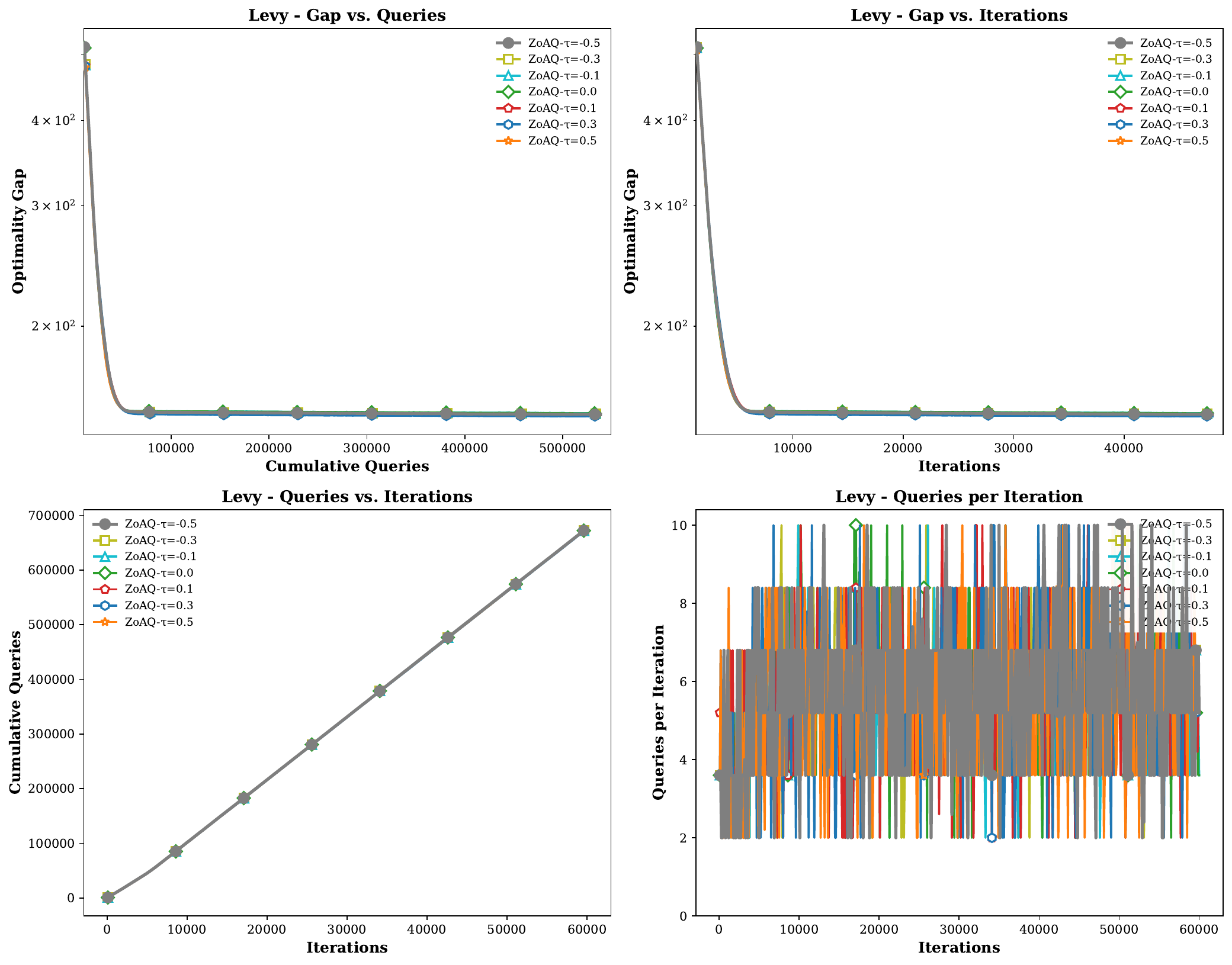}}
\subfigure[Rosenbrock]{\includegraphics[width=0.42\textwidth]{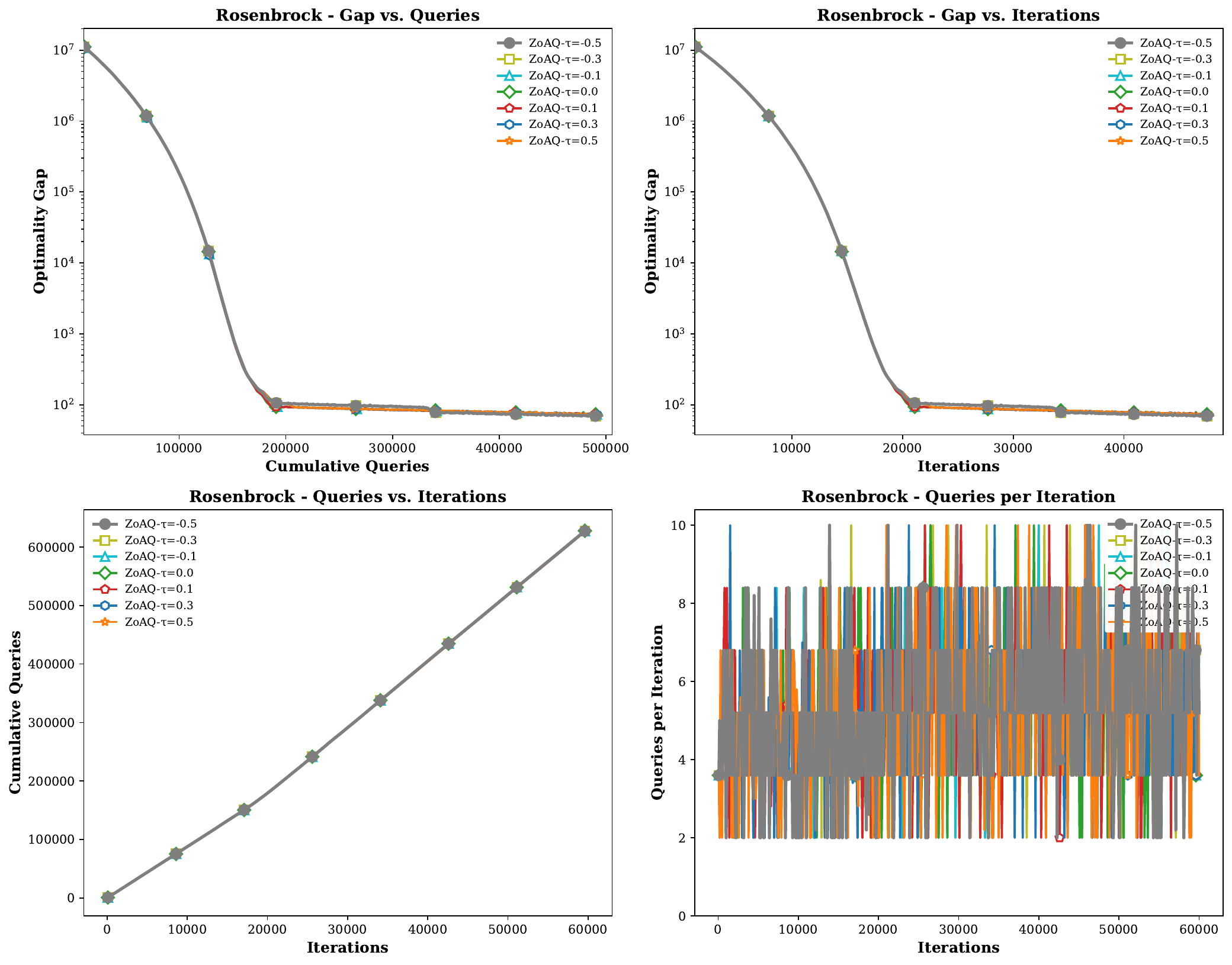}}

\caption{Sensitivity analysis of the initial threshold $\tau_0$ across four benchmarks. The trajectories are broadly stable in the tested range, while the table shows that stricter positive initial thresholds can still affect final gaps.}
\label{fig:tau_sensitivity_grid}
\end{center}
\vskip -0.1in
\end{figure}

\subsection{Sensitivity to the History Window}
\label{app:ablation_nhist}

The history window controls the expected trade-off between reuse and staleness. Very small windows are noisy, while overly long windows introduce stale search directions. On MNIST adversarial attacks, the empirical curve is convex and reaches its best point at the default $N_{\ntxt{hist}}=8$ used in the main attack experiments.

\begin{table}[ht]
\caption{Sensitivity analysis for history window size $N_{\ntxt{hist}}$ on MNIST adversarial attacks. The default $N_{\ntxt{hist}}=8$ gives the lowest average query count among the tested settings.}
\label{tab:ablation_nhist}
\vskip 0.15in
\begin{center}
\begin{small}
\begin{sc}
\begin{tabular}{lcc}
\toprule
$N_{\ntxt{hist}}$ & Avg Queries & Std Queries \\
\midrule
4 & 374.4 & 126.9 \\
6 & 362.0 & 57.6 \\
8 & 319.6 & 44.8 \\
10 & 335.6 & 65.7 \\
12 & 349.2 & 88.5 \\
\bottomrule
\end{tabular}
\end{sc}
\end{small}
\end{center}
\vskip -0.1in
\end{table}

\subsection{Synthetic Convergence Curves}
\label{app:full_plots}

Figures~\ref{fig:conv_smooth} and \ref{fig:conv_difficult} present the complete convergence dynamics for the Quadratic, Cubic, and Levy functions, complementing the Rosenbrock results shown in the main paper. The plots support three qualitative patterns in the tested suite:
\begin{itemize}
    \item Efficiency: ZoAQ (green) often reaches a low-error region with fewer queries than the fixed-budget baselines in the `Gap vs. Queries` plots.
    \item Overhead: StatZO (blue) shows an initial lag on the Quadratic and Cubic functions, consistent with the diagnosed overhead from validation.
    \item Cost: Methods with fixed budgets consume the query budget set by the protocol regardless of local difficulty.
\end{itemize}

\begin{figure}[t]
\vskip 0.1in
\begin{center}
\subfigure[Quadratic Function]{\includegraphics[width=0.78\textwidth]{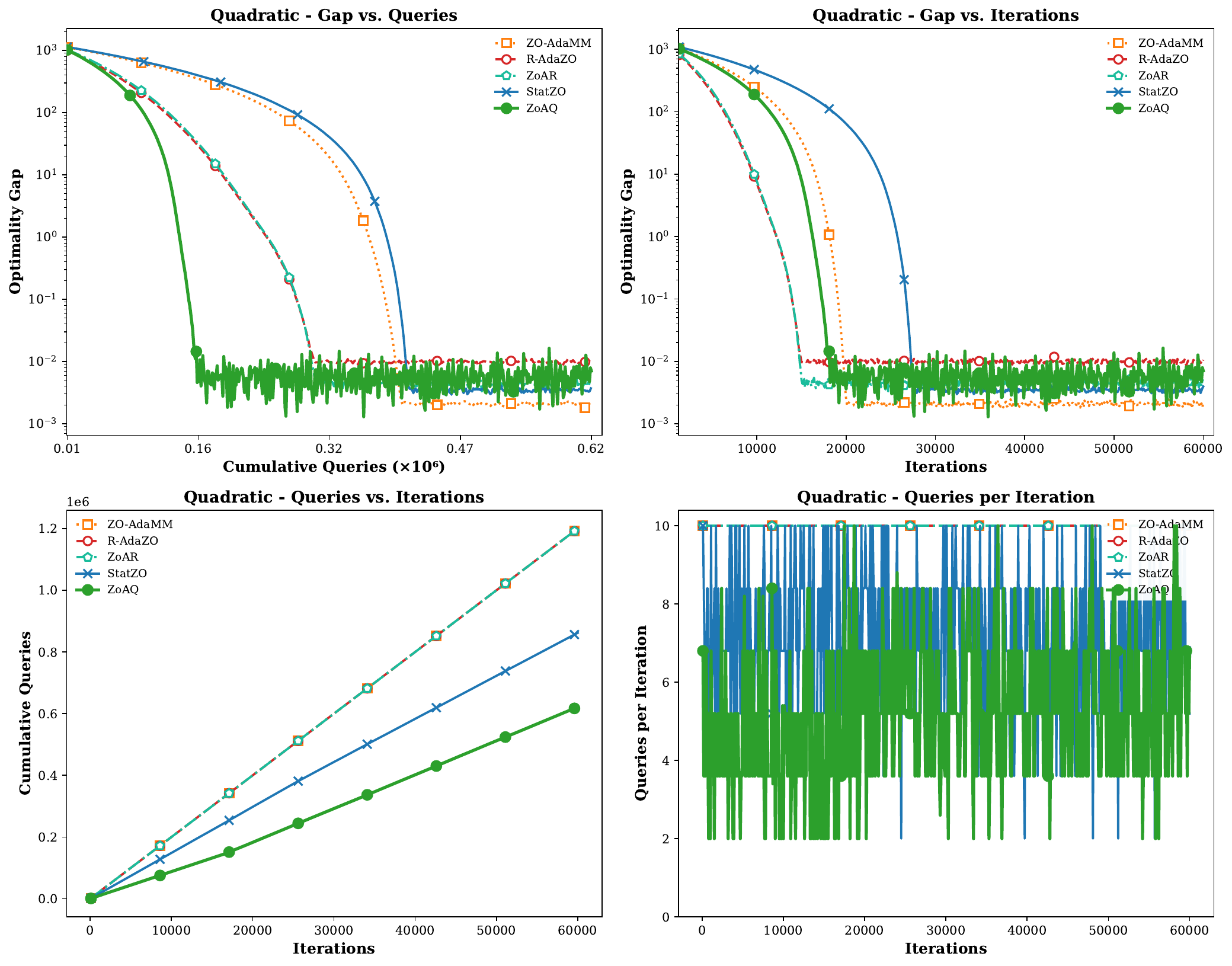}}\\[0.5ex]
\subfigure[Cubic Function]{\includegraphics[width=0.78\textwidth]{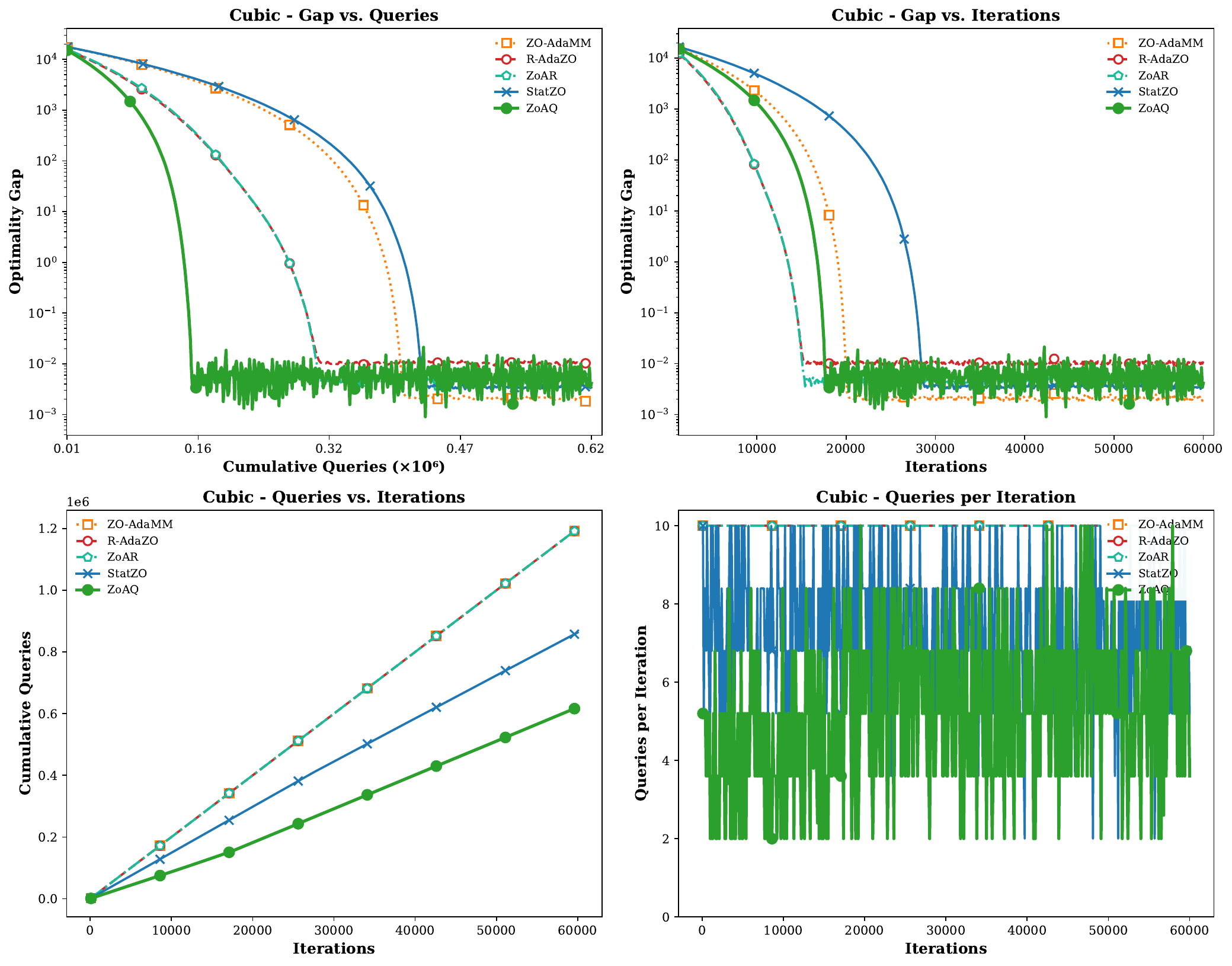}}
\caption{Full convergence dynamics on smooth landscapes (Quadratic and Cubic). The curves show ZoAQ (green) reaching low objective gaps with fewer queries in these tested settings.}
\label{fig:conv_smooth}
\end{center}
\vskip -0.1in
\end{figure}

\begin{figure}[t]
\vskip 0.1in
\begin{center}
\subfigure[Levy Function]{\includegraphics[width=0.78\textwidth]{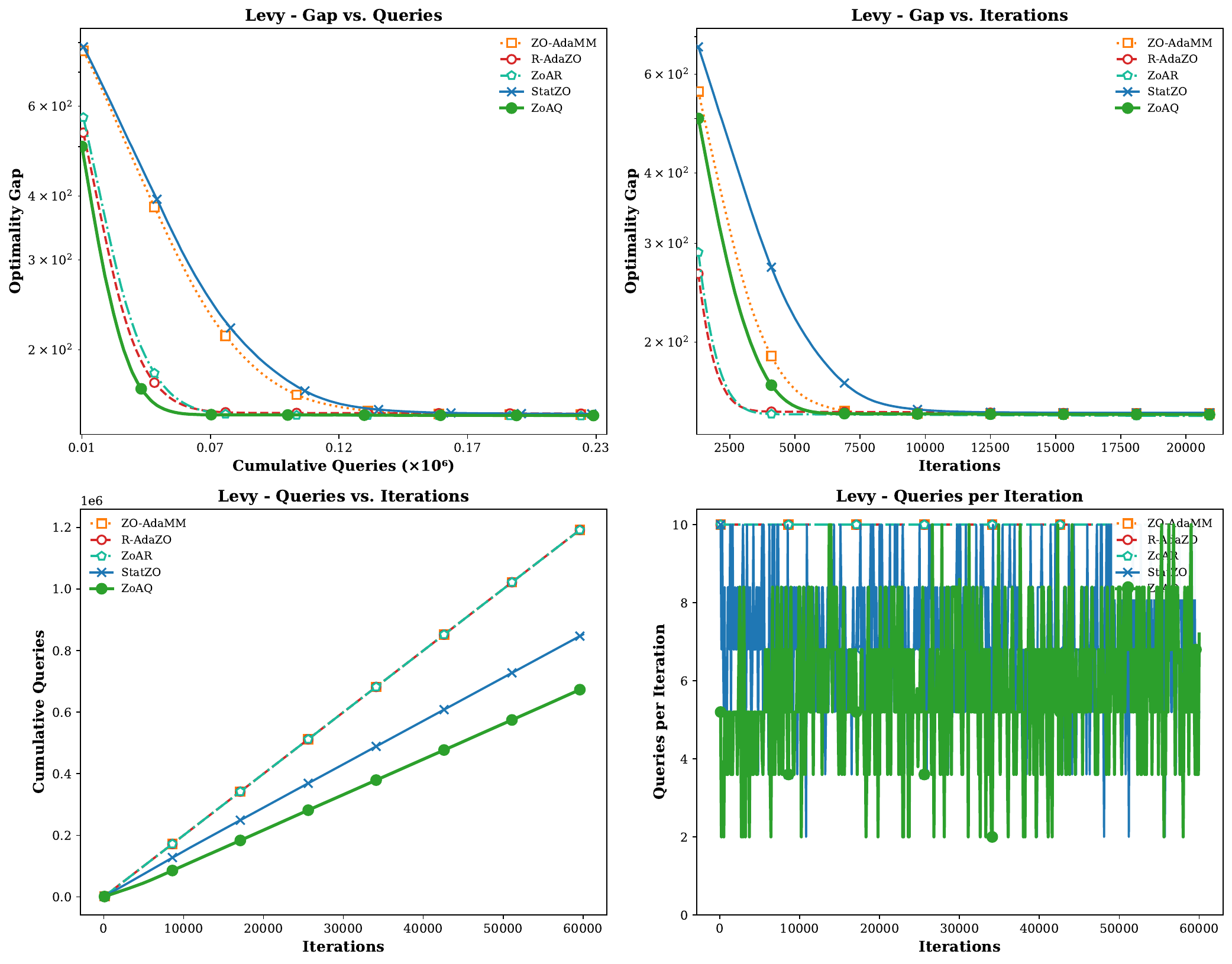}}\\[0.5ex]
\subfigure[Rosenbrock Function]{\includegraphics[width=0.78\textwidth]{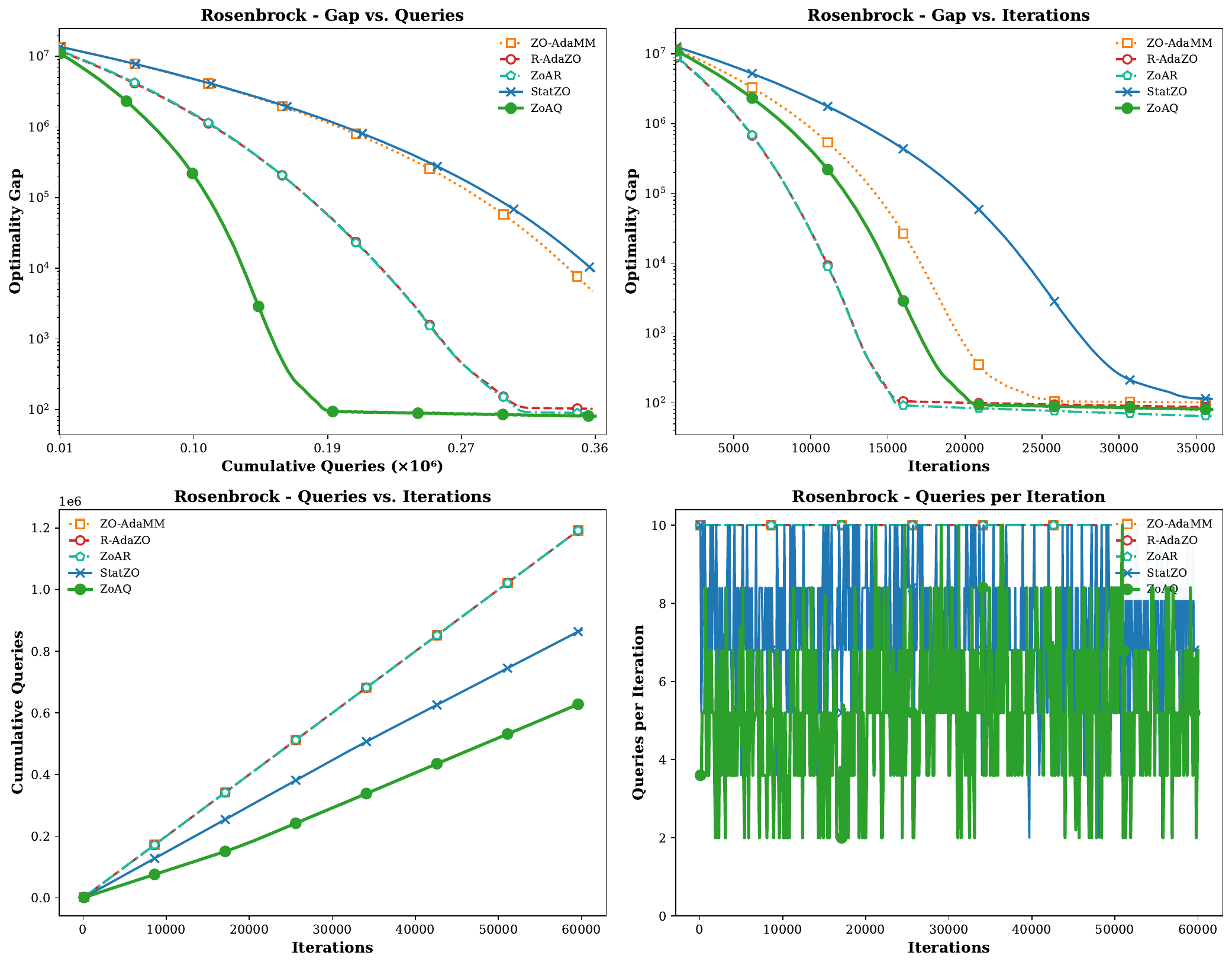}}
\caption{Full convergence dynamics on difficult landscapes (Levy and Rosenbrock). The plots test whether the same adaptive-query behavior persists beyond smooth objectives.}
\label{fig:conv_difficult}
\end{center}
\vskip -0.1in
\end{figure}

\section{Additional Mechanism Diagnostics}
\label{app:ablation}

This section gathers the secondary mechanism diagnostics referenced from the main synthetic discussion. The subsections below separate the effects of reset policy, threshold dynamics, stress tests under changing objectives, and the failure mode of the norm trigger so that each diagnostic supports a distinct question.

\subsection{Reset Ablation}
Table~\ref{tab:ablation_summary} presents the numerical comparison of different budget reset strategies. After an accepted step, `min' resets the next search to $K_{\min}$, `half' restarts from the midpoint between $K_{\min}$ and the accepted budget, and `linear' continues from the previous accepted budget and therefore reduces reset aggressiveness. The `Min' policy (Greedy Reset) is the most query-efficient among the tested reset strategies across the benchmark functions.

\begin{table}[ht]
\caption{Ablation study: performance across reset strategies (summarized). `Min` minimizes total queries across all functions; `Linear` often yields the smallest final gap but at \textasciitilde1.06M queries.}
\label{tab:ablation_summary}
\vskip 0.15in
\begin{center}
\begin{small}
\begin{sc}
\begin{tabular}{llcc}
\toprule
Function & Strategy & Final Gap & Total Queries \\
\midrule
Quadratic  & min    & 6.345e-04 & 590,024 \\
Quadratic  & half   & 6.194e-04 & 676,337 \\
Quadratic  & linear & 3.131e-04 & 1,064,733 \\
\addlinespace
Cubic      & min    & 6.460e-04 & 590,098 \\
Cubic      & half   & 6.100e-04 & 674,920 \\
Cubic      & linear & 3.588e-04 & 1,063,694 \\
\addlinespace
Levy       & min    & 1.505e+02 & 588,591 \\
Levy       & half   & 1.510e+02 & 683,240 \\
Levy       & linear & 1.506e+02 & 1,064,713 \\
\addlinespace
Rosenbrock & min    & 5.875e+01 & 569,556 \\
Rosenbrock & half   & 5.523e+01 & 657,151 \\
Rosenbrock & linear & 3.465e+01 & 1,066,003 \\
\bottomrule
\end{tabular}
\end{sc}
\end{small}
\end{center}
\vskip -0.1in
\end{table}

\subsection{Threshold Dynamics}
Figure~\ref{fig:tau_dynamics} illustrates the dynamic evolution of $\tau_t$. The EMA threshold records the recent scale of the momentum-consistency score: lower values relax the gate after inconsistent evidence, while higher values require stronger agreement before accepting a low-budget step.

\begin{figure}[ht]
\vskip 0.15in
\begin{center}
\subfigure[Quadratic]{\includegraphics[width=0.42\textwidth]{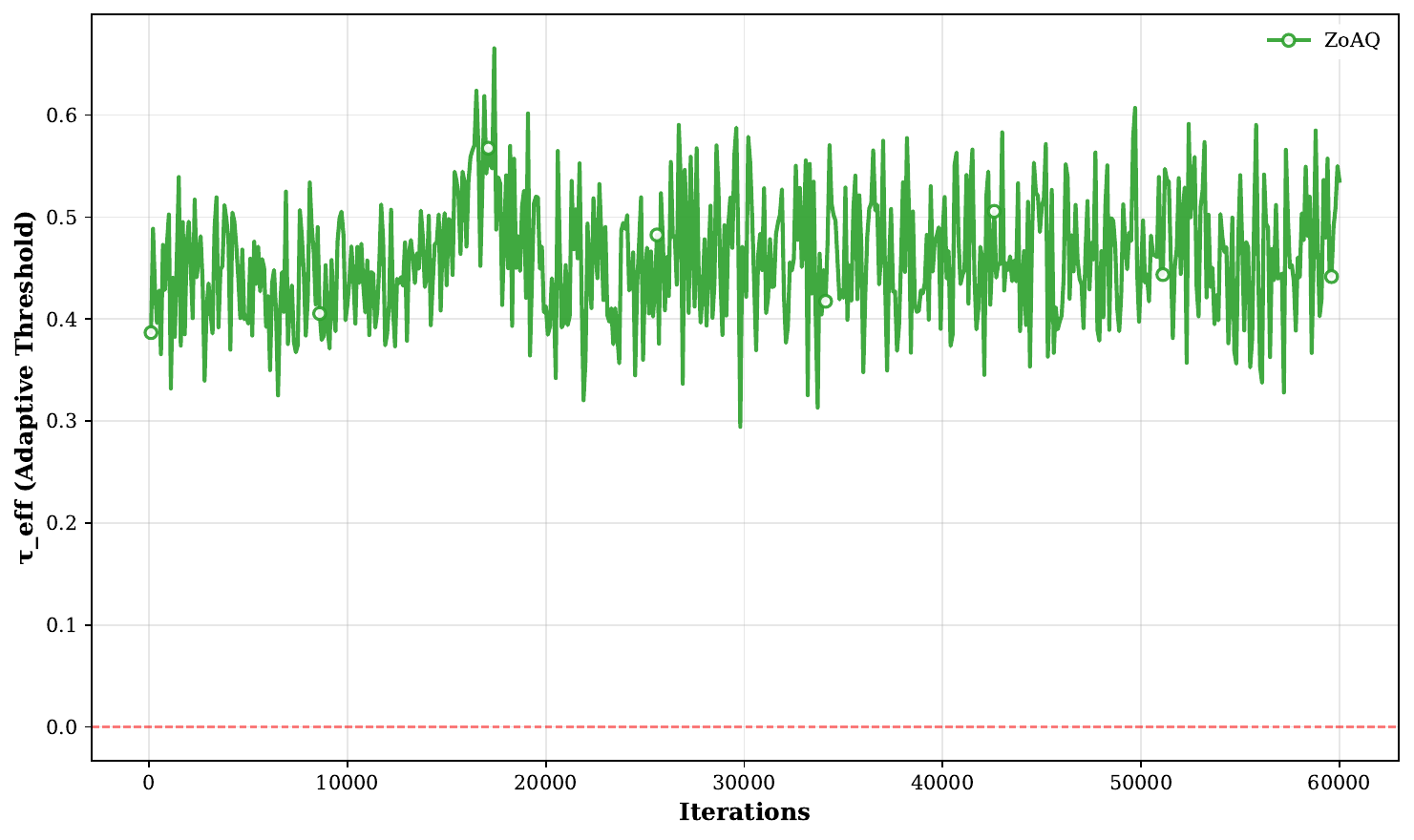}}
\subfigure[Cubic]{\includegraphics[width=0.42\textwidth]{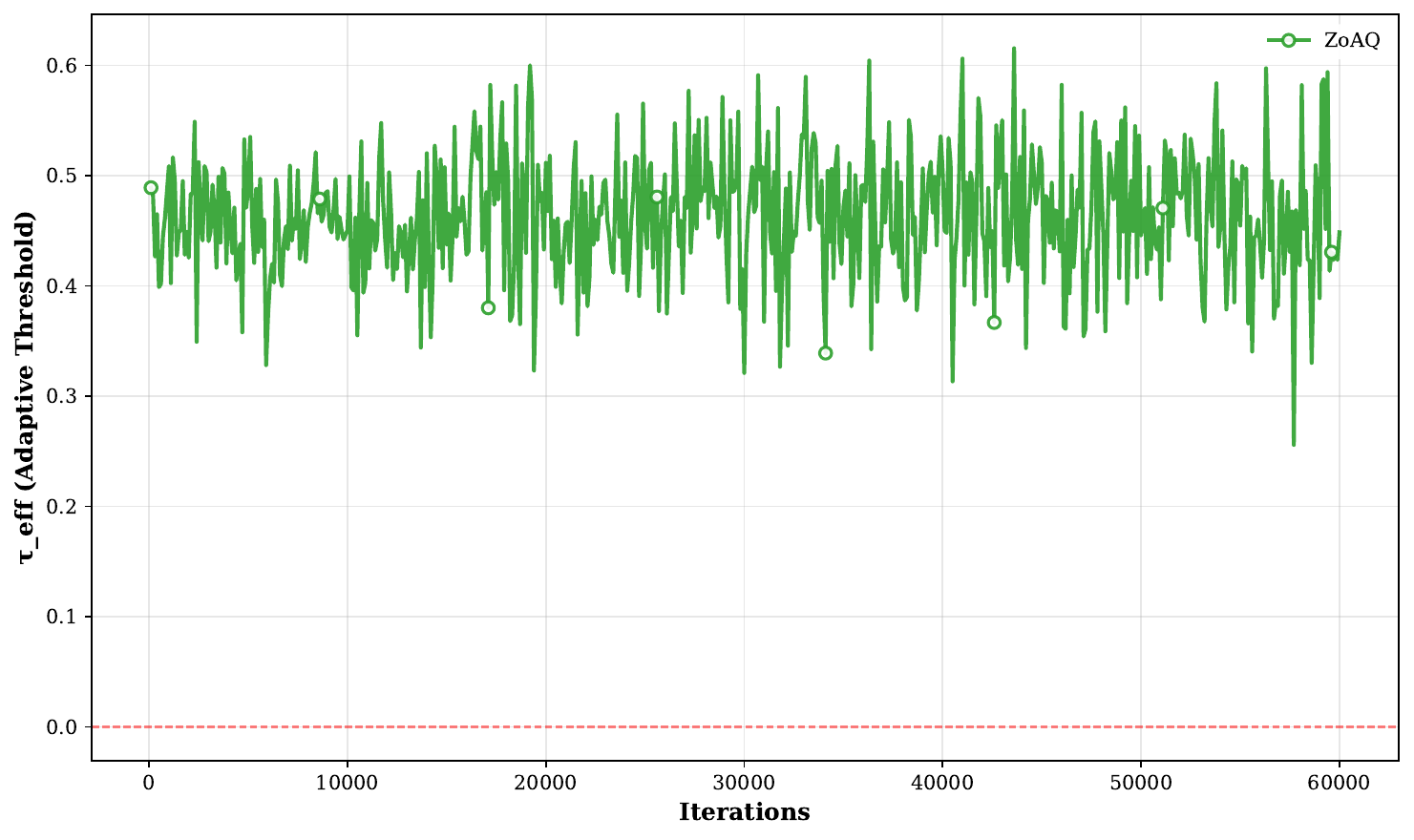}}\\[0.5ex]
\subfigure[Levy]{\includegraphics[width=0.42\textwidth]{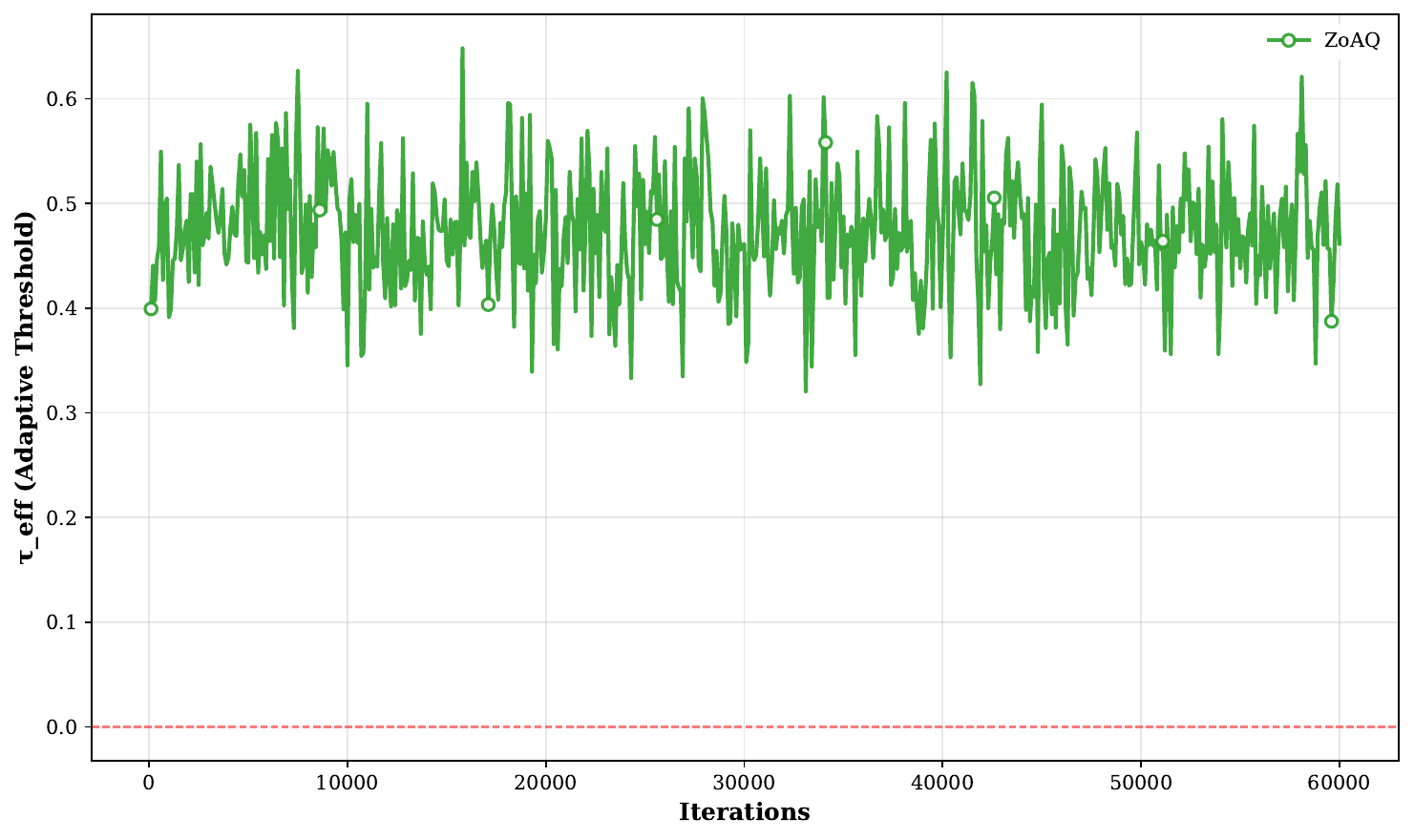}}
\subfigure[Rosenbrock]{\includegraphics[width=0.42\textwidth]{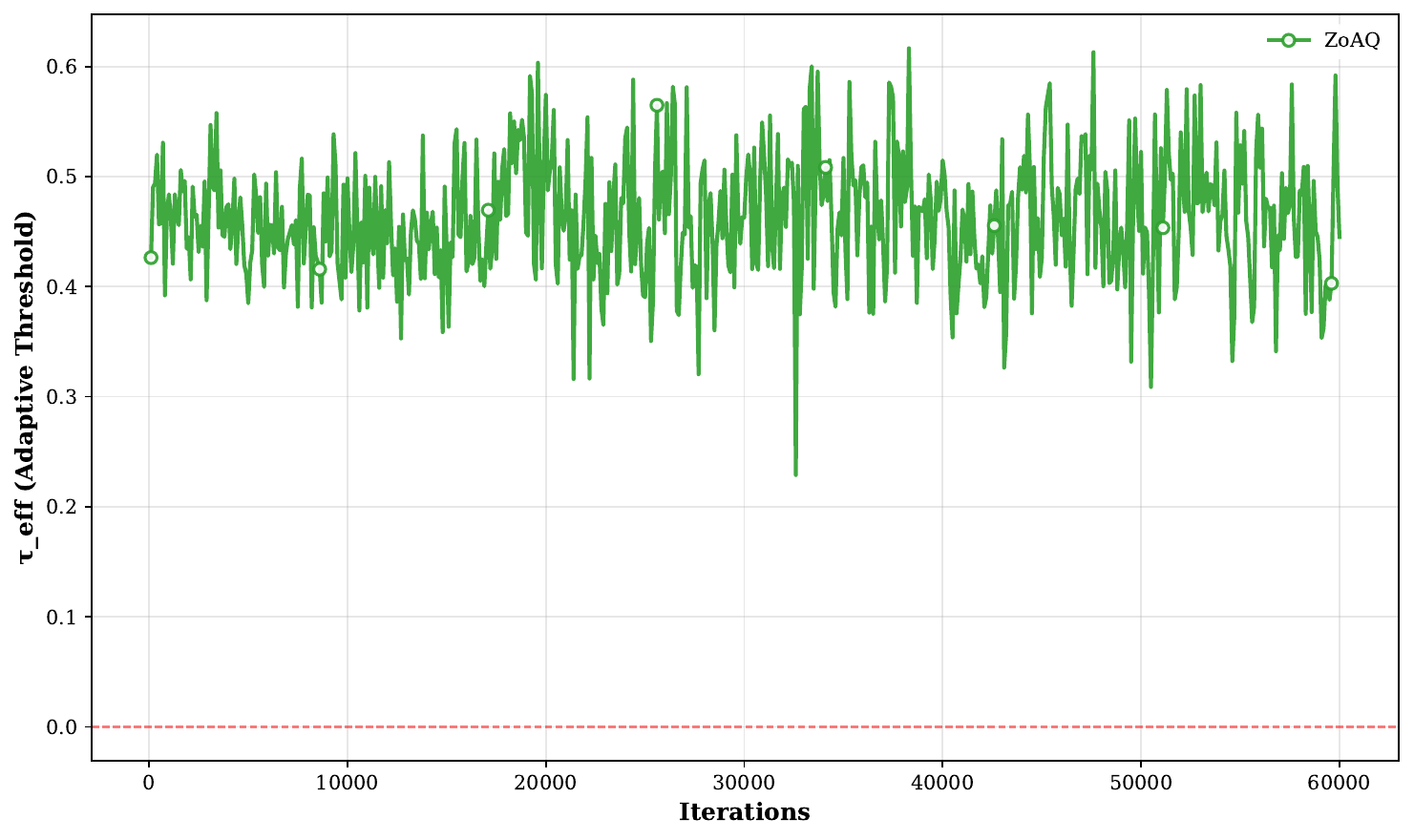}}
\caption{Dynamics of the effective threshold $\tau_t$ across functions. The EMA adapts the gate to the observed momentum-consistency scale, lowering it after inconsistent evidence and keeping it higher when recent accepted scores remain strong.}
\label{fig:tau_dynamics}
\end{center}
\vskip -0.1in
\end{figure}

\subsection{Controlled Objective Switch}
\label{app:sharp_turn}

A key concern with history reuse is whether stale gradients could trap the optimizer when the landscape changes abruptly. We design a Sharp Turn stress test to directly evaluate ZoAQ's drift-detection capability.

\paragraph{Protocol.} We optimize the Rosenbrock function ($d=100$) for 15{,}000 iterations, then instantly switch the objective to the Levy function (same $d$) and continue for another 15{,}000 iterations. This simulates an abrupt distribution shift where accumulated momentum points in a fundamentally wrong direction.

\paragraph{Results.} Figure~\ref{fig:stress_gap} shows the optimality gap. The apparent drop at step 15{,}000 is expected because Rosenbrock has naturally larger objective values than Levy. After the switch, ZoAQ reaches a lower final loss than R-AdaZO in this test, indicating that the retained buffer does not trap the method in the old landscape under this protocol.

\begin{figure}[ht]
\begin{center}
\centerline{\includegraphics[width=0.9\columnwidth]{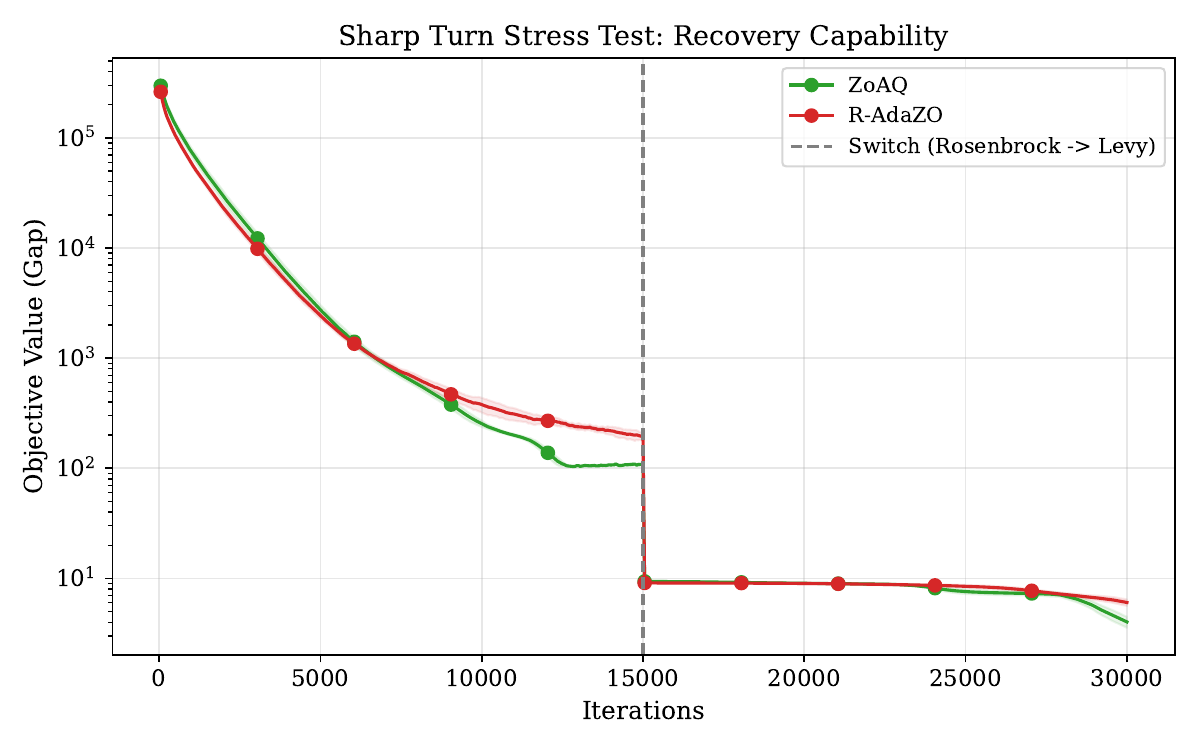}}
\caption{Optimality Gap during the Sharp Turn test. After the switch at $T=15$k, ZoAQ adapts faster to the new Levy landscape, reaching a lower final objective than R-AdaZO.}
\label{fig:stress_gap}
\end{center}
\vskip -0.1in
\end{figure}

\paragraph{Mechanism: $\tau_t$ as a drift-sensitive trigger.} Figure~\ref{fig:stress_tau} shows that, at the switch point, $\tau_t$ drops sharply toward zero, indicating a mismatch between the momentum accumulated on Rosenbrock and the new Levy geometry. The failed-consistency phase refreshes the buffer through added queries, while the lower EMA threshold records the new local score scale rather than treating the old high-consistency regime as still valid.

\begin{figure}[ht]
\begin{center}
\centerline{\includegraphics[width=0.9\columnwidth]{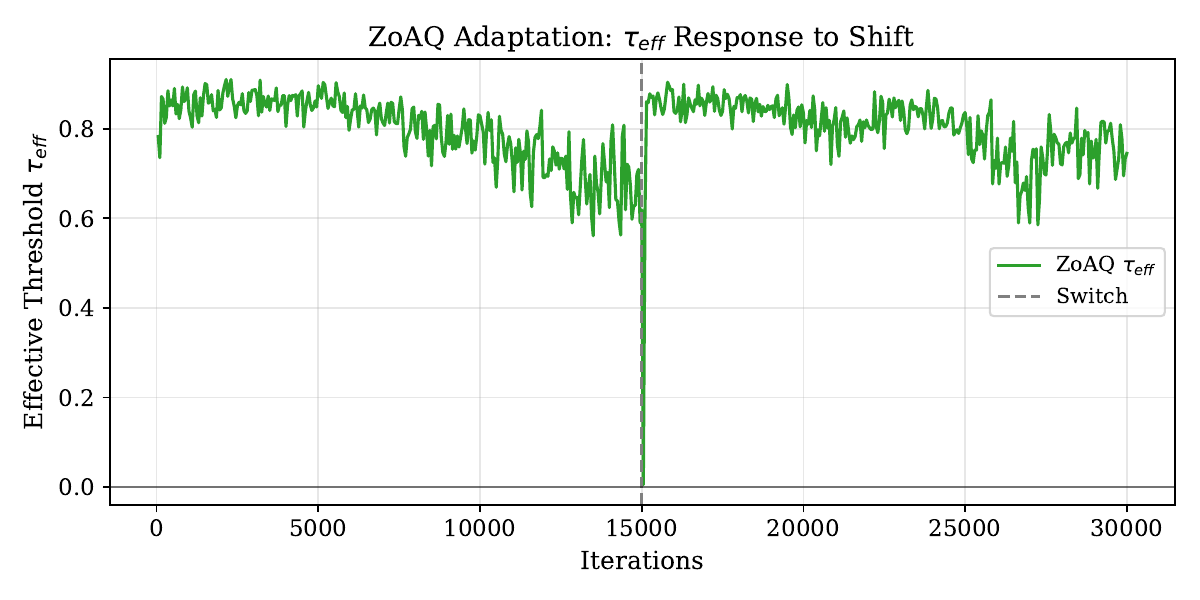}}
\caption{ZoAQ's effective threshold $\tau_t$ during the Sharp Turn. The drop at $T=15$k shows that the momentum-consistency trigger reacts to the abrupt landscape change.}
\label{fig:stress_tau}
\end{center}
\vskip -0.1in
\end{figure}

\paragraph{Takeaway.} This experiment shows that ZoAQ refreshes its history after the objective changes. The momentum-consistency trigger detects the mismatch, adds fresh evidence during failed tests, and recalibrates the acceptance scale.

\subsection{Trigger Metric Ablation}
\label{app:ablation_norm_theory}

The main text compares the momentum-consistency trigger with a relative norm trigger,
\begin{equation}
    \frac{\|\vg_t-\vm_{t-1}\|}{\|\vm_{t-1}\|}
    \le c_{\ntxt{norm}}.
    \label{eq:norm_trigger_ablation_app}
\end{equation}
This ablation asks whether matching vector magnitude is a useful allocation signal in the reported high-dimensional setting. It is not used as a theoretical lower bound on norm-based adaptive sampling.

Two effects can make Eq.~\ref{eq:norm_trigger_ablation_app} more demanding than a cosine trigger. First, the paired Gaussian contribution has the local second-moment scale in Eq.~\ref{eq:paired_gaussian_noise_scale_app}, so finite-direction estimates can fluctuate substantially in magnitude even when their directions are useful. Second, the denominator becomes sensitive when the momentum norm is small. A relative norm trigger may therefore expand the budget in response to radial fluctuation as well as directional disagreement.

The formal descent result in Appendix~\ref{app:proof_complexity} isolates this distinction with a globally normalized proxy update: its one-step bound depends on angular alignment and not on the candidate norm. This statement should not be extended automatically to the experimental R-AdaZO update. Coordinatewise second-moment preconditioning need not preserve Euclidean angles; the ablation below is empirical evidence about the tested controller and backbone, not a proof of scale invariance for Adam-style updates.

Figure~\ref{fig:ablation_norm_main} reports the comparison. On Rosenbrock with $d=10{,}000$, the relative-norm trigger uses 28.1\% more queries to reach the same target loss, while ZoAQ obtains a larger loss drop per 1k queries (1.9 versus 1.3). The result supports the use of directional consistency in this operating regime without implying that norm-based criteria are universally inferior.

\begin{figure}[ht]
\begin{center}
\centerline{\includegraphics[width=0.70\textwidth]{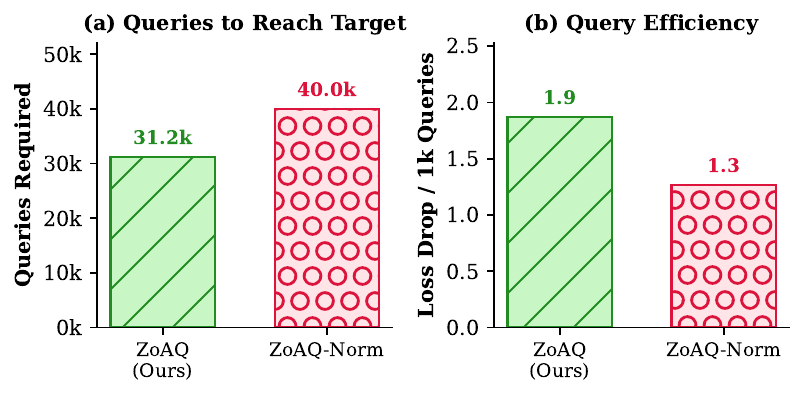}}
\caption{Trigger-metric ablation on Rosenbrock ($d=10{,}000$). In this run, the relative-norm trigger requires 28.1\% more queries to reach the same target, and its loss drop per 1k queries is 1.3 versus 1.9 for momentum consistency.}
\label{fig:ablation_norm_main}
\end{center}
\end{figure}

\section{Query-Reuse Diagnostics}
\label{app:aggregation_analysis}

This section connects the stored-response estimator analyzed in Appendix~\ref{app:proof_stat} to the empirical momentum-consistency traces. Its purpose is diagnostic: the formal guarantees come from candidate accuracy and debiased-EMA tracking, not from treating overlap-induced agreement as a certificate by itself.

\subsection{Fresh and Stored-Response Estimation}

A fresh spatial estimator uses only directions acquired at the current step,
\begin{equation}
    \vg_t^{\ntxt{fresh}}
    =
    \frac{1}{K_t}\sum_{k=1}^{K_t}
    \vg_\mu(\vtheta_t,\vu_{t,k}).
    \label{eq:spatial_agg}
\end{equation}
ZoAQ instead averages the contributions observed at acquisition in its candidate record set,
\begin{equation}
    \vg_t^{\ntxt{hist}}(K)
    =
    \frac{1}{n_t(K)}
    \sum_{j\in\mathcal J_t(K)}X_j,
    \qquad
    X_j=r_\mu(\vtheta_{a_j},\vu_j)\vu_j.
    \label{eq:history_agg}
\end{equation}
Retained records therefore increase the effective sample count without new function evaluations, but they target gradients at their acquisition iterates. The exact trade-off is the error envelope
\begin{equation}
    \|\vg_t^{\ntxt{hist}}(K)-h_t\|
    \le
    q(n_t(K),\delta)
    +\frac{L A_t(K)}{n_t(K)}.
    \label{eq:diagnostic_error_envelope_app}
\end{equation}
The first term decreases with accumulated evidence; the second prices movement since acquisition. Their balance is the mechanism used by the theory.

\begin{proposition}[When Momentum Consistency Is Interpretable]
\label{prop:directional_stability}
Suppose the simultaneous candidate event holds, and let accepted estimates satisfy $\|\vg_r-h_r\|\le\varepsilon_r$. Then Proposition~\ref{prop:tracking_anchor_main} gives
\begin{equation}
    \|\widetilde{\vm}_{t-1}-h_t\|
    \le
    \sum_{r<t}\bar w_{r,t}\varepsilon_r
    +LD_t^{\ntxt{ema}}.
    \label{eq:diagnostic_tracking_app}
\end{equation}
If the right-hand side is at most $\kappa_t\|h_t\|$, a passing score obeys the alignment bound in Eq.~\ref{eq:gate_soundness_app}. Thus history reuse can stabilize the score through improved candidate accuracy, while FIFO staleness and EMA lag determine whether that score still represents current geometry.
\end{proposition}

\begin{proof}
Equation~\ref{eq:diagnostic_tracking_app} is Eq.~\ref{eq:debiased_tracking_app} with the certificates for accepted candidates substituted. The final statement is the soundness part of Theorem~\ref{thm:implicit_norm_condition_proof}.
\end{proof}

The proposition also explains why raw agreement alone is insufficient. A large stale buffer may make consecutive estimates visually stable while tracking an outdated direction. The path-length terms in Eqs.~\ref{eq:diagnostic_error_envelope_app}--\ref{eq:diagnostic_tracking_app} exclude precisely this failure mode.

\subsection{Cosine Stability}

We record the cosine similarity between the candidate estimate and the previous momentum throughout optimization. Figure~\ref{fig:cosine_main} reports the synthetic diagnostics.

\begin{figure}[ht]
\begin{center}
\subfigure[Quadratic]{\includegraphics[width=0.45\textwidth]{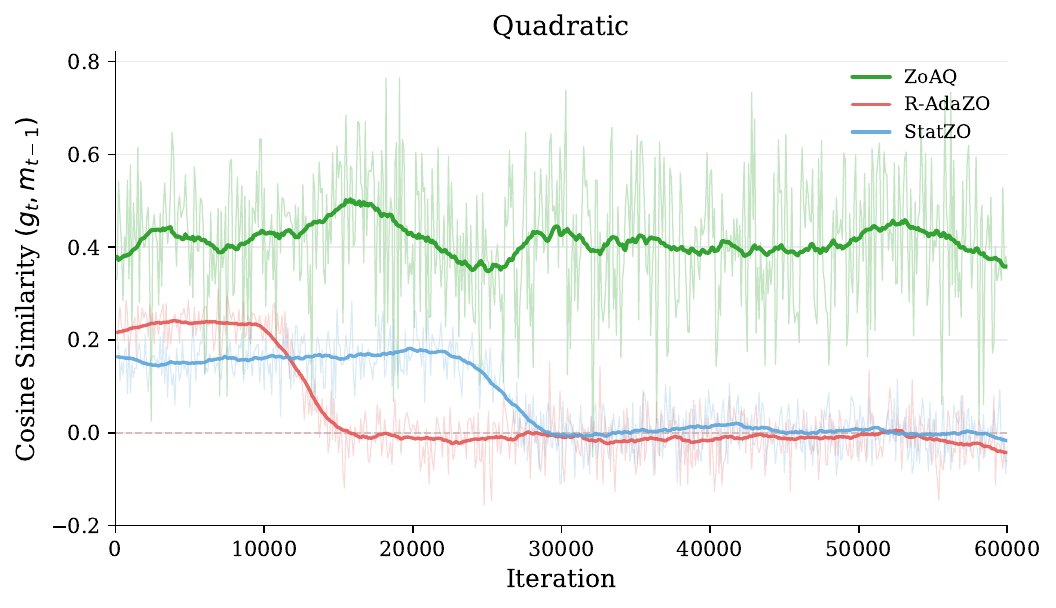}}
\subfigure[Cubic]{\includegraphics[width=0.45\textwidth]{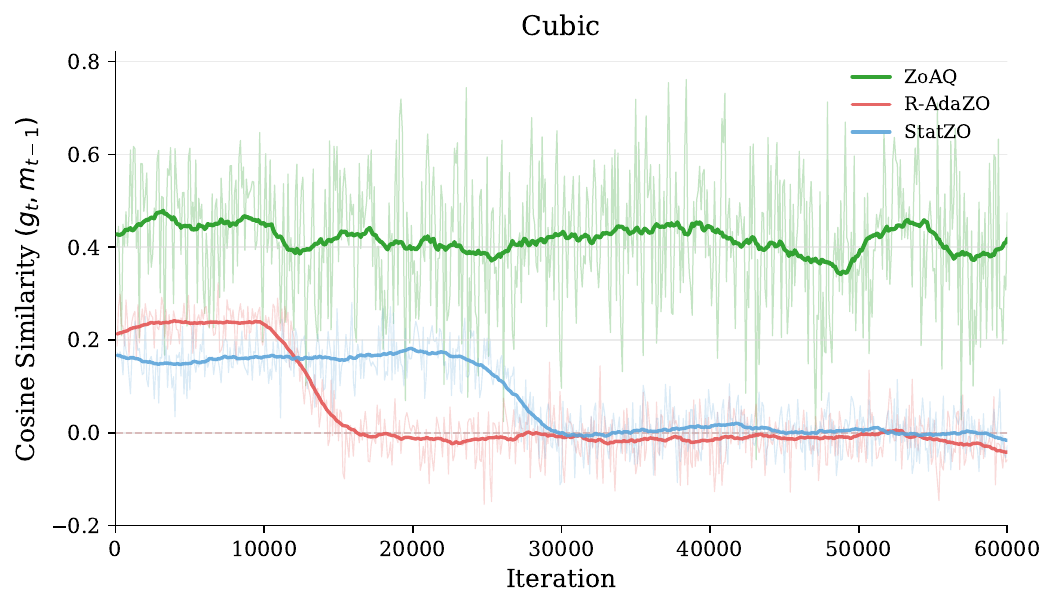}}

\subfigure[Levy]{\includegraphics[width=0.45\textwidth]{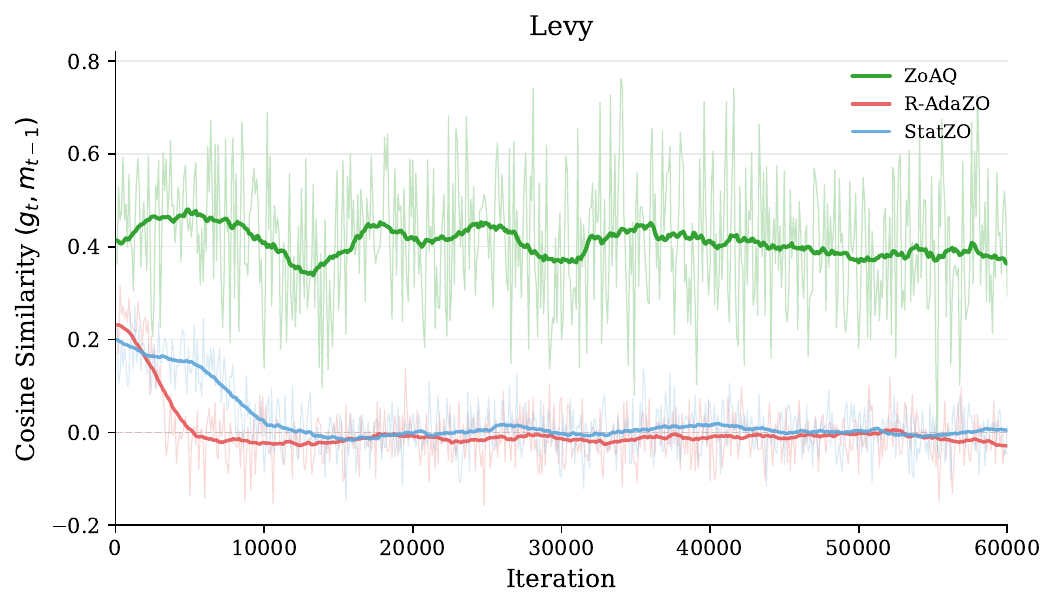}}
\subfigure[Rosenbrock]{\includegraphics[width=0.45\textwidth]{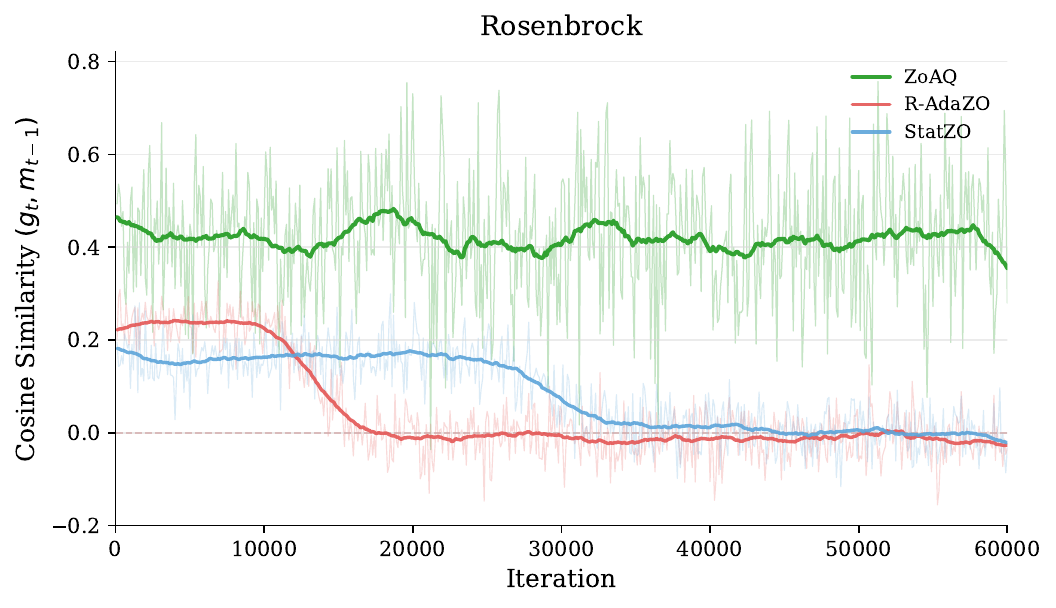}}
\caption{Synthetic cosine similarity between the candidate estimate and previous momentum under the 60k-step protocol. Curves are recorded every 100 iterations and averaged over five runs; faint lines show the recorded means and dark lines show smoothed trends. ZoAQ remains more persistently positive than the displayed baselines in these runs.}
\label{fig:cosine_main}
\end{center}
\end{figure}

The comparison is descriptive rather than a direct measurement of the theorem's tracking ratio. A positive trace is consistent with a stable reference, but the guarantee additionally requires the candidate and path-length bounds above.

Figure~\ref{fig:attack_cosine} gives the corresponding black-box attack diagnostic. The same qualitative pattern appears in this setting: the history-reuse variants maintain a more persistent positive similarity signal, whereas the spatial variants fluctuate more strongly.

\begin{figure}[ht]
\begin{center}
\subfigure[R-AdaZO (Spatial)]{\includegraphics[width=0.42\textwidth]{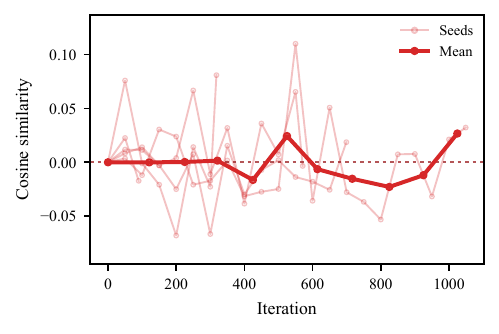}}
\subfigure[ZO-AdaMM (Spatial)]{\includegraphics[width=0.42\textwidth]{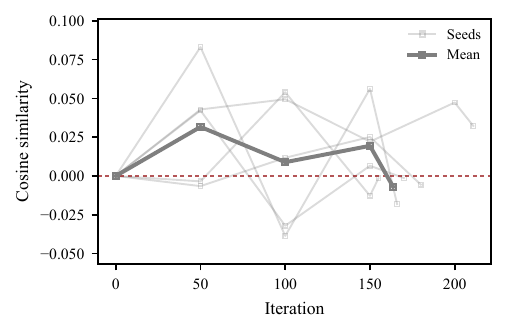}}

\subfigure[ZoAR (History Reuse)]{\includegraphics[width=0.42\textwidth]{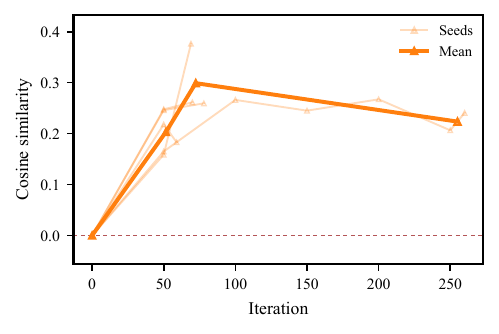}}
\subfigure[ZoAQ (History Reuse)]{\includegraphics[width=0.42\textwidth]{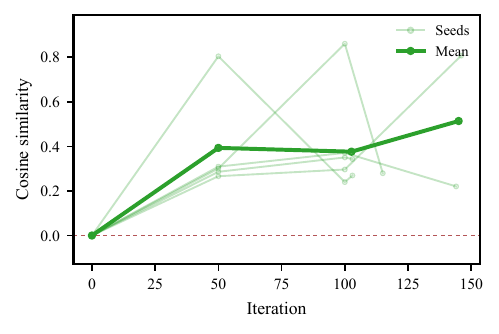}}
\caption{Momentum-consistency traces during black-box adversarial attacks on MNIST. History-reuse variants show a more persistently positive signal than the spatial variants in the reported runs.}
\label{fig:attack_cosine}
\end{center}
\end{figure}

\subsection{Query-Space and Update-Space Averaging}

History reuse and momentum average at different stages. Equation~\ref{eq:history_agg} pools acquisition records before a candidate is accepted. Momentum then averages the sequence of accepted estimates:
\begin{equation}
    \vm_t
    =
    (1-\beta_1)\sum_{r=1}^{t}\beta_1^{t-r}\vg_r.
    \label{eq:momentum_expansion}
\end{equation}
The first operation trades sampling concentration against response staleness; the second trades estimate noise against trajectory lag. Their effects are complementary but not interchangeable, and both costs appear explicitly in the proofs.

This view suggests a practical design rule. A longer history is useful in noisy, slowly moving regions, whereas faster movement favors a shorter history. The capacity scale in Eq.~\ref{eq:capacity_scale_app} makes this dependence explicit. Similarly, the momentum parameter should be interpreted through the lag factor $(1-\beta_1)^{-1}$ in Eq.~\ref{eq:uniform_debiased_tracking_app}, rather than as cost-free smoothing.

\subsection{Diagnostic Scope}

The figures support two empirical observations: candidates built from stored responses yield a more stable score in the reported regimes, and failed tests add fresh evidence before the candidate is accepted. They do not show that overlap guarantees correctness. Formal interpretability requires the same ingredients as the main theorem. These are candidate concentration valid under selection, bounded acquisition staleness, debiased EMA tracking, and a feasible threshold band.

\section{Limitations and Future Directions}
\label{app:limitations}
\label{app:extended_discussion}

We currently retain a fixed number $N_{\ntxt{hist}}$ of recent records and initialize the EMA threshold with a warm start. Our window sweep and controlled objective switch show when reused responses remain helpful and when they become stale. Adapting or reweighting the retained history as the trajectory changes is therefore a natural next step. The present analysis assumes explicit bounds on iterate movement, EMA tracking, thresholds, and local query budgets; extending it across initialization, resets, and record pruning would give a fuller account of transient behavior. We evaluate forward-pass fine-tuning on OPT-1.3B and OPT-13B with SST-2 and COPA. Broader model families, tasks, and training horizons remain to be studied.

\end{document}